\documentclass[11pt, a4paper, logo]{googlecloud}

\pdftrailerid{redacted}

\makeatletter
\renewcommand\bibentry[1]{\nocite{#1}{\frenchspacing\@nameuse{BR@r@#1\@extra@b@citeb}}}
\makeatother

\usepackage{hyperref}
\usepackage{url}
\usepackage{natbib}
\usepackage{amsmath, amssymb, amsthm}
\usepackage{algorithm}
\usepackage{algpseudocode}
\usepackage{xcolor}
\usepackage[table]{xcolor}
\newcommand{\stderr}[1]{%
    {\scriptsize\color{gray!85}$\pm$ #1}%
}
\usepackage{booktabs} 
\usepackage[T1]{fontenc}
\usepackage{enumitem}
\usepackage{graphicx}
\usepackage{wrapfig}
\usepackage{graphicx}
\usepackage{bbm}
\usepackage{fontawesome5}
\usepackage{twemojis}
\usepackage{subcaption}
\usepackage{circledsteps}
\usepackage{multirow}

\usepackage[most]{tcolorbox}
\usepackage{fancyvrb}
\usepackage{fvextra}
\fvset{
  fontsize=\small,
  breaklines=true,
  breaksymbolleft={},
  breaksymbolright={},
  breakautoindent=false,   
}
\tcbset{
    promptbox/.style={
        enhanced,
        breakable,
        colback   = blue!3,
        colframe  = blue!45!black,
        colbacktitle = blue!45!black,
        coltitle  = white,
        fonttitle = \bfseries,
        boxrule   = 0.6pt,
        arc       = 2mm,
        left      = 3mm, right = 3mm, top = 2mm, bottom = 2mm,
        title     = {#1},
    },
}

\usepackage{pgfplots}
\pgfplotsset{compat=1.18}

\definecolor{mydarkblue}{rgb}{0,0.08,0.45}

\usepackage{minitoc}

\newtheorem{assumption}{Assumption}

\newtheorem{proposition}{Proposition}
\newtheorem{definition}{Definition}

\usepackage[capitalize,noabbrev]{cleveref}

\usepackage[textsize=tiny]{todonotes}

\title{\texttwemoji{bullseye} \textsc{AIM}: Agentic Idea Management for Automated Research}

\correspondingauthor{hyeongkyu.choi@wisc.edu$^\dagger$, bhavana@google.com, jiefengc@google.com\\$^\dagger$ This work was done while Hyeong Kyu Choi was a Student Researcher at Google Cloud AI Research.}

\author[1,2]{Hyeong Kyu Choi}
\author[1]{Bhavana Dalvi Mishra}
\author[1]{Jiefeng Chen}
\author[1]{Mihir Parmar}
\author[1]{Rui Meng}
\author[1]{Chun-Liang Li}
\author[1]{\\Xiangru Tang}
\author[2]{Sharon Li}
\author[1]{Jinsung Yoon}
\author[1]{Tomas Pfister}

\affil[1]{Google Cloud AI Research}
\affil[2]{University of Wisconsin-Madison}

\begin{abstract}
Frontier LLMs are increasingly used to automate scientific research through iterative search.
We distinguish idea-driven search from solution-driven search and identify three core challenges: organizing evolving research ideas, selecting promising directions, and maintaining alignment between ideas and their implementations.
To address these challenges, we introduce the \emph{Agentic Idea Manager}~(\textsc{AIM}), a fully autonomous framework for managing and exploring research directions in idea-driven automated research.
Inspired by Bayesian optimization, \textsc{AIM} uses an Agentic Surrogate and an Agentic Acquisition mechanism to organize discovered ideas and guide their selection.
A Solution Auditor maintains idea--solution integrity, while a Resource Planner adaptively allocates the remaining experimental budget across parallel search branches.
Experiments on 10 AutoLab benchmark tasks show that \textsc{AIM} 
surpasses the strongest baseline by 1.6 percentage points on System Optimization tasks and 4.9 percentage points on long-horizon Model Development \& CUDA tasks.
Notably, \textsc{AIM} reaches the best baseline performance up to 3.1$\times$ faster in wall-clock time.
We further provide a theoretical analysis of \textit{when searching over ideas becomes beneficial}.
Our analysis shows that explicit idea-level allocation makes semantic coverage directly controllable, and that broader coverage becomes increasingly valuable when competitive research directions are sparse among many plausible alternatives.\\
\faGlobe\enspace\textbf{Project Page: } \href{https://imhgchoi.github.io/agentic-idea-manager/}{https://imhgchoi.github.io/agentic-idea-manager/}
\end{abstract}

\begin{document}
\doparttoc 
\faketableofcontents 

\maketitle

\section{Introduction}
\label{sec:introduction}
Large language model (LLM) agents are increasingly used to automate scientific research through iterative experimentation.
Modern research agents can propose candidate approaches, implement solutions, execute experiments, inspect verifier feedback, and use accumulated evidence to determine what to try next~\citep{lu2024ai,jiang2025aide,toledo2026ai,meng2026scientistone}.
Because implementation and evaluation are expensive, effective automated research requires allocating a limited experimental budget across possible approaches.
The resulting search process must support both the discovery of promising directions and the refinement of their implementations.

In this work, we introduce the first categorization of automated research, according to its primary unit of search: \emph{solution-driven} and \emph{idea-driven}. 
Solution-driven approaches~(Figure~\ref{fig:dummy}~(a)) search directly over executable artifacts, using verifier feedback to iteratively modify candidate solutions~\citep{cemri2026adaevolve,liu2026evox,toledo2026ai,jiang2025aide}.
Idea-driven approaches~(Figure~\ref{fig:dummy}~(b)), on the other hand, maintain research ideas as explicit object of reasoning; they select \emph{which idea to investigate} and delegate \emph{how to implement it} to a solver~\citep{meng2026scientistone,jin2026toward,yamada2025ai,weng2026deepscientist}.
Both paradigms ultimately evaluate executable solutions, but organize search at different levels of abstraction.
Working directly with solutions supports fine-grained code refinement, while idea-driven methods facilitate comparison across approaches, helps the transfer of lessons, and makes research trajectories easier to interpret.

\begin{figure}[t!]
    \centering
    \includegraphics[width=0.65\linewidth]{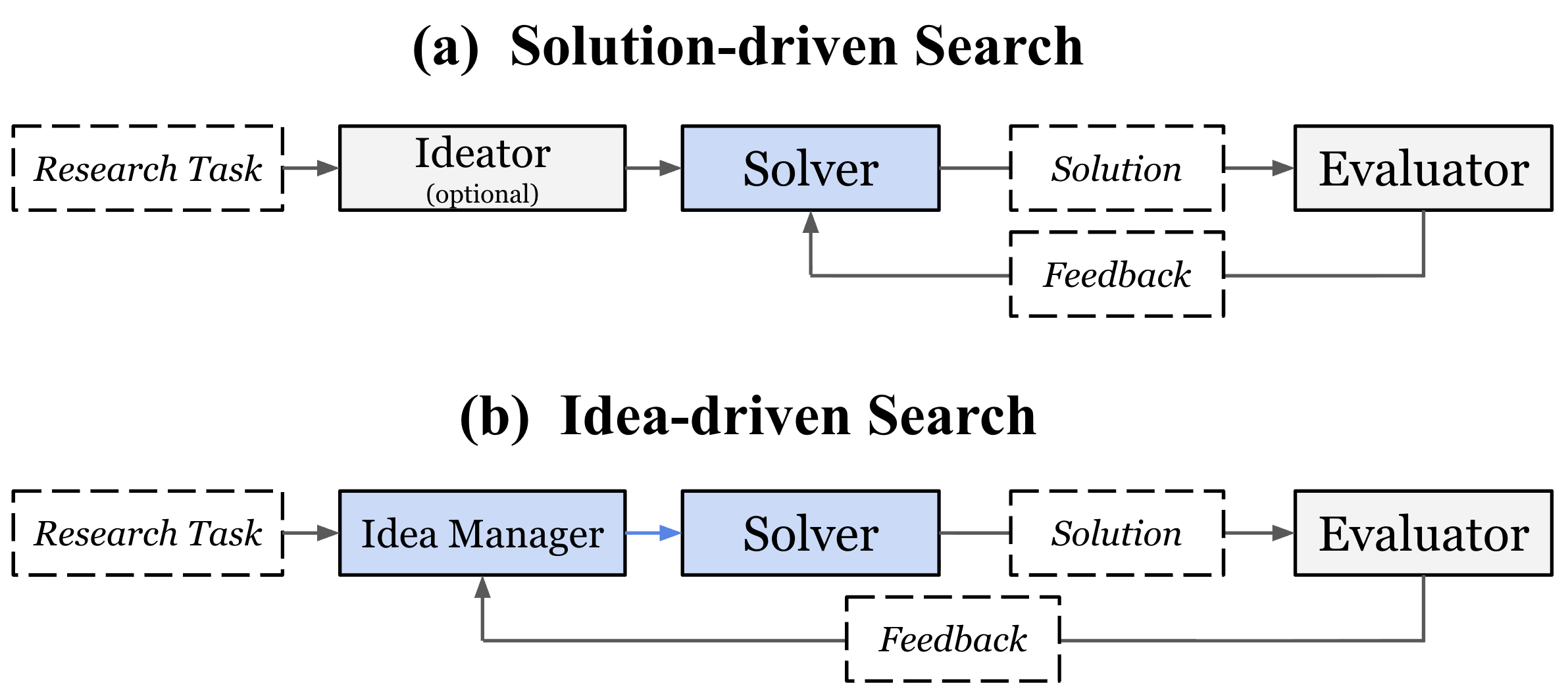}
    \caption{\textbf{Solution-driven vs Idea-driven comparison.} While Solution-driven approaches take the solution code as the target object for optimization, Idea-driven approaches first reason over research ideas, directions, or hypotheses based on evaluator feedback, after which the chosen ideas are implemented by the solver agent.}
    \label{fig:dummy}
\end{figure}
This categorization highlights the management of research ideas as a distinct design problem in automated research, with three central challenges in determining the \emph{idea management structure}, \emph{idea selection strategies}, and ensuring \emph{idea--solution integrity}.
To address these challenges, we introduce the \textbf{{Agentic Idea Manager}~(\textsc{AIM})}, a fully-autonomous research idea management framework for automated research~(Figure~\ref{fig:mainfig}).
Inspired by the surrogate-acquisition structure of Bayesian optimization, its \emph{Agentic Surrogate} maintains a dynamic structure of idea clusters and ranks their promise using experimental evidence, grouping related proposals across generation lineages.
The \emph{Agentic Acquisition} mechanism balances exploration and exploitation at both cluster- and idea-level selection, dispatching chosen candidates to parallel solvers.
After implementation and evaluation, the \emph{Solution Auditor} checks task validity and idea--solution integrity, and the audited evidence guides subsequent direction of research.
Finally, the \emph{Resource Planner} adjusts parallelism across iterations, balancing concurrent trials with longer exploration under a fixed experimental budget.

We evaluate \textsc{AIM} on ten AutoLab tasks spanning System Optimization and Model Development \& CUDA tasks.
\textsc{AIM} achieves the highest average scores in both task groups: {67.0\%} on System Optimization and {55.8\%} on long-horizon Model Development \& CUDA, exceeding the strongest baseline, ScientistOne~\citep{meng2026scientistone}, by {1.6} and {4.9 percentage points}, respectively.
Moreover,  \textsc{AIM} reaches ScientistOne's best score up to 3.1$\times$ faster, demonstrating its efficiency and supporting the value of agentic idea organization, evidence-guided selection, and audited feedback for automated research.
In addition to empirical studies, we further provide theoretical analyses to identify when explicit idea-level search is useful.
Our analysis shows that idea-driven search makes such breadth directly controllable through explicit allocation to distinct ideas.
More importantly, broader coverage becomes increasingly valuable when a task contains many plausible research directions but only a small fraction are competitive.

Our main contributions are:

\begin{itemize}
    \item We introduce a categorization of automated research methods: \textit{solution-driven} and \textit{idea-driven}. We identify \textit{three central challenges} of idea-driven approaches: idea management, idea selection, and idea--solution integrity.
    \item We introduce \textit{Agentic Idea Manager}~(\textsc{AIM}), a fully-agentic idea-driven framework that effectively manages ideas as semantic clusters, selects ideas based on its explicit estimation of performance, and audits the solutions to achieve a reliable research pipeline.
    \item We achieve \textit{strong empirical performance} on 10 Autolab tasks, surpassing the strongest baseline by 1.6pp on System Optimization and 4.9pp on long-horizon Model Development \& CUDA tasks. We also provide \textit{rigorous theoretical analyses} of when idea-driven search can be useful.
\end{itemize}
\section{Related Works}

\paragraph{Solution-driven Search.}
Solution-driven approaches remain close to the verifier and make the executable implementation the primary unit of search.
AIDE~\citep{jiang2025aide} and AIRA~\citep{toledo2026ai} explore candidate solutions through tree search, while AlphaEvolve~\citep{novikov2025alphaevolve}, AdaEvolve~\citep{cemri2026adaevolve}, and EvoX~\citep{liu2026evox} use evolutionary mechanisms.
MLE-Star~\citep{nam2026mle}, DS-Star~\citep{nam2025ds}, and RPM~\citep{foster2026ai} further develop solution-level refinement and selection.
Solution-driven search's tight coupling of the evaluator and solver lets experimental feedback directly guide implementation changes, but can entangle progress on a research direction with engineering decisions about a particular artifact.

\paragraph{Idea-driven Search.}
Idea-driven approaches make research ideas or hypotheses explicit search objects, then delegate their implementation to a solver.
The AI Scientist-v2~\citep{yamada2025ai}, MARS~\citep{chen2026mars}, and Arbor~\citep{jin2026toward} manages a tree of ideas or lessons, and DeepScientist~\citep{weng2026deepscientist} maintains the research ideas and hypotheses as a list.
A subset of ideas is implemented by a solver module.
ScientistOne~\citep{meng2026scientistone}, on the other hand, utilizes a beam-search-like algorithm that keeps the best-performing ideas throughout the research iterations.
These approaches separates idea search from implementation, enabling deliberate exploration of distinct ideas, but requires deciding which should merit costly experiments and evaluations.

This idea-driven search introduces three central challenges:
(1) \textit{Idea management scaffold}: the framework must organize an expanding pool of ideas and accumulated evidence into a persistent, interpretable representation.
(2) \textit{Idea selection}: it must select which ideas to implement and how to allocate the remaining budget; existing methods typically rely on fixed rules such as UCB~\citep{weng2026deepscientist} or MCTS~\citep{yamada2025ai}, or naively use LLM judgments without explicit estimates~\citep{jin2026toward}.
(3) \textit{Idea--Solution integrity}: the separation between ideation and implementation creates an integrity risk.
A solver may produce code that does not faithfully realize the selected idea, causing its score and derived lessons to be misattributed.
These challenges motivate jointly managing the idea space, making evidence-grounded search decisions, and verifying idea--solution alignment, which we address with the our proposed method.
In this work, we propose an idea-driven research framework that addresses these challenges.
A formal discussion on when to search over ideas is provided in Section~\ref{sec:theory}. 
An extended section for Related Works is in Appendix~\ref{apdx:related-works}.
\section{Automated Research: Problem Definition}
\label{sec:setup}

In automated research, it is generally infeasible to evaluate every plausible idea because implementation and verifier calls are expensive.
Effective automated research therefore requires more than sequential idea and code generation.
It requires a structured representation of the ideas and solutions discovered so far, and a principled mechanism to navigate the search process under a limited computational budget.
Accordingly, we formally define the relevant search space and objectives:

\paragraph{Idea and Solution Spaces.}
Let $\mathcal{X}$ denote the space of admissible research ideas for task $\tau$, where each $x\in\mathcal{X}$ is a natural-language description of a candidate approach.
In this work, an \textit{idea} is defined as, but not limited to, a structured text comprising a short title, brief hypothesis/abstract, and experiment plans~(see Appendix~\ref{example1} for examples).
At timestep $t$, the agent has access only to a finite pool
$\mathcal{P}_t\subset\mathcal{X}$ containing the ideas discovered thus far.
Meanwhile, we distinguish the space of executable solutions $\mathcal{Z}$ from the idea space $\mathcal{X}$.
Given an idea $x$, an implementation process $g:\mathcal{X}\rightarrow\mathcal{Z}$ produces an executable solution $z=g(x)$, assuming a single deterministic mapping per idea in this work, which is then scored by a fixed verifier $v:\mathcal{Z}\rightarrow\mathbb{R}$:
\begin{equation}
z=g(x), \qquad y=v(z) \quad \Longrightarrow \quad f(x)\triangleq v(g(x)).
\label{eq}
\end{equation}
Here, $f$ denotes the expensive research process for implementation and evaluation.

\paragraph{Objective.} 
Evaluating $f(x)$ requires first translating the natural-language idea into an executable implementation and then compiling and running that implementation against a fixed verifier. 
Consequently, the dominant cost is not in proposing an idea, but obtaining reliable evidence about its quality through implementation and verification, which can itself even be very expensive and noisy.
Thus, the core objective is to identify the highest-performing idea within a given computational budget measured in the number of experiment executions and/or runtime hours for search.
Let $c(x)$ denote the cost of implementing and verifying idea $x$, and let $N$ be the total computational budget.
For a sequence of evaluated ideas $x_1,\ldots,x_n$, the objective is to maximize the best performance:
\begin{equation}
\max_{x_1,\ldots,x_n\in\mathcal{X}}
\;
\max_{1\leq i\leq n} f(x_i)
\qquad
\text{s.t.}
\qquad
\sum_{i=1}^{n}c(x_i)\leq N.
\label{eq:objective}
\end{equation}
Within this framework, the agent must therefore determine which subset of ideas to explore first.

\section{The \textsc{AIM} Framework}
\label{sec:method}

\noindent\textbf{Overview.}
We introduce the \textbf{Agentic Idea Manager}~(\textsc{AIM}), a fully autonomous framework for managing and searching research ideas under a limited experimental budget. 
For a research task $\tau$, \textsc{AIM} maintains the search state
\begin{equation}
\mathcal{S}_t =
\left(
\mathcal{T},
\mathcal{P}_t,
\mathcal{M}_t,
\mathcal{C}_t,
\mathcal{R}_t,
\mathcal{D}_t
\right),
\label{eq}
\end{equation}
where $\mathcal{T}$ is the fixed task context, $\mathcal{P}_t$ is the current idea pool, $\mathcal{M}_t$ is the memory or lessons that store verified lessons from previous experiments, $\mathcal{C}_t$ is the semantic organization of the pool, $\mathcal{R}_t$ contains ordinal promisingness estimates over clusters and ideas, and $\mathcal{D}_t$ contains past idea--score observations. 
Following \cite{meng2026scientistone}, we adopt its ideator structure and its construction of the task context.
Specifically, $\mathcal{T}$ consists of the original task description and an initial research brief that provides supplementary context for ideation\footnote{In this work, the research brief is generated from the task description using Claude Code~\citep{claude_code}}.

Drawing functional inspiration from Bayesian optimization, \textsc{AIM} comprises an \emph{Agentic Surrogate} that organizes ideator-generated candidates and estimates the relative promise of clusters and ideas~(Section~\ref{sec:agentic-surrogate}), and an \emph{Agentic Acquisition} module that selects ideas and adaptively balances exploration and exploitation~(Section~\ref{sec:agentic-acquisition}). 
Each selected idea is implemented and evaluated by an independent Solver. 
The Solution Auditor then validates the result and reconciles the intended idea with the mechanism realized in code~(Section~\ref{sec:solution-auditor}).
The audited observations and lessons are used to update the search state and expand the idea pool for the next iteration.
Finally, the \emph{Resource Planner} determines how the remaining experiment budget is allocated across subsequent search iterations~(Section~\ref{sec:resource-planner}). 
A visual overview is in Figure~\ref{fig:mainfig}, and the algorithm for \textsc{AIM} is in Algorithm~\ref{alg:aim}.

\begin{figure}[t]
    \centering
    \includegraphics[width=0.9\linewidth]{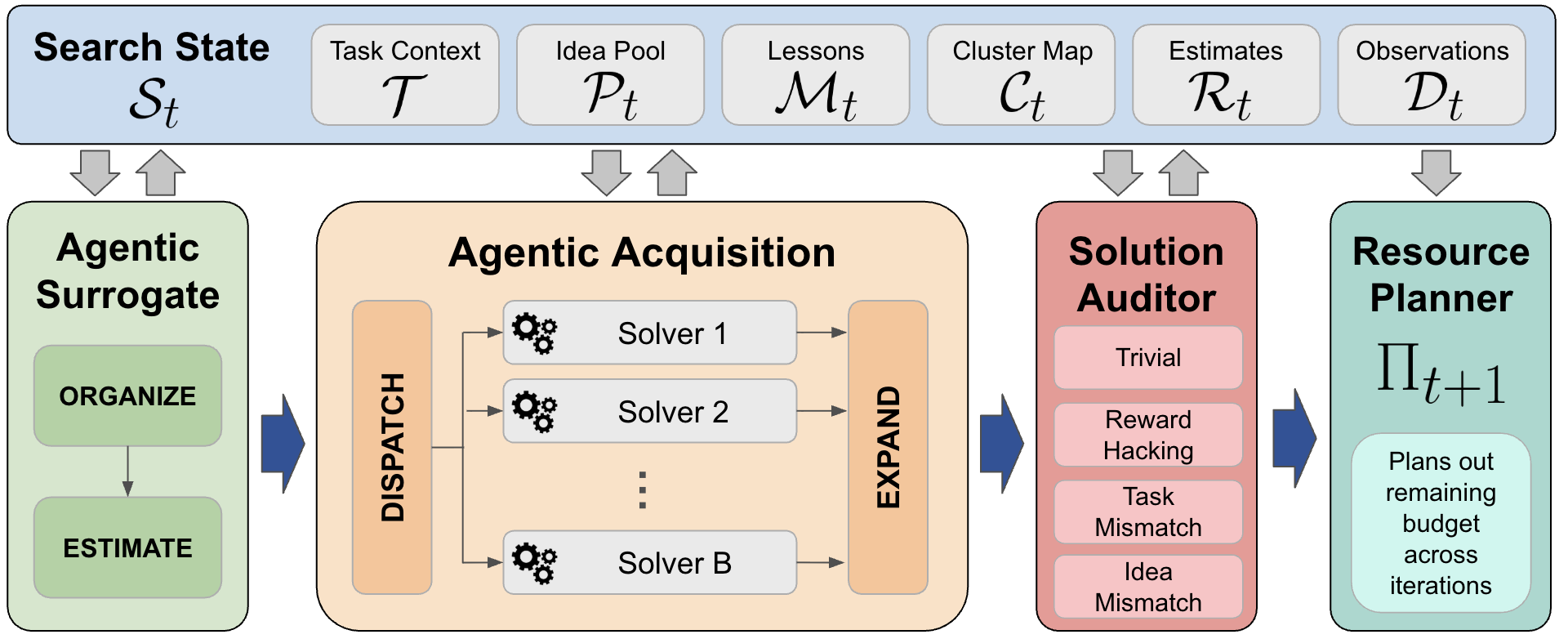}
    \caption{The \textsc{AIM} pipeline: the \textit{Agentic Surrogate} organizes and estimates candidate idea promise and the \textit{Agentic Acquisition} selects ideas for execution and expands the idea pool. Results are audited by the \textit{Solution Auditor}, while the \textit{Resource Planner} dynamically allocates the budget.}
    \label{fig:mainfig}
\end{figure}
\subsection{Agentic Surrogate: Organize and Estimate}
\label{sec:agentic-surrogate}
The Agentic Surrogate module constructs an explicit, evidence-conditioned representation of the discovered idea space through two operators: \emph{Organize} and \emph{Estimate}.

\noindent\textbf{Organize.}
Given the current idea pool $\mathcal{P}_t$ and evaluation history $\mathcal{D}_t$, the Organize operator constructs a cluster map
\begin{equation}
\mathcal{C}_t
=
\phi_{\mathrm{org}}
\left(
\mathcal{T},
\mathcal{P}_t,
\mathcal{D}_t,
\mathcal{C}_{t-1}
\right).
\label{eq}
\end{equation}
Each cluster in $\mathcal{C}_t$ represents a broad research direction and contains semantically related ideas from $\mathcal{P}_t$. 
Evaluated ideas and their scores serve as empirical landmarks for interpreting related but unevaluated candidates. 
Because the map is reconstructed from the complete pool at every timestep $t$, ideas from different generation lineages may be grouped together, and the organization may change as new candidates and evidence become available.
A visualization of how the clusters evolve throughout iterations is provided in Figure~\ref{fig:example}, also demonstrating \textsc{AIM}'s interpretability as an idea-driven approach.

\noindent\textbf{Estimate.}
Conditioned on the organized idea map $\mathcal{C}_t$, and evaluation history $\mathcal{D}_t$, the Estimate operator produces ordinal estimates of promisingness at both the cluster and idea levels:
\begin{equation}
\mathcal{R}_t
=
\phi_{\mathrm{est}}
\left(
\mathcal{C}_t,
\mathcal{D}_t
\right)
=
\left(
\mathcal{R}_t^{\mathrm{cluster}},
\mathcal{R}_t^{\mathrm{idea}}
\right).
\label{eq}
\end{equation}
Here, $\mathcal{R}_t^{\mathrm{cluster}}$ ranks the cluster-level research directions represented in $\mathcal{C}_t$, while $\mathcal{R}_t^{\mathrm{idea}}$ ranks the unevaluated ideas within each cluster.
The estimator considers observed performance, evidence scarcity, semantic novelty, and relevant implementation lessons. We use ordinal estimates because the purpose is to make relative promisingness explicit without requiring the agent to produce calibrated numerical reward predictions.
An analysis on the preciseness of the estimation is in Appendix~\ref{apdx:estimate}.

\begin{figure}[t!]
    \centering
    \includegraphics[width=0.6\linewidth]{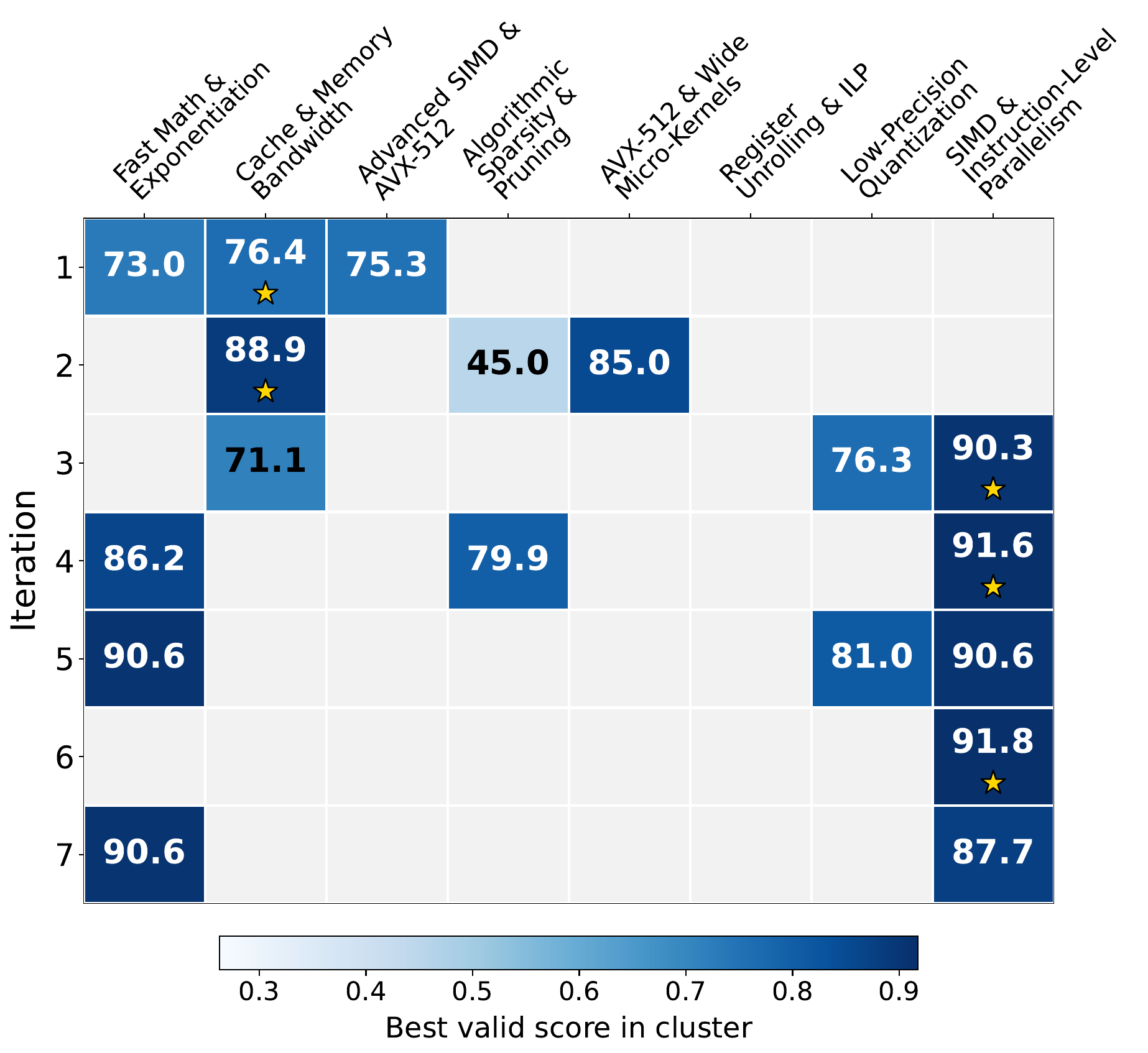}
    \caption{A visualization of a real pipeline example on the Flash Attention task. The x-axis shows all the cluster themes explored, numbers in each cell are the best scores from respective idea clusters at each round, and the stars indicate the trailing best score.}
    \label{fig:example}
\end{figure}

\subsection{Agentic Acquisition: Dispatch, Solve, and Expand}
\label{sec:agentic-acquisition}

The Agentic Acquisition module uses the current map $\mathcal{C}_t$ and promisingness estimates $\mathcal{R}_t$ to determine how to select which ideas to evaluate. 
Rather than relying on fixed selection rules or heuristics, \emph{Dispatch} assigns exploration and exploitation actions using evidence-grounded estimates of the promisingness of each candidate idea.
After the dispatched ideas are implemented and evaluated, \emph{Expand} uses the resulting evidence to generate candidates for subsequent iterations.

\noindent\textbf{Dispatch.}
Given the current idea scaffold $\mathcal{C}_t$, ordinal estimates $\mathcal{R}_t$, and resource plan $\Pi_t$~(Section~\ref{sec:resource-planner}), the Idea Dispatch operator $\phi_{\mathrm{disp}}$ selects a batch of unevaluated ideas:
\begin{equation}
    \mathcal{Q}_t
    =
    \phi_{\mathrm{disp}}
    \bigl(
        \mathcal{P}_t,
        \mathcal{C}_t,
        \mathcal{R}_t,
        \Pi_t
    \bigr)
\label{eq:idea-dispatch}
\end{equation}
where $\mathcal{Q}_t$ contains the ideas assigned to $B_t$ parallel Solver branches, with $B_t$ determined by $\Pi_t$.
In practice, Dispatch is implemented through two sequential LLM calls.
The first call jointly assigns each branch a two-tier action consisting of a cluster-level and an idea-level decision, each selected from $\{\textsc{explore},\textsc{exploit}\}$.
Exploitation prioritizes highly ranked research directions or ideas, whereas exploration favors underexplored or uncertain alternatives. 
The second call instantiates these chosen actions by selecting a target cluster and a corresponding idea for each branch.
For instance, the cells activated in Figure~\ref{fig:example} show which cluster theme was selected at each iteration.
Then, Each $x\in\mathcal{Q}_t$ is passed to an independent Solver module, which produces an executable solution $z=g(x)$, a verifier score $\widetilde{y}=v(z)$, and an execution record.
In this work, we mainly use the Gemini Deep Solver utilized in ScientistOne~\citep{meng2026scientistone}~(Claude Code solver substitution analysis is in Appendix~\ref{apdx:claude-code}).

\noindent\textbf{Expand.}
The Expand operator updates the search space in two stages.
First, it extracts reusable lessons $\Delta\mathcal{M}_{t+1}$ from the audited Solver runs and updates the implementation memory:
\begin{equation}
\mathcal{M}_{t+1}
=
\mathcal{M}_t
\cup
\Delta\mathcal{M}_{t+1}.
\label{eq:memory-update}
\end{equation}
These lessons capture effective implementation choices, unresolved performance bottlenecks, and compilation, execution, or verification failures to repair or avoid.
This design is inspired by evolutionary frameworks that use lessons distilled from experimental outcomes to guide subsequent discovery~\citep{cemri2026adaevolve,liu2026evox}.

Second, conditioned on the updated memory and audited evaluation history, the Expand operator generates new candidates and adds them to the idea pool:
\begin{equation}
\Delta\mathcal{P}_{t+1}
=
\phi_{\mathrm{exp}}
\left(
\mathcal{P}_t,
\mathcal{D}_{t+1},
\mathcal{M}_{t+1}
\right),
\qquad
\mathcal{P}_{t+1}
=
\mathcal{P}_t
\cup
\Delta\mathcal{P}_{t+1}.
\label{eq:idea-push}
\end{equation}
For each candidate, the operator selects relevant source ideas or lessons and applies one of four generation modes:
\begin{itemize}[leftmargin=*]
    \item \textbf{Score-guided refinement} preserves the validated components of a high-performing idea while addressing its remaining bottlenecks;
    \item \textbf{Cross-pollination} combines complementary ideas or lessons within or across clusters;
    \item \textbf{Error-guided repair}revises an unsuccessful idea using verifier feedback; and
    \item \textbf{Novel idea generation} introduces a previously unrepresented research direction without requiring an existing parent.
\end{itemize}
The first three modes develop existing idea lineages using accumulated evidence, whereas \texttt{new\_idea} broadens the search space and prevents concentration on established directions.

\subsection{Solution Auditor}
\label{sec:solution-auditor}
The Solution Auditor prevents invalid or misattributed results from corrupting subsequent search decisions. For each dispatched idea $x$, implementation $z$, reported score $\widetilde{y}$, and execution record $h$, it first audits the resulting solution:
\begin{equation}
    \mathcal{F} = \phi_{\mathrm{audit}}(x,z,\widetilde{y},h)
    \subseteq
    \{
        \mathrm{trivial},
        \mathrm{task\_mismatch},
        \mathrm{idea\_mismatch},
        \mathrm{reward\_hacking}
    \},
\label{eq:audit-decision}
\end{equation}
where $\mathcal{F}=\varnothing$ denotes a valid result.
The Auditor checks whether the implementation follows the intended idea, satisfies the task requirements, and obtains its score without exploiting the verifier.

The audit outcome determines how the result enters the search state.
If trivial, task\_mismatch or reward\_hacking is detected, the result is discarded and excluded from both the evaluation history and lesson extraction, although the execution still consumes the experimental budget. 
If the only issue is idea\_mismatch, the Auditor reconstructs the input idea and updates the evaluation history accordingly.
The score and extracted lessons are then associated with the updated idea.

This audit-and-align procedure addresses the idea--solution integrity challenge by creating a feedback cycle between the two stages:
\text{Idea}
$\rightarrow$
\text{Implementation}
$\rightarrow$
\text{Audit}
$\rightarrow$
Align Idea to Implementation.
Thus, subsequent search decisions are grounded in the mechanism \textit{actually} evaluated rather than the initially misaligned one.

\subsection{Resource Planner}
\label{sec:resource-planner}

The Resource Planner dynamically distributes a fixed total number of Solver branches across search iterations. 
First of all, given an execution budget $N$ and $B_{\mathrm{tot}}$ total branches, each branch receives a solution execution budget of $N_{\mathrm{branch}}=\lfloor N/B_{\mathrm{tot}}\rfloor$.
Then, the planner controls the trade-off between parallel breadth and sequential adaptivity: wider iterations with larger $B_t$ evaluate more ideas concurrently, whereas narrower iterations enable more frequent updates from experimental feedback.

Let $b_t=\sum_{j=1}^{t-1}B_j$ denote the number of branches already dispatched.
Given the current search state $\mathcal{S}_t$, the remaining branches, and maximum parallelism $B_{\max}$, the planner selects
\begin{equation}
    \Pi_t
    \equiv
    B_t
    =
    \phi_{\mathrm{plan}}
    \left(
        \mathcal{S}_t,
        B_{\mathrm{tot}}-b_t,
        B_{\max}
    \right),
    \qquad
    1\leq B_t\leq
    \min\{B_{\max},B_{\mathrm{tot}}-b_t\}.
\label{eq:resource-plan}
\end{equation}
The resulting search iteration is
\begin{equation}
    I
    =
    \min\left\{
        i:
        \sum_{t=1}^{i}B_t=B_{\mathrm{tot}}
    \right\},
    \qquad
    B_{\mathrm{tot}}N_{\mathrm{branch}}\leq N.
\label{eq:adaptive-search-horizon}
\end{equation}
The planner changes the frequency of feedback-driven updates while preserving the total branch and execution budgets.
Thus, the Resource Planner effectively decides when to exhaust the computational budget and terminate the research pipeline, after which the best scoring solution will be chosen as the final output.

\section{Experiments}
\label{sec:experiments}

\noindent\textbf{Baselines.}
We compare \textsc{AIM} with various solution-driven and idea-driven methods.
Solution-driven baselines include evolutionary solution management approaches like EvoX~\citep{liu2026evox}, AdaEvolve~\citep{cemri2026adaevolve}, AIRA~\citep{toledo2026ai}, and MCTS-based search methods like AIRA~(MCTS version).
For the idea-driven baselines, we compare with methods with various idea management scaffolds, including list-based DeepScientist~\citep{weng2026deepscientist}, MCTS-based AI-Scientist-v2~\citep{yamada2025ai}, tree-based Arbor~\citep{jin2026toward}, and elite-preserving beam-search-based ScientistOne~\citep{meng2026scientistone}.
The backbone LLM is Gemini-3.1-Pro-Preview~\citep{team2023gemini} for all methods.

\paragraph{Benchmark Tasks.}
We evaluate our method and baselines on a wide range of tasks from the AutoLab~\citep{xu2026autolab} suite: System Optimization~(Flash Attention, Radix Sort, AES128 Ctr, FFT Rust, and Z-order Range Scan), Model Development~(Moving MNIST World Model, Data Select Ifeval), and CUDA~(Huffman Canonical Decode, NTT Butterfly, and ICP Correspondence Step).
These tasks span different levels of semantic breadth in their viable solution strategies and different degrees of implementation-level optimization depth.
For the system optimization tasks, we impose a budget of at most 300 experiment executions and a hard wall-clock limit of 6 hours; for the model development tasks, we set at most 60 executions and wall-clock limit of 24 hours; for the CUDA tasks, we set at most 60 executions and wall-clock limit of 12 hours.
If either of the budget is exhausted, the process terminates.
Additional details on the baselines, benchmark tasks, and implementation setup are provided in Appendix~\ref{apdx:details}.

\subsection{Results}
\label{sec:results}

\begin{table}[t!]
\centering
\caption{\textbf{System Optimization Task Results}. The mean and standard error of three independent runs are reported. All idea-driven approaches take the same research brief as additional context for a fair comparison.}
\label{tab:optimization_results}
\renewcommand{\arraystretch}{1.45} 
\resizebox{\textwidth}{!}{%
\begin{tabular}{l c | c c c c c | c}
\toprule
\textbf{Methods} 
& \textbf{Scaffold} 
& \textbf{Flash Attention} 
& \textbf{Radix Sort} 
& \textbf{FFT Rust} 
& \textbf{AES128 Ctr} 
& \textbf{Z Range Scan} 
& \textbf{Average} \\
\midrule 

\rowcolor{blue!10} 
\multicolumn{8}{l}{{\textit{Solution-driven Approaches}}} \\

\textbf{EvoX}                   
& Evolutionary 
& 55.9 \stderr{3.3} 
& 55.5 \stderr{1.8} 
& 55.2 \stderr{0.1} 
& 63.1 \stderr{0.1} 
& 43.7 \stderr{1.6}
& 54.7 \\

\textbf{AdaEvolve}              
& Evolutionary 
& 85.3 \stderr{7.6}
& 62.4 \stderr{3.4}
& 55.6 \stderr{0.1}
& 65.9 \stderr{1.1}
& 47.3 \stderr{2.3}
& 63.3 \\

\textbf{AIRA}~(Evolutionary)    
& Evolutionary 
& 76.3 \stderr{1.0}
& 66.7 \stderr{2.5} 
& 56.2 \stderr{0.4}
& 62.8 \stderr{0.5} 
& 51.1 \stderr{0.8} 
& 62.6 \\

\textbf{AIRA}~(MCTS)            
& MCTS         
& 77.9 \stderr{0.2}
& 61.3 \stderr{5.8}
& \underline{56.4} \stderr{0.1}
& 62.6 \stderr{0.3}
& 50.4 \stderr{1.3}
& 61.7 \\

\midrule

\rowcolor{blue!10} 
\multicolumn{8}{l}{{\textit{Idea-driven Approaches}}} \\

\textbf{DeepScientist} 
& List 
& 72.3 \stderr{3.1}
& 66.5 \stderr{0.4}
& 54.0 \stderr{0.7}
& 65.1 \stderr{1.9}
& \underline{51.4} \stderr{1.9}
& 61.9 \\

\textbf{AI-Scientist-v2} 
& MCTS 
& 55.0 \stderr{3.5}
& 65.7 \stderr{0.5}
& 55.1 \stderr{1.1} 
& 66.3 \stderr{1.6}
& 42.0 \stderr{0.5}
& 56.8 \\

\textbf{Arbor}~(max depth = 2)   
& Tree 
& 73.3 \stderr{5.6}
& 58.5 \stderr{2.7}
& 54.5 \stderr{0.5}
& 64.2 \stderr{1.5}
& 41.4 \stderr{1.8}
& 58.4 \\

\textbf{Arbor}~(max depth = 3)  
& Tree 
& 77.2 \stderr{3.8}
& 58.2 \stderr{4.2}
& 53.5 \stderr{1.1}
& \textbf{67.2} \stderr{0.1}
& 46.9 \stderr{2.1}
& 60.6 \\

\textbf{ScientistOne}  
& Beam 
& \underline{86.2} \stderr{3.1}
& \underline{68.7} \stderr{1.6}
& 55.9 \stderr{0.4}
& 65.3 \stderr{0.2}
& 50.9 \stderr{1.2}
& 65.4 \\

\midrule

\textbf{\textsc{AIM}} (Ours)        
& Autonomous 
& \textbf{90.5} \stderr{0.8}
& \textbf{69.5} \stderr{2.3}
& \textbf{56.5} \stderr{0.5}
& \underline{66.4} \stderr{0.4}
& \textbf{52.1} \stderr{1.1}
& \textbf{67.0} \\




\bottomrule
\end{tabular}%
}
\end{table}

\paragraph{Results on System Optimization Tasks.}
Table~\ref{tab:optimization_results} compares \textsc{AIM} with solution-driven and idea-driven baselines across five systems-optimization tasks. 
\textsc{AIM} achieves the highest average score of $67.0\%$, outperforming the strongest idea-driven baseline, ScientistOne ($65.4\%$), by $1.6$ points and the strongest solution-driven baseline, AdaEvolve ($63.3\%$), by $3.7$ points. 
The largest gain occurs on Flash Attention, where \textsc{AIM} exceeds ScientistOne and AdaEvolve by $4.3$ and $5.2$ points, respectively.
Overall, these results show that \textsc{AIM} performs robustly across tasks with different search characteristics.

\paragraph{Results on Model Development \& CUDA Tasks.}
In Table~\ref{tab:gpu_tasks}, we show the results on the long-horizon tasks, comparing with the strongest baseline configuration within each scaffold based on average performance in Table~\ref{tab:optimization_results}.
In the table, \textsc{AIM} achieves the highest mean score on three of the five tasks: Moving MNIST World Model, Huffman Canonical Decode, and NTT Butterfly, improving over ScientistOne
by $4.2$, $0.4$, and $2.1$ points, respectively.
On Data Selection IFEval, \textsc{AIM} achieves $61.8$, outperforming the evaluated idea-driven baselines but falling below AdaEvolve~($82.7$).
This strict advantage of AdaEvolve on the Data Selection IFEval task may reflect the task's nature of a smooth, local-search-friendly landscape where score is dominated by fine-tuning a single heuristic recipe~(keyword filters, length caps, source balance) rather than by exploring qualitatively different strategies. 
Its mutation loop compounds refinements against a persistent best program, which we conjecture gives it an advantage over other baselines.

\begin{table}[t!]
\centering
\caption{\textbf{Model Development \& CUDA Task Results}. All idea-driven approaches take the same research brief as additional context for a fair comparison. Average scores are computed over all five tasks and omitted for methods with missing results. Missing results with `--' is because the baseline methods could not handle the Moving Mnist World Model task that requires multiple file outputs.}
\label{tab:gpu_tasks}
\renewcommand{\arraystretch}{1.7} 
\setlength{\tabcolsep}{2pt}
\resizebox{\textwidth}{!}{%
\begin{tabular}{l c | c c c c c | c}
\toprule
\textbf{Methods} 
& \textbf{Scaffold} 
& \textbf{MM World Model} 
& \textbf{Data Select IE} 
& \textbf{Huffman Dec.} 
& \textbf{NTT Butterfly} 
& \textbf{ICP Corr. Step}
& \textbf{Average} \\
\midrule 

\rowcolor{blue!10} 
\multicolumn{8}{l}{\textit{Solution-driven Approaches}} \\

\textbf{AdaEvolve}
& Evolutionary
& --
& \textbf{82.7} \stderr{7.1}
& 23.3 \stderr{1.5} 
& 54.9 \stderr{0.9}
& 18.4 \stderr{11.8}
& 44.8 \\

\textbf{AIRA}~(MCTS)
& MCTS
& --
& 38.4 \stderr{5.9}
& 24.3 \stderr{10.8}
& 51.6 \stderr{6.8}
& \underline{54.5} \stderr{0.9}
& 42.2 \\

\midrule

\rowcolor{blue!10} 
\multicolumn{8}{l}{\textit{Idea-driven Approaches}} \\

\textbf{DeepScientist}
& List
& 14.6 \stderr{8.7}
& 19.4 \stderr{19.4}
& 35.5 \stderr{3.7}
& 42.2 \stderr{7.1}
& \textbf{55.7} \stderr{0.5}
& 33.5 \\

\textbf{Arbor}~(max depth = 3)
& Tree
& 33.7 \stderr{11.9}
& 47.0 \stderr{16.1}
& 37.0 \stderr{2.3}
& 47.9 \stderr{5.0}
& 40.8 \stderr{10.2}
& 41.3 \\

\textbf{ScientistOne}
& Beam 
& 57.3 \stderr{6.3}
& 45.1 \stderr{1.1}
& \underline{43.0} \stderr{2.6}
& \underline{56.9} \stderr{1.3}
& 52.1 \stderr{1.0}
& 50.9 \\

\midrule

\textbf{\textsc{AIM}} (Ours)
& Autonomous 
& \textbf{61.5} \stderr{0.9}
& \underline{61.8} \stderr{3.2}
& \textbf{43.4} \stderr{1.3} 
& \textbf{59.0} \stderr{1.3}
& 53.5 \stderr{0.5}
& \textbf{55.8} \\

\bottomrule
\end{tabular}%
}
\end{table}

\paragraph{Time Efficiency.}
Figure~\ref{fig:efficiency-fa} compares the mean best-so-far score over wall-clock time on the Flash Attention task. 
\textsc{AIM} reaches the final score levels of all competing baselines within approximately the first 1--2 hours, whereas the baselines require between $2.2$ and $5.5$ hours to attain those scores. 
Moreover, \textsc{AIM} reaches its best score of $90.5\%$ after $3.3$ hours, exceeding  AdaEvolve's $85.3\%$ at $5.5$ hours and ScientistOne's $86.2\%$ at $3.5$ hours.
Thus, \textsc{AIM} not only discovers a better solution, but also reaches competitive performance substantially earlier in the search, up to 3.1$\times$ faster compared to the strongest baseline, ScientistOne.
Plots for the rest of the tasks are in Appendix~\ref{apdx:full-plots}.

\begin{figure}[t]
    \centering
    \begin{minipage}[t]{0.59\linewidth}
        \centering
        \includegraphics[
            width=\linewidth
        ]{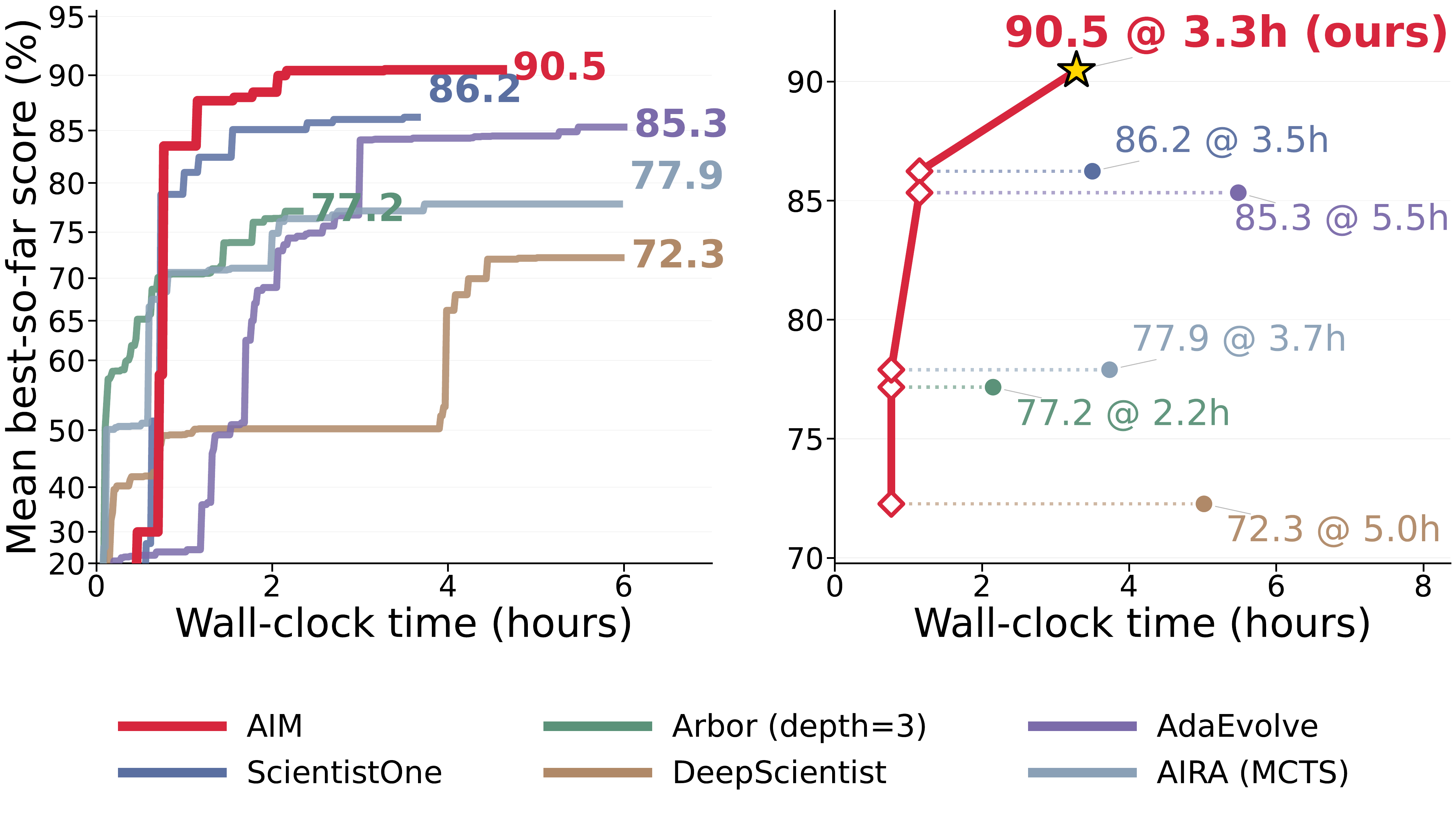}
        \caption{Time Efficiency for Discovery}
        \label{fig:efficiency-fa}
    \end{minipage}
    \hfill
    \begin{minipage}[t]{0.4\linewidth}
        \centering
        \includegraphics[
            width=\linewidth
        ]{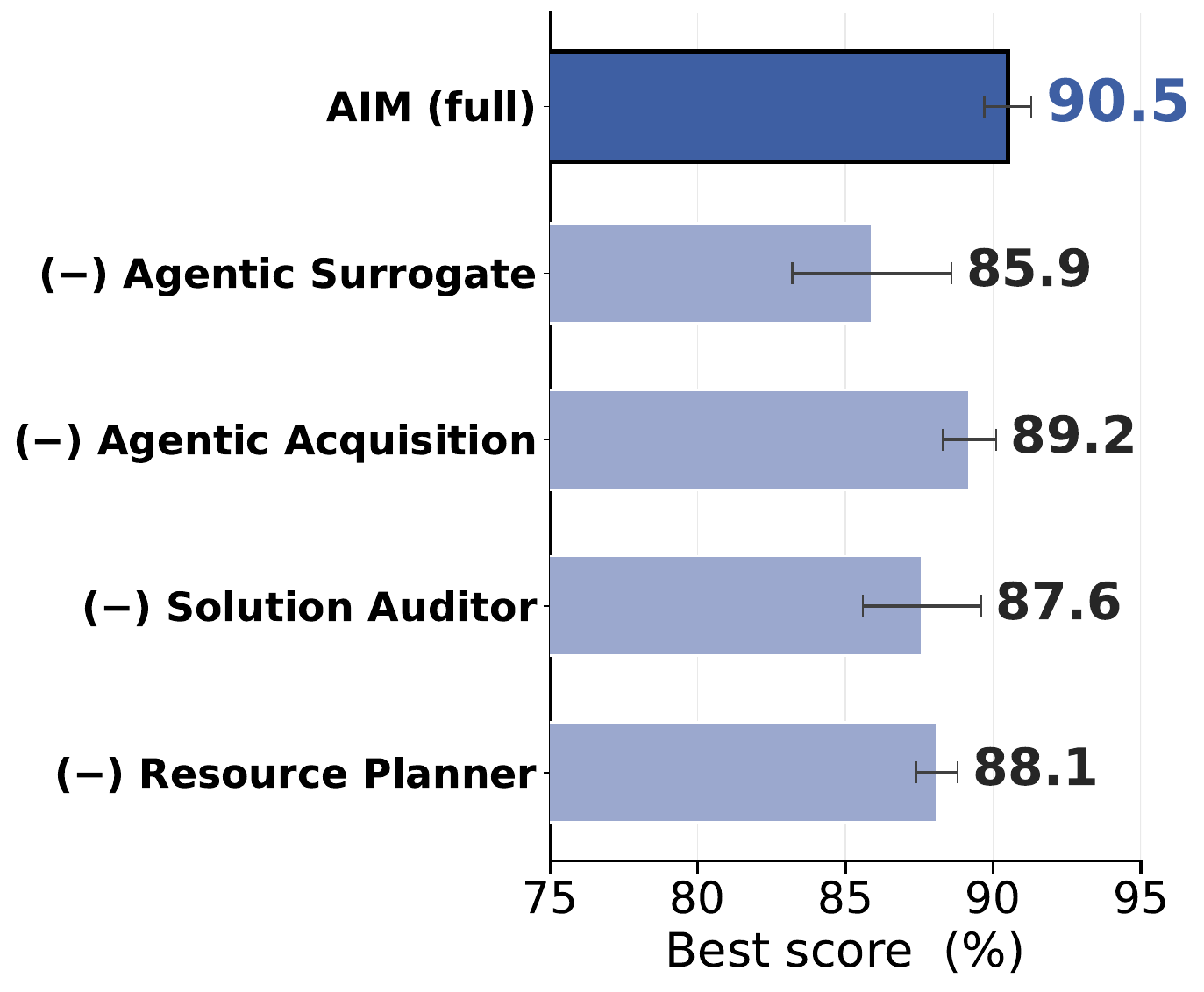}
        \caption{Ablation Studies}
        \label{fig:ablation}
    \end{minipage}
\end{figure}

\subsection{Ablation Studies}
\label{sec:ablation}
We ablate each major component of \textsc{AIM} on Flash Attention, as shown in Figure~\ref{fig:ablation}.
Without the Agentic Surrogate, Organize/Estimate are removed, and Dispatch receives only a shared, unstructured list of ideas.
Without Agentic Acquisition, the LLM directly selects ideas from the clustered and ranked pool without explicitly assigning explore/exploit actions. 
Without the Solution Auditor, all flagged solutions are retained for subsequent lessons and idea generation, idea mismatches are not reconstructed, and audit flags are ignored when computing the final score. 
Finally, removing the Resource Planner replaces dynamic allocation with the fixed $5\times5$ execution schedule.

Removing any component reduces performance from the full model's $90.5\%$. 
The largest degradation occurs without the Agentic Surrogate ($85.9\%$; $-4.6$ points), demonstrating the importance of explicitly organizing the idea space and estimating promisingness. 
The next biggest drop is observed when Solution Auditor is removed, reducing performance to $87.6\%$ ($-2.9$), showing that it is critical to watch out for invalid or misattributed evidence that can propagate through later iterations. 
Also, the fixed resource schedule reaches $88.1\%$ ($-2.4$), while removing explicit acquisition actions yields $89.2\%$ ($-1.3$). 
These results support all the central design choices of \textsc{AIM}: structured idea management, grounded and adaptive idea selection, and reliable idea--solution alignment.

\subsection{Qualitative Examples}

Here, we provide an actual case of \textsc{AIM}'s idea selection process.
In the example below, we show the example from the third iteration of the Data Select IFEval, which is a task about how to select the right data subset for LLM finetuning for a target benchmark, IFEval~\citep{zhou2023instruction}.

First, the Agentic Surrogate's Organizer partitions the 33-idea pool into four clusters, and the Estimator ranks ``metadata stratification'' first, grounding on the strongest performance of its member idea~(0.377), while explicitly demoting ``generative-probing \& LLM-as-a-judge" on the basis of distilled lessons from earlier failures. 
In the second part, the Agentic Acquisition's Dispatch proceeds with determining the \texttt{Explore}/\texttt{Exploit} actions on the cluster-level and idea-level.
For instance, Branch 0 and 1 are both assigned the cluster-level \texttt{Exploit}, but differs in the idea-level actions.
The two branches each choose a high-ranking idea and a low-ranking idea accordingly.
The outcome illustrates why both actions are retained: the exploited selection regressed to 0.119, while the explored pick matched the trailing best score (0.377).
More qualitative examples are in Appendix~\ref{apdx:morequal}
\clearpage
\begin{tcolorbox}[promptbox={Qualitative Example -- Data Select IFEval~(iteration 3)}]
\label{qual}
\footnotesize
\noindent\textbf{(1) Clusters and rank estimates} \quad {\scriptsize pool = 33 ideas}

\smallskip
\setlength{\tabcolsep}{3pt}\renewcommand{\arraystretch}{1.05}
\begin{tabular}{@{}c c l c c c@{}}
\toprule
Rank & Cl. & Label & \#ideas & \#evaluated & best score\\
\midrule
\textbf{1} & [3] & Metadata Stratification \& Representations & 7 & 1 & \textbf{0.377} \\
\textbf{2} & [2] & Deterministic Heuristics \& Lexical Diversity & 9 & 3 & 0.285 \\
\textbf{3} & [0] & Base-Model IFD \& Curriculums & 8 & 2 & 0.193 \\
\textbf{4} & [1] & Generative Probing \& LLM-as-a-Judge & 9 & 2 & 0.137 \\
\bottomrule
\end{tabular}

\smallskip
\noindent\textit{``Cluster 3 achieves the highest evaluated score (0.3770) and aligns with strong lessons advocating for metadata distribution and length proxies. [...] Clusters 0 and 1 have much weaker top scores, with lessons explicitly warning against the generative probing approaches found in Cluster 1.''}

\tcbline

\noindent\textbf{(2) Actions and idea selection}

\smallskip
\noindent\textbf{b0} $\cdot$ cluster [3] \colorbox{orange!20}{\tiny\textsc{exploit}} $\cdot$ idea \colorbox{orange!20}{\tiny\textsc{exploit}} $\rightarrow$ \textbf{Source-Balanced IO-Length Stratification} {\tiny\texttt{rank 1/6,}}\\
\textit{``A direct refinement of the best-performing idea so far: it iterates on the successful source-balancing strategy by incorporating output length to filter out terse responses.''}

\smallskip
\noindent\textbf{b1} $\cdot$ cluster [3] \colorbox{orange!20}{\tiny\textsc{exploit}} $\cdot$ idea \colorbox{teal!15}{\tiny\textsc{explore}} $\rightarrow$ \textbf{Unsupervised TF-IDF + KMeans Stratification} {\tiny\texttt{rank 5/6}}\\
\textit{``To execute an explore action within this top-performing cluster, we select a low-rank idea (5/6) that introduces a completely different mechanism: data-driven semantic boundaries instead of native `source' metadata.''}

\smallskip
\noindent\textbf{Outcome:} \textbf{b0} 0.119 $\cdot$ \textbf{b1} 0.377 (ties best) $\cdot$ \textbf{b2} 0.230
\end{tcolorbox}

\paragraph{Following additional experiments and analyses are deferred to the Appendix:} 
$\Circled{1}$ solver substitution using Claude Code~(Appendix~\ref{apdx:claude-code});
$\Circled{2}$ in-depth ablation study on the Agentic Surrogate component~(Appendix~\ref{apdx:surrogate});
$\Circled{3}$ qualitative analysis of the Organize operator~(Appendix~\ref{apdx:organize});
$\Circled{4}$ preciseness of the Estimate operator's ordinal predictions~(Appendix~\ref{apdx:estimate});
$\Circled{5}$ distribution of the Dispatch operator's exploration--exploitation actions~(Appendix~\ref{apdx:dispatch});
$\Circled{6}$ distribution of the Expand operator's generation modes~(Appendix~\ref{apdx:expand});
$\Circled{7}$ effectiveness of the Solution Auditor's idea reconstruction mechanism~(Appendix~\ref{apdx:auditor}).

\section{When to Search Over Ideas? A Formal Understanding}
\label{sec:theory}

We have distinguished two paradigms for automated research.
Idea-search approaches first search over semantic research directions and then instantiate selected ideas as executable solutions, whereas solution-driven approaches directly propose and refine executable solutions.
We provide a theoretical view of when the idea-level abstraction can be useful through two questions:
(1) \textit{which paradigm can provide broader coverage of the solution space?}, and
(2) \textit{when is this additional coverage more valuable for finding competitive directions?}

\paragraph{\textcolor{mydarkblue}{(1) \textit{Which paradigm has broader coverage of the solution space?}}}
We begin by setting assumptions and defining the Semantic Coverage:
\begin{assumption}[\textbf{Semantic Decomposition}]
\label{assump:semantic-decomposition}
Let each executable solution $z\in\mathcal{Z}$ be associated with an underlying
research direction $\pi(z)\in\mathcal{X}$, where $\mathcal{X}$ is the finite set of ideas.
Then, for each idea $x\in\mathcal{X}$, its corresponding semantic solution
region is $\mathcal{Z}_x = \{z\in\mathcal{Z}:\pi(z)=x\}$.
\end{assumption}

\begin{assumption}[\textbf{Faithful Realization}]
\label{assump:faithful-realization}
For every selected idea $x\in\mathcal{X}$, the implementation process can
produce an executable solution $z=g(x)$ that faithfully realizes the selected
idea: $\pi(g(x))=x$.
\end{assumption}

\begin{definition}[\textbf{Semantic Coverage}]
\label{def:semantic-coverage}
For a set of evaluated solutions $\mathcal{E}\subseteq\mathcal{Z}$, define its Semantic Coverage as the number of distinct research directions represented in it: $C(\mathcal{E}) = \left|\{\pi(z):z\in\mathcal{E}\}\right|$.
\end{definition}

Assumption~\ref{assump:semantic-decomposition} and Definition~\ref{def:semantic-coverage} provides a useful way to view the difference between the two search paradigms.
For instance, consider a search graph whose nodes are executable solutions.
A solution-driven search trajectory may move through $z_1 \rightarrow z_2 \rightarrow z_3 \rightarrow z_4$ $(z_i \in \mathcal{Z})$, while the corresponding semantic directions are $x_1 \rightarrow x_1 \rightarrow x_1 \rightarrow x_2$ $(x_i \in \mathcal{X})$.
In such cases, although four distinct solutions have been evaluated, this trajectory covers only two semantic directions.
More generally, the mapping $\pi:\mathcal{Z}\rightarrow\mathcal{X}$ groups solutions according to their underlying research direction.
Formally, $\pi$ induces the equivalence relation $z \sim z'$ if and only if $\pi(z)=\pi(z')$, thereby contracting solutions that instantiate the same idea into a common semantic class.
This yields a quotient view of the solution space, in which solution-level trajectories are represented by the distinct semantic directions they traverse.

In addition, while Assumption~\ref{assump:faithful-realization} may seem a bit idealistic, we argue that the Solution Auditor ensures  the evidence that the agent observes is faithful and reliable.
Figure~\ref{fig:auditor} in Appendix~\ref{apdx:auditor} supports this view, in that it substantially reduces the idea-mismatch cases after the Solution Auditor's idea reconstruction takes place.
Furthermore, our discussion on the practical considerations in  Appendix~\ref{apdx:practical-considerations} suggests that a hybrid design of solution-driven and idea-driven approaches might be desired to ensure faithful implementation of ideas.
Based on these grounds, we formalize the attainable semantic coverage of the two search paradigms:

\begin{proposition}[\textbf{Attainable Semantic Coverage}]
\label{thm:semantic-coverage}
Let $C^\text{idea}$ be the semantic coverage of an extreme idea-driven search procedure that renders $N$ distinct experiment executions to $N$ distinct research ideas, and $C$ be the semantic coverage of any search procedure that evaluates at most $N$ executable solutions, including solution-driven approaches. Then, $C \leq C^\text{idea}$.
\end{proposition}

Note that Proposition~\ref{thm:semantic-coverage} does not imply that solution-driven search must have lower semantic coverage.
A solution-driven method can attain the same bound if its code-level search also produces semantically distinct solutions.
The distinction is that idea-driven search makes semantic breadth an explicit and directly controllable property of the search process.
By selecting distinct ideas before implementation, the framework can deliberately allocate its expensive experiment budget across distinct research directions, rather than obtaining semantic diversity only indirectly through solution-level transitions.
Whether maximizing such semantic breadth is beneficial, however, depends on the trait of the research task, which we discuss next.

\paragraph{\textcolor{mydarkblue}{(2) \textit{When is broader semantic coverage useful?}}}
We next ask when the additional coverage is useful for solving the research task.
Intuitively, breadth should matter most when there are many possible research directions but only a small subset of them can achieve near-optimal performance.

\begin{definition}[\textbf{Competitive Semantic Directions}]
\label{def:competitive-directions}
For a target task $\tau$, let there be $K$ materially distinct research ideas $\mathcal{X}_\tau=\{x_1,\ldots,x_K\}\subseteq\mathcal{X}$.
Each direction has an attainable value~(e.g., performance measure), $V(x) = \max_{z\in\mathcal{Z}_x} v(z)$, and let $V^\star = \max_{x\in\mathcal{X}_\tau} V(x)$.
For $\varepsilon\geq0$, define the set of $\varepsilon$-optimal directions as $\mathcal{G}_\varepsilon = \left\{x\in\mathcal{X}_\tau: V(x)\geq V^\star-\varepsilon\right\}$, and $G_\varepsilon = |\mathcal{G}_\varepsilon|$.
\end{definition}

\begin{assumption}[\textbf{Competitive Direction Exchangeability}]
\label{assump:semantic-uncertainty}
Before the search evaluates a direction, the identities of the
$G_\varepsilon$ $\varepsilon$-optimal directions are uniformly distributed among the $K$ admissible directions. 
Conditional on having evaluated only non-$\varepsilon$-optimal directions, the optimal directions remain exchangeable among the unexplored directions.
\end{assumption}

\begin{proposition}[\textbf{Semantic Coverage Sufficient Condition}]
\label{thm:coverage-benefit}
Under Definition~\ref{def:competitive-directions} and
Assumption~\ref{assump:semantic-uncertainty}, a search procedure with semantic coverage $C$ discovers at least one $\varepsilon$-optimal direction with probability $P_\varepsilon(C) = 1 - \binom{K-G_\varepsilon}{C}/\binom{K}{C}.$
Then, a sufficient condition of the semantic coverage $C \geq K/G_\varepsilon \ln \left(1/\delta\right)$ guarantees a success probability of $1-\delta$, if the right-hand side is feasible.
\end{proposition}

\paragraph{\textcolor{mydarkblue}{\textbf{Implication: Tasks with larger effective semantic breadth favor idea-driven search.}}}
Combining Propositions~\ref{thm:semantic-coverage} and~\ref{thm:coverage-benefit}, tasks with many materially distinct research directions but relatively few competitive ones benefit more from broad semantic coverage. 
In such tasks, idea-driven search can be advantageous because it makes semantic breadth explicit and controllable. 
Conversely, when $K/G_\varepsilon$ is small, \textit{i.e.}, the task admits only a few meaningful directions or many directions are similarly competitive, the value of additional semantic coverage is limited, and solution-driven and idea-driven approaches may perform similarly, leaving implementation-level optimization as a potentially more important determinant of performance.
Proofs are in Appendix~\ref{apdx:proofs}.

\begin{figure}[t!]
    \centering
    \begin{tikzpicture}
        \begin{axis}[
            width=0.72\linewidth,
            height=0.48\linewidth,
            xlabel={Effective Semantic Breadth $B_\varepsilon = K/G_\varepsilon$},
            ylabel={Required Semantic Coverage $C$},
            xmin=1,
            xmax=20,
            ymin=0,
            ymax=65,
            xtick={1,5,10,15,20},
            ytick={0,10,20,30,40,50,60},
            grid=major,
            legend style={
                at={(0.03,0.97)},
                anchor=north west,
                draw=none,
                fill=none,
            },
            legend cell align={left},
        ]

        \addplot[
            very thick,
            domain=1:20,
            samples=100,
        ]
        {x*ln(2)};
        \addlegendentry{$P_\varepsilon \geq 0.50$}

        \addplot[
            very thick,
            dashed,
            domain=1:20,
            samples=100,
        ]
        {x*ln(10)};
        \addlegendentry{$P_\varepsilon \geq 0.90$}

        \addplot[
            very thick,
            dotted,
            domain=1:20,
            samples=100,
        ]
        {x*ln(20)};
        \addlegendentry{$P_\varepsilon \geq 0.95$}

        \end{axis}
    \end{tikzpicture}

    \caption{
    Semantic coverage required to achieve a fixed probability of discovering
    an $\varepsilon$-optimal research direction.
    As competitive directions become sparser, the effective semantic breadth
    $B_\varepsilon=K/G_\varepsilon$ increases, requiring proportionally broader
    semantic coverage.
    The curves show the sufficient condition
    $C \geq B_\varepsilon \log(1/\delta)$ for different target success
    probabilities $1 - \delta$.
    }
    \label{fig:semantic-breadth-coverage}
\end{figure}
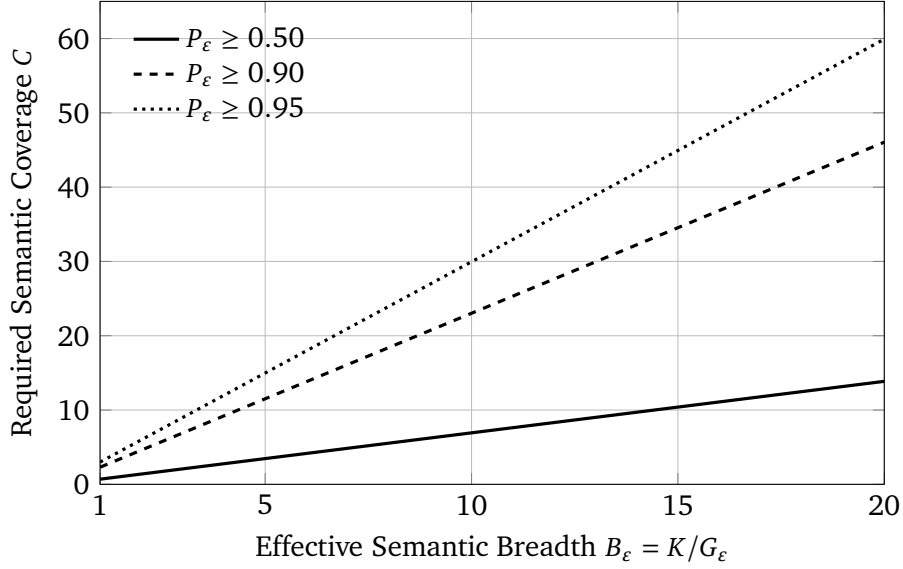
\paragraph{\textcolor{mydarkblue}{Theoretical Bound Visualization.}}
Figure~\ref{fig:semantic-breadth-coverage} illustrates the semantic coverage requirement derived in Proposition~\ref{thm:coverage-benefit}. 
Let $B_\varepsilon=K/G_\varepsilon$ denote the effective semantic breadth, where $K$ is the number of admissible semantic directions and $G_\varepsilon$ is the number of $\varepsilon$-optimal directions. 
For a target success probability $1 - \delta$, the sufficient coverage condition can be written as
\begin{equation}
C \geq B_\varepsilon \ln\left(\frac{1}{\delta}\right).
\end{equation}
The required semantic coverage therefore grows linearly with effective semantic breadth.
When competitive directions are sparse, such that $G_\varepsilon$ is small relative to $K$, a search procedure must examine substantially more distinct directions to maintain the same probability of discovering an $\varepsilon$-optimal one. 
The requirement also becomes steeper as the target success probability increases: at $B_\varepsilon=20$, for example, the sufficient coverage is approximately $14$, $46$, and $60$ directions for target probabilities of $0.50$, $0.90$, and $0.95$, respectively. 
Thus, broader semantic coverage is particularly valuable for tasks with sparse competitive directions and when a high probability of success is required.

\subsection{Empirical Support}
Our empirical analysis in Figure~\ref{fig:theory} is consistent with the implication. 
As a rough proxy for the breadth of the explored solution space, we embed all intermediate solutions from each method using Gemini-Embedding-001~\citep{lee2025gemini}, project the embeddings into a two-dimensional PCA space, and measure the area of their convex hull.
The box-and-whisker plots compare the distribution of the hull area, and the PCA hull plots visualize example cases.
For tasks like Flash Attention, which admit diverse algorithmic directions but only a small subset of them yield substantial speedups, idea-driven search approaches explore substantially broader regions than solution-driven approaches, as reflected by larger convex-hull areas. 
In contrast, for tasks like Radix Sort, where the high-level algorithm is largely fixed and progress depends primarily on code-level optimization, the gap in convex-hull area between the two paradigms becomes considerably smaller. 

\begin{figure}
    \centering
    \includegraphics[width=\linewidth]{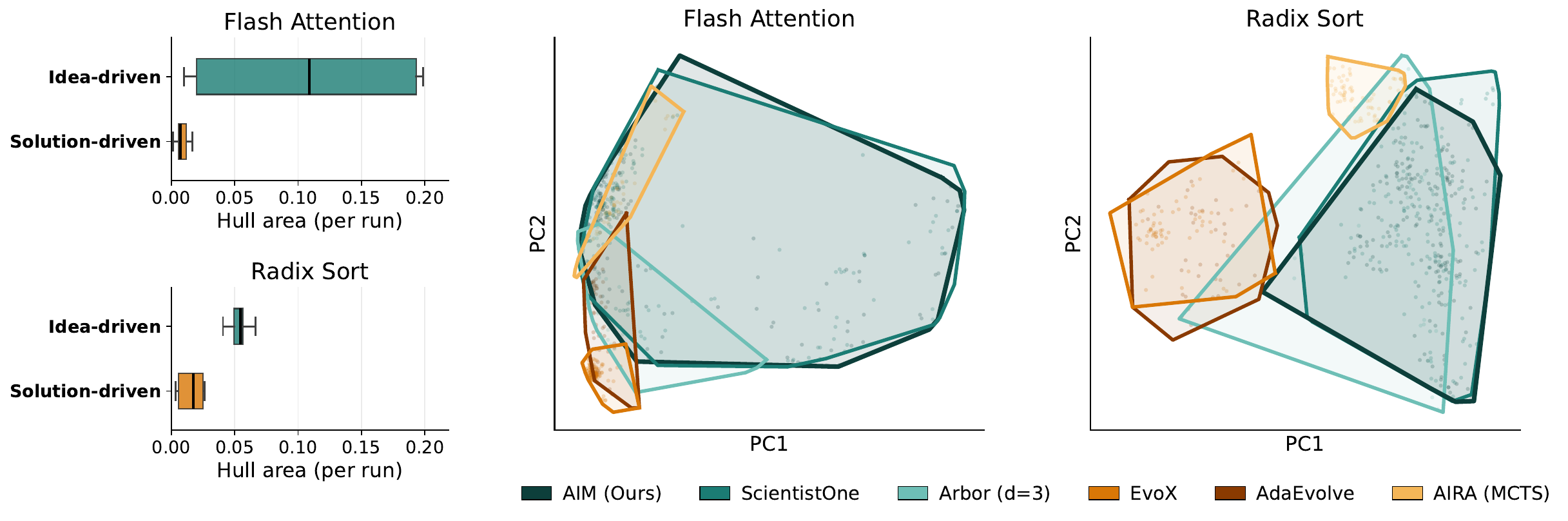}
    \caption{Idea-driven approaches generally show broader coverage of solutions in the embedding space. Embedding cosine similarity-based supplementary analysis is in Appendix~\ref{apdx:cosine}. Plots for all tasks are in Appendix~\ref{apdx:boxplots} and Appendix~\ref{apdx:convex-hull}.}
    \label{fig:theory}
\end{figure}

\subsection{Practical Considerations}
\label{apdx:practical-considerations}

Our theoretical analysis isolates the benefit of semantic breadth by assuming \emph{faithful realization} (Assumption~\ref{assump:faithful-realization}): a selected idea can be instantiated as a solution that faithfully reflects its intended research direction.
This assumption is useful for understanding the coverage advantage of idea-driven search, but it could be slightly idealized.
In practice, solver agents may produce incomplete or incorrect implementations, or realize a solution that only partially matches the intended idea.
Even when the implementation is faithful, some tasks inherently require substantial solution-level refinement before the value of a research direction can be assessed.
For example, machine learning systems may require hyperparameter tuning, while systems optimization tasks may depend on low-level implementation choices that cannot be fully specified at the idea level.

These considerations suggest that practical research agents lie on a continuum between purely idea-driven and purely solution-driven search. 
At one extreme, allocating each experiment execution to a new idea maximizes semantic breadth, but may under-invest in realizing and optimizing each direction.
At the other extreme, repeatedly refining the same implementation can exploit fine-grained execution feedback, but reduces the budget available for exploring alternative research directions.
The appropriate balance therefore depends on the task: problems with many meaningfully different high-level approaches may benefit more from semantic exploration, whereas tasks with relatively fixed strategies but substantial implementation-level optimization may benefit more from solution refinement.

Our framework, to be precise, adopts a hybrid design. 
Although search decisions are made over explicit research ideas, each solver branch is allowed to iteratively refine its implementation using execution and verifier feedback.
Thus, the framework retains the semantic organization of idea-driven search while still incorporating the local optimization behavior characteristic of solution-driven approaches.
This also motivates our idea--solution auditing mechanism, which explicitly checks whether the final implementation remains aligned with the selected idea before its evaluation is attributed back to the idea-level search process.
More broadly, the optimal balance between semantic exploration and implementation refinement is likely task-dependent.
An important direction for future work is therefore to adapt this balance dynamically based on observed implementation difficulty, solver fidelity, and the semantic structure of the task, rather than fixing the degree of idea-driven versus solution-driven behavior in advance.

\section{Conclusion}
\label{sec:conclusion}

In this work, we introduced \textsc{AIM}, a fully autonomous framework for managing and searching research ideas in idea-driven automated research. 
\textsc{AIM} dynamically organizes an evolving idea pool, explicitly estimates the promise of research directions, adaptively balances exploration and exploitation, and audits implementations before incorporating their outcomes into subsequent search. 
Across 10 AutoLab tasks, \textsc{AIM} achieves the best overall performance, while reaching strong solutions substantially faster compared to baselines. 
Our theoretical analysis further clarifies when idea-driven search can be advantageous: explicit idea-level allocation makes semantic coverage controllable, and broader coverage becomes more valuable when competitive directions are sparse among many plausible alternatives. 
Together, these results establish agentic idea management as a transparent and robust approach to budget-constrained automated research.

\section*{AI Use Statement}

In this work, we used generative AI tools for polishing the writings and generating plots.
We have not used generative AI tools to run the experiments for the study outside the agent methods being evaluated. We have reviewed all AI-assisted work. 
We manually checked the writing faithfully conveys our intended content, and if the plots correctly express values.
We take responsibility for the final content of this work,
including text, claims or artifacts produced with the aid of generative AI.

\section*{Ethics Statement}
This work does not involve human subjects, personal data, or user-facing deployment. 
Nevertheless, autonomous research agents generate and execute code, which may introduce risks such as insecure implementations, verifier exploitation, or unintended task behavior. 
We mitigate these risks by conducting all experiments in isolated, containerized benchmark environments with fixed verifiers and by using the Solution Auditor to identify task mismatches and reward-hacking behavior. 
No generated solutions are deployed in real-world systems. 
These safeguards cannot eliminate all risks, and human review remains necessary before applying autonomous research systems to consequential scientific or engineering settings.

\section*{Reproducibility Statement}
We provide pseudocode for the complete \textsc{AIM} pipeline in Algorithm~\ref{alg:aim}. 
Appendix~\ref{apdx:details} describes the evaluated benchmarks, baseline configurations, scoring procedure, execution and wall-clock budgets, branch-wise resource allocation, and implementation hyperparameters. 
The prompt templates for all LLM-driven operators are also included in the appendix. 
We report results over three independent runs using a common backbone model and evaluation environment, with means and standard errors. 


\bibliographystyle{abbrvnat}
\nobibliography*
\bibliography{ref}

\newpage
\appendix
\addcontentsline{toc}{section}{Appendix}
\part{Appendix} 
\parttoc 
\clearpage

\section{Related Works}
\label{apdx:related-works}

\paragraph{Automated Research with LLM Agents.}
Recent work has expanded LLM agents from supporting individual scientific tasks to conducting increasingly complete research workflows.
At the ideation stage, ResearchAgent~\citep{baek2025researchagent} grounds research idea generation in retrieved literature and iterative feedback, while HARPA~\citep{vasu2025harpa} develops literature-grounded, testable hypotheses and refines them using experimental evidence.
Complementarily, \citet{si2025novelideas} conduct a large-scale expert study of the novelty and quality of LLM-generated research ideas.
Broader research systems automate multiple stages of the scientific process: CodeScientist~\citep{jansen2025codescientist} develops a coding-based framework for semi-automated scientific discovery; The AI Scientist~\citep{lu2024ai} spans ideation, experimentation, paper writing, and review; and MLR-Copilot~\citep{li2024mlr}, Agent Laboratory~\citep{schmidgall2025agent}, AI-Researcher~\citep{tang2026ai}, and InternAgent~\citep{team2025internagent} similarly automate multi-stage workflows from hypothesis formation to experimental validation and reporting.
These works establish the broader setting of automated research, including the end-to-end pipeline resulting in a publishable paper, while our focus is specifically on the ``discovery" stage of automated research: \emph{how an agent should search under a limited experimental budget}.

\paragraph{Solution-driven Search over Executable Code.}
A prominent class of research agents directly searches over executable artifacts, keeping implementation and search tightly coupled.
AIDE~\citep{jiang2025aide} searches directly in the space of code, while AlphaEvolve~\citep{novikov2025alphaevolve} uses evolutionary code generation for algorithmic discovery.
MLE-Star~\citep{nam2026mle} and DS-Star~\citep{nam2025ds} iteratively improve machine-learning and data-science solutions through targeted refinement, planning, and verification.
More recent systems further develop solution-level search through evolutionary and reflective mechanisms: AdaEvolve~\citep{cemri2026adaevolve} performs adaptive LLM-driven zeroth-order optimization, EvoX~\citep{liu2026evox} introduces meta-evolution for automated discovery, 
AIDE~\citep{jiang2025aide} utilize an MCTS approach for machine learning tasks,
AIRA~\citep{toledo2026ai} studies evolutionary and tree-search strategies for machine-learning research, and RPM~\citep{foster2026ai} introduces a preference model for candidate solutions.
These methods exemplify the \emph{solution-driven} paradigm considered in our work: the executable solution itself remains the primary search object, allowing verifier feedback to be applied directly to subsequent implementation-level refinement.

\paragraph{Idea-driven Search over Research Ideas and Hypotheses.}
A complementary line of work maintains ideas or hypotheses as explicit intermediate objects before delegating their realization to an implementation agent.
The AI Scientist-v2~\citep{yamada2025ai} uses agentic tree search to progressively develop research directions and experiments.
DeepScientist~\citep{weng2026deepscientist} maintains candidate research directions and progressively selects them using a structured exploration mechanism, while AutoDiscovery~\citep{agarwal2026autodiscovery} conducts data-driven open-ended scientific discovery through search guided by Bayesian surprise.
Arbor~\citep{jin2026toward} explicitly organizes hypotheses in a refinement tree, separating hypothesis development from downstream implementation.
MARS~\citep{chen2026mars} performs modular reflective search over candidate ideas and solutions.
ScientistOne~\citep{meng2026scientistone} similarly makes research ideas explicit and uses an elite-preserving search procedure within its Chain-of-Evidence framework, while additionally preserving traceability between scientific claims, experimental evidence, and implementations.

These \emph{idea-driven} systems reveal several design choices for managing the intermediate idea space.
Existing approaches employ structures such as lists~\citep{weng2026deepscientist}, search trees~\citep{yamada2025ai,jin2026toward,agarwal2026autodiscovery}, or beam-style candidate retention~\citep{meng2026scientistone}, together with predefined selection mechanisms such as UCB or tree search.
In contrast, our work focuses on making \emph{idea management itself} an autonomous, evidence-conditioned component of the research process.
AIM dynamically reorganizes the evolving idea pool, explicitly estimates the relative promise of research directions, and adaptively determines where to explore or exploit based on accumulated experimental evidence.
Moreover, because separating ideation from implementation introduces the possibility that a solver may not faithfully realize its intended idea, AIM explicitly audits and reconciles idea--solution mismatches before incorporating results into subsequent search.
Thus, our contribution is complementary to prior idea-driven research agents: rather than introducing another fixed idea-search scaffold, we study how the idea space can be autonomously organized, searched, and maintained as evidence accumulates.

\paragraph{Testbeds for Automated Research.}
In parallel, a growing body of work has made automated experimentation measurable.
MLAgentBench~\citep{huang2024mlagentbench} evaluates agents that modify machine-learning code and iteratively interpret experimental feedback, while MLE-bench~\citep{chan2025mle} evaluates agents on competition-style machine-learning engineering.
RE-Bench~\citep{wijk2025re} studies frontier AI research-and-development capabilities relative to human experts, and MLGym~\citep{nathani2025mlgym} provides diverse environments for evaluating AI research agents.
More recent benchmarks broaden evaluation toward open-ended research problems, including MLRC-Bench~\citep{zhang2026mlrc}, MLR-Bench~\citep{chen2026mlr}, and ResearchGym~\citep{garikaparthi2026researchgym}.
We primarily evaluate on AutoLab~\citep{xu2026autolab}, which provides long-horizon research and engineering tasks with executable verifiers and enables controlled comparison of search strategies in systems optimization.

Overall, prior work has made substantial progress in both direct solution optimization and end-to-end automated research.
Our work builds on this literature by explicitly distinguishing \emph{solution-driven} search over executable solutions from \emph{idea-driven} search over ideas followed by implementation, and focuses on the latter's distinctive challenges: maintaining an evolving semantic representation of candidate ideas, making evidence-grounded decisions about which ideas deserve expensive implementation, and preserving idea--solution integrity throughout the search process.

\clearpage

\section{Agentic Idea Manager Algorithm}
\label{apdx:algorithm}

\begin{algorithm}[h!]
\caption{\textsc{AIM}: Agentic Idea Manager}
\label{alg:aim}
\begin{algorithmic}[1]
\Require Task context $\mathcal{T}$; initial idea pool $\mathcal{P}_1$;
execution budget $N$; total branches $B_{\mathrm{tot}}$;
maximum parallelism $B_{\max}$; time limit $H$.

\State $\mathcal{D}_1,\mathcal{M}_1\gets\varnothing$;
       $b\gets0$;
       $t\gets1$;
       $N_{\mathrm{branch}}\gets\lfloor N/B_{\mathrm{tot}}\rfloor$

\While{$b<B_{\mathrm{tot}}$ \textbf{and} elapsed time $<H$}
 
    \State $B_t\gets
        \begin{cases}
        5, & t=1,\\
        \phi_{\mathrm{plan}}(\mathcal{S}_t,B_{\mathrm{tot}}-b,B_{\max}),
        & t>1
        \end{cases}$
        \Comment{Resource Planner}

    \State $\mathcal{C}_t\gets
        \phi_{\mathrm{org}}(\mathcal{T},\mathcal{P}_t,\mathcal{D}_t,\mathcal{C}_{t-1})$
        \Comment{Agentic Surrogate: Organize}
    \State $R_t\gets
        \phi_{\mathrm{est}}(\mathcal{C}_t,\mathcal{D}_t)$
        \Comment{Agentic Surrogate: Estimate}
       
    \State $\mathcal{Q}_t\gets
        \phi_{\mathrm{disp}}(\mathcal{P}_t,\mathcal{C}_t,R_t,B_t)$
        \Comment{Agentic Acquisition: Dispatch}

    \ForAll{$x\in\mathcal{Q}_t$ \textbf{in parallel}}
        \State $(z_x,y_x,h_x)\gets
            \phi_{\mathrm{solve}}(\mathcal{T},x;N_{\mathrm{branch}})$
    \EndFor

    \State $\mathcal{V}_t\gets
        \phi_{\mathrm{audit}}
        \bigl(\{(x,z_x,y_x,h_x):x\in\mathcal{Q}_t\}\bigr)$
        \Comment{Solution Auditor}

    \State $\mathcal{D}_{t+1}\gets
        \mathcal{D}_t\cup
        \{(x^{+},y):(x^{+},z,y,h)\in\mathcal{V}_t\}$
    \State $\mathcal{M}_{t+1}\gets
        \mathcal{M}_t\cup\phi_{\mathrm{lesson}}(\mathcal{V}_t)$
    \State $\Delta \mathcal{P}_t \gets \phi_{\mathrm{exp}}(\mathcal{P}_t,\mathcal{D}_{t+1},\mathcal{M}_{t+1})$
        \Comment{Agentic Acquisition: Expand}
    \State $\mathcal{P}_{t+1}\gets
        \mathcal{P}_t\cup
        \{x^{+}:(x^{+},z,y,h)\in\mathcal{V}_t\}\cup \Delta \mathcal{P}_t $

    \State $b\gets b+B_t$;
           $t\gets t+1$
\EndWhile

\State \Return $\quad {\arg\max}_{(x,y)\in\mathcal{D}_t} \quad y$

\end{algorithmic}
\end{algorithm}
\section{Experimental Details}
\label{apdx:details}

\subsection{Further Baseline and Benchmark Details}

\paragraph{Baselines and budgets.}
All baseline agents use \texttt{gemini-3.1-pro-preview}~\citep{team2023gemini} as their underlying language model. 
Each method is allowed at most $N$ experiment executions and a wall-clock budget limit is set per run.
For the system optimization tasks, we impose a budget of at most 300 experiment executions and a hard wall-clock limit of 6 hours; for the model development tasks, we set at most 60 executions and wall-clock limit of 24 hours; for the CUDA tasks, we set at most 60 executions and wall-clock limit of 12 hours.
If either of the budget is exhausted, the process terminates.
For baseline methods that support parallel search or execution, we set the maximum parallelism to five workers, to keep it commensurable with \textsc{AIM} which usually spends at most five branches per iteration. 
All idea-driven search baselines receive the same initial research brief. 
We perform three independent runs for each (method, task) pair and report the mean and standard error.

\paragraph{AutoLab tasks.}
We evaluate on ten tasks from AutoLab~\citep{xu2026autolab} spanning three categories.
\emph{System optimization} (CPU): Flash Attention, Radix Sort, FFT Rust, AES128 Ctr, and Z-order Range Scan, which respectively optimize scaled dot-product attention in C, sorting of 50 million unsigned integers in C, a 32,768-point real-valued DFT in Rust, AES-128-CTR encryption of 256 MiB in C, and two-dimensional range-count queries over a Rust spatial index.
\emph{Model development} (GPU): MM World Model, which trains a video world model from scratch on Moving MNIST and is scored by PSNR on 10-step rollouts under a fixed four-hour training budget, and Data Select IE, which selects 5{,}000 training samples from a 50,000-sample pool drawn from 19 sources for LoRA fine-tuning of Qwen2.5-3B-Instruct and is scored by prompt-level strict accuracy on IFEval.
\emph{CUDA kernels} (GPU): Huffman Canonical Decode, which decodes $K$ independent canonical-Huffman bitstreams to their byte payloads; NTT Butterfly, which applies an in-place forward number-theoretic transform over the Goldilocks prime field to batched 64-bit rows; and ICP Correspondence Step, which performs one Iterative-Closest-Point correspondence step, i.e., nearest-neighbor search against a prebuilt KD-tree together with accumulation of the cross-covariance, residual error, and correspondence count.

Each AutoLab task provides a natural-language instruction, a containerized environment, an editable codebase containing a correct but deliberately weak baseline (an inefficient implementation for the optimization and CUDA tasks, or a simple training or selection recipe for the model-development tasks), and a local evaluation script.
During search, an agent may repeatedly edit the implementation, execute it, inspect its score and correctness feedback, and refine subsequent solutions.
The task additionally contains a human-written reference implementation and a held-out verifier.
The reference solution is used only to calibrate the scoring scale and is not exposed to the research agent.

\paragraph{Evaluation and scoring.}
The verifier first checks functional correctness and task-specific constraints; a solution that fails these checks or does not improve upon the baseline receives zero reward.
Valid solutions are scored on one of two scales, depending on the task category.

\emph{Runtime tasks (system optimization and CUDA).}
Because runtime depends on the execution environment, for the system-optimization tasks we run both the baseline and reference implementations on the same local hardware used to evaluate generated solutions.
Let $t_{\mathcal{B}}$, $t_{\mathcal{R}}$, and $t(z)$ denote the runtimes of the baseline, reference, and generated solution $z$, respectively. Following AutoLab, the normalized score is
\begin{equation}
s(z)
=
\begin{cases}
0,
& \text{if $z$ is invalid or } t(z)\geq t_{\mathcal{B}},\\[2mm]
\displaystyle
\operatorname{clip}\!\left(
\frac{1}{2}
\frac{\log\!\left(t_{\mathcal{B}}/t(z)\right)}
     {\log\!\left(t_{\mathcal{B}}/t_{\mathcal{R}}\right)},
0,1
\right),
& \text{otherwise}.
\end{cases}
\label{eq:autolab-score}
\end{equation}
The baseline therefore receives $s=0$, while matching the reference solution gives $s=0.5$.
Solutions outperforming the reference receive scores above $0.5$, up to a maximum of $1.0$.

\emph{Model-development tasks.}
These tasks are scored on a task-specific quality metric $m(z)$ (higher is better) rather than runtime. Following AutoLab, the score interpolates linearly between a baseline anchor $m_{\mathcal{B}}$ and a reference anchor $m_{\mathcal{R}}$:
\begin{equation}
s(z)
=
\begin{cases}
0,
& \text{if $z$ is invalid},\\[2mm]
\displaystyle
\operatorname{clip}\!\left(
\frac{m(z)-m_{\mathcal{B}}}{m_{\mathcal{R}}-m_{\mathcal{B}}},
0,1
\right),
& \text{otherwise},
\end{cases}
\label{eq:autolab-score-modeldev}
\end{equation}
with $(m_{\mathcal{B}}, m_{\mathcal{R}})=(14.0, 20.0)$~dB PSNR for MM World Model and $(0.38, 0.48)$ IFEval accuracy for Data Select IE, which are provided by the benchmark.

We multiply $s(z)$ by $100$ when reporting percentage scores in the main paper.

\subsection{Implementation Details}

\paragraph{Branch-wise execution budgets.}
We distinguish the total number of Solver branches, $B_{\mathrm{tot}}$, from the number of branches executed concurrently at iteration $t$, $B_t$. 
Before each run, the total execution budget is divided equally among the Solver branches:
\begin{equation}
N_{\mathrm{branch}}
=
\left\lfloor
\frac{N}{B_{\mathrm{tot}}}
\right\rfloor.
\label{eq:branch-budget}
\end{equation}
If we use $N=300$ and $B_{\mathrm{tot}}=25$, it gives each branch a fixed budget of $12$ experiment executions. 
A branch may reason over multiple rounds, modify its intermediate implementation, and invoke the task evaluator up to $12$ times. 
Every attempted execution counts toward this budget, including executions whose outputs are subsequently rejected by the Solution Auditor.

\paragraph{Initial Idea Pool.}
We initialize the idea pool using the Ideator from ScientistOne~\citep{meng2026scientistone}. 
The Ideator generates candidate ideas conditioned on the task context, after which its feasibility critic filters out ideas that are impractical or incompatible with the task requirements. 
This process yields an initial pool of approximately ten feasible ideas.

\paragraph{Agentic Surrogate.}
To avoid producing either a degenerate single cluster or an excessively fragmented map, the Organize operator is constrained to create between $C_{\min}$ and $C_{\max}$ semantic clusters, which we set to $C_{\min}=2$ and  $C_{\max}=5$.
The same bounds are used across all tasks and iterations. 
Within these constraints, the operator autonomously determines the number of clusters, their semantic descriptions, and the assignment of ideas to clusters.

\paragraph{Agentic Acquisition.}
At each round, the Expand operator pushes new ideas into the idea pool based on the four modes described in the main paper: push\_score, cross\_pollinate, fix\_error, and new\_idea.
To avoid an overflow of new ideas, we cap the maximum number of ideas for each branch to 3. 
For instance, if there were 5 branches in an iteration, the maximum number of idea that can be added to the pool in that iteration is 15.

\paragraph{Resource Planner.}
The first search iteration always launches 5 parallel Solver branches, ensuring an initial breadth of exploration before sufficient experimental evidence is available for adaptive planning.
In subsequent iterations, the Resource Planner dynamically plans branches for future iterations.
To ensure breadth of exploration at each iteration, we also set a minimum of 3 branches and a maximum of 10 branches every round.
Thus,
\begin{equation}
B_1=5,
\qquad
3 \leq B_t \leq 10  \;\; (t>1),
\qquad
\sum_{t=1}^{K}B_t=B_{\mathrm{tot}}.
\label{eq:implementation-resource-plan}
\end{equation}
The planner therefore changes how the fixed set of branches is distributed across iterations, while the per-branch execution budget remains fixed at $N_{\mathrm{branch}}$.
Empirically, however, the planner agent prefers to assign smaller number of parallel branches than 5.

\subsection{Prompt Templates}

Here, we show the prompt templates used by LLM-driven stages of AIM.
Each box contains an operator's prompt package: the system prompt followed by the USER PROMPT  with per-call payload fields.
Angle-bracketed placeholders (\texttt{<...>}) name the runtime content that fills each slot at prompt-render time.

\begin{tcolorbox}[promptbox={Agentic Surrogate: Organize Prompt Template}]
\begin{Verbatim}[fontsize=\small,breaklines=true]
You are the meta research agent overseeing a parallel research-discovery loop. Your job at every iteration is to ORGANIZE the current idea pool into K thematic CLUSTERS so the downstream allocator can spread parallel branches across distinct research approaches.

Hard constraints:
  - K within [K_min, K_max] from the user prompt (read every iter; do not assume a default).
  - Every idea id in the pool MUST appear in exactly ONE cluster.
  - Each cluster contains at least one idea.
  - Themes must be meaningful — group by mechanism / assumption / approach; avoid catch-alls.

OUTPUT — single JSON object inside ```json fences:
{
  "K": <int>,
  "clusters": [
    { "label": "Short name (<=6 words)",
      "theme": "1-3 sentences on what unifies these ideas",
      "idea_ids": ["<id>", ...],
      "rationale": "why these ideas belong together" }, ...
  ]
}
\end{Verbatim}

\tcbline

\begin{Verbatim}[fontsize=\small]
USER PROMPT:
## Benchmark task
<benchmark task description>

## Iteration
This is iteration <current iter number>. You have B = <parallel_branches this iter> parallel branches. K_max = <upper bound on K> is a HARD MAX — NOT a target. Pick K in [2, <K_max>] that best matches the pool's thematic structure.

## Prior organization
<previous iter's clustering JSON, or null on iter-1>

## Pool snapshot
<full idea pool: per-idea id, source (ideator/refined/fresh_inject),
 parent_id, latest_iter, latest_score, success, title, abstract>
\end{Verbatim}
\end{tcolorbox}

\begin{tcolorbox}[promptbox={Agentic Surrogate: Estimate Stage 1 Prompt Template}]
\begin{Verbatim}[fontsize=\small,breaklines=true]
You are ranking research-idea CLUSTERS by promising-ness. Your rankings feed a downstream allocator that decides how many parallel research branches to spend on each cluster.

You see: label, theme, n_evaluated / n_ideas, mean_evaluated_score, best_evaluated_score. Nothing more — you are judging the DIRECTION, not individual ideas.

## Ranking scheme — DENSE RANKS
Integers from 1 (most promising). Ties allowed — same integer. Next distinct rank after a tie is +1, not skipped.
  Valid:   [1,2,2,3,4]   [1,1,1,1,1]   [1,2,3,4,5]
  Invalid: [1,2,2,4,5]   (gap after tie)

## Guidance
- High `best_evaluated_score`     → LOW rank (promising).
- Many evaluated, all scored poorly → HIGH rank (exhausted).
- n_evaluated=0 → rank on THEME quality vs task's known bottlenecks.
- Do NOT collapse everything to rank 1.

## Output — JSON, fenced with ```json:
{ "cluster_ranks": [{"cluster_idx": <int>, "rank": <int>}, ...],
  "rationale": "1-3 sentences — cite specific clusters" }
\end{Verbatim}

\tcbline

\begin{Verbatim}[fontsize=\small]
USER PROMPT:
## Benchmark task
<benchmark task description>

## Scoring scale
<score direction (minimize/maximize), units, explicit numeric range if known>

## Clusters to rank (<N> total)
<per-cluster block: cluster_idx, label, theme, n_ideas, n_evaluated,
 mean_evaluated_score, best_evaluated_score>
\end{Verbatim}
\end{tcolorbox}

\begin{tcolorbox}[promptbox={Agentic Surrogate: Estimate Stage 2 Prompt Template}]
\begin{Verbatim}[fontsize=\small,breaklines=true]
You are ranking research IDEAS within a single cluster by promising-ness. Your ranking feeds the downstream allocator that picks which idea each branch attempts.

You see: unevaluated candidates within ONE cluster (title, short hypothesis, abstract). The cluster's LABEL and THEME give you the direction they share.

## Ranking scheme — DENSE RANKS (same rules as Stage 1)
  Valid:   [1,2,2,3]   [1,1,2,3]   [1,2,3,4]
  Invalid: [1,2,2,4]   (gap after tie)

## Guidance
- Rank by how strongly mechanism, novelty, and specificity predict a good score.
- LOW rank for ideas that clearly instantiate the cluster's theme with a concrete, testable optimization.
- HIGH rank for ideas that repeat existing themes or lack mechanism specificity.
- Do NOT collapse everything to rank 1 unless truly indistinguishable.

## Output — JSON, fenced with ```json:
{ "idea_ranks": [{"idea_id": "<id>", "rank": <int>}, ...],
  "rationale": "1-3 sentences" }
\end{Verbatim}

\tcbline

\begin{Verbatim}[fontsize=\small]
USER PROMPT:
## Benchmark task
<benchmark task description>

## Scoring scale
<direction, units, numeric range if known>

## This cluster
label: <cluster's short label>   
cluster_rank: <cluster's rank from Stage 1>
theme: <cluster's 1-3 sentence theme>

## Ideas to rank
<per-idea block: idea_id, title, short hypothesis, abstract>
\end{Verbatim}
\end{tcolorbox}

\begin{tcolorbox}[promptbox={Agentic Acquisition: Dispatch Stage 1 Prompt Template}]
\begin{Verbatim}[fontsize=\small,breaklines=true]
You are the meta-allocator. In THIS step you decide the SHAPE of the search for this iter (which cluster each branch is assigned to, plus cluster- and idea-level actions), NOT the specific ideas. Stage B picks the concrete ideas next.

## Vocabulary
cluster_action = "exploit" → invest branches in an already-producing cluster.
cluster_action = "explore" → invest in an underprobed cluster.
idea_action = "exploit"    → within the cluster, next step picks the most-promising unevaluated idea.
idea_action = "explore"    → within the cluster, next step picks a novel-mechanism unevaluated idea.

## Rank semantics (v1.6)
- cluster_rank=1 → most promising (prefer for exploit).
- Higher rank → less promising. If well-probed, may be exhausted; if under-probed, holds explore value.
- Ties allowed and common — treat tied clusters as comparable.

## Output — JSON:
{
  "cluster_decisions": [
    {"cluster_idx": <int>, "action": "exploit"|"explore", "n_branches": <int>, "rationale": "..."}, ...
  ],
  "branch_action_assignments": [
    {"branch": <int>, "cluster_idx": <int>, "cluster_action": "exploit"|"explore",
     "idea_action": "exploit"|"explore", "rationale": "..."}, ...
  ],
  "rationale_summary": "..."
}

Validator constraints: each cluster_idx appears at most once in cluster_decisions; sum(n_branches) == B; EXACTLY B branch_action_assignments; each branch's cluster_action matches its cluster's action.

A run that puts everything on one cluster, or labels every action "exploit", is failing the loop's purpose. Do NOT output any `idea_id` in this step.
\end{Verbatim}

\tcbline

\begin{Verbatim}[fontsize=\small]
USER PROMPT:
## Benchmark task
<benchmark task description>

## Iteration
iter_idx=<current iter, 0-indexed>, total_iters=<total planned iters>, B=<parallel branches this iter>

## Previous iter's allocation
<the last iter's allocation.json — rationale summary + per-branch picks — or null on iter-1>

## Score history
<per-branch score records so far: iteration, branch, score, success, role, idea_id, audit_flags>

## Clusters
<per-cluster block: cluster_idx, label, theme, member counts,
 actual-score distribution (min/max/mean of evaluated members),
 cluster_rank, top-K unevaluated preview by within_cluster_rank>
\end{Verbatim}
\end{tcolorbox}

\begin{tcolorbox}[promptbox={Agentic Acquisition: Dispatch Stage 2 Prompt Template}]
\begin{Verbatim}[fontsize=\small,breaklines=true]
You are the idea-selection agent for ONE branch. Stage A already decided this branch's cluster and (cluster_action, idea_action). Pick the SPECIFIC unevaluated idea from the assigned cluster that best matches idea_action.

## Action interpretation (v1.6 rank semantics)
Each unevaluated member carries `within_cluster_rank=R/M` (1=most promising in the cluster, M=cluster size) and `cluster_rank=K`.

For idea_action = "exploit":
- Pick the LOWEST within_cluster_rank member.
- Ties are common — break using abstract + cluster context.
- Use evaluated members as evidence — what mechanisms paid off?
- A refined child whose parent scored well is often the right exploit pick even if within_cluster_rank isn't strictly lowest.

For idea_action = "explore":
- Pick a HIGH within_cluster_rank member — evaluating it gives more info gain than another shot at the top.
- Prefer mechanisms NOT YET tried by evaluated members.
- Refined children on NEW trajectories beat refined children on already-explored trajectories.
- DO NOT pick the lowest within_cluster_rank — that is exploit.

## Hard constraints
- `idea_id` MUST be UNEVALUATED and in the assigned cluster.
- `idea_id` MUST NOT be in the "Already picked by other branches" list.
- Evaluated members are shown for REASONING only — never pick one.

## Output — JSON only, no prose:
{ "branch": <int>, "idea_id": "<unevaluated-id>", "rationale": "<2-4 sentences>" }

## Iter-1 suffix (appended only when iter_idx == 0):
No idea has been evaluated yet — every predicted_score is an LLM estimate. For every branch, set idea_action: "exploit" and pick the highest-predicted_score idea. Iter-1 is for validating the estimator's top predictions.
\end{Verbatim}

\tcbline

\begin{Verbatim}[fontsize=\small]
USER PROMPT:
## Benchmark task
<benchmark task description>

## This branch
branch_idx=<branch index within [0,B)>, cluster_idx=<assigned cluster index>, cluster_action=<exploit or explore, from Stage A>, idea_action=<exploit or explore, from Stage A>

## Cluster
label=<cluster label>, theme=<cluster 1-3 sentence theme>

## Cluster members
<per-member block: id, title, abstract, was_assigned,
 for evaluated: latest_score, latest_iter, latest_success,
 for unevaluated: within_cluster_rank, cluster_rank, n_cluster_members, lineage (parent_id, parent_score, trajectory_trend)>

## Already picked by other branches
<sorted list of idea_ids picked by earlier-index branches this iter>
\end{Verbatim}
\end{tcolorbox}

\begin{tcolorbox}[promptbox={Agentic Acquisition: Expand Prompt Template}]
\begin{Verbatim}[fontsize=\small,breaklines=true]
You are the swarm-aware refiner for a parallel research-search loop. Each iteration, multiple branches independently try ideas; your job is to take ONE branch's outcome plus the full swarm context and propose refined ideas that should rejoin the pool.

You see: this branch's parent idea + evaluation; sibling branches' outcomes this iter; cluster context; prior-iter top-3 / bottom-3 across the run; distilled lessons.

You have up to <max proposal slots for this branch> PROPOSAL SLOTS. Emit AT MOST ONE proposal per action kind; skip any that don't apply. Quality over quota.

The 4 actions (menu):

- `fix_error` — branch failed. Propose a targeted fix. Skip if it succeeded.
- `push_score` — branch succeeded with headroom. Propose a tighter implementation.
- `cross_pollinate` — a specific sibling outcome OR lesson gives a compositional improvement.
    HARD: base reasoning ONLY on ## Lessons + ## Sibling outcomes.
    Cite exactly ONE source:
      * sibling_referenced: <sibling_idea_id from Sibling outcomes>, OR
      * lesson_context_idea: <tag or context_idea from Lessons>
    Both empty, both set, or unknown id → proposal dropped.
- `new_idea` — no in-cluster improvement to offer; propose an ORTHOGONAL direction. Parentless.
    Workflow: (1) read Cluster context for the diversity BASELINE;
              (2) name 1-2 orthogonal directions in the Abstract;
              (3) instantiate one.
    If best you can offer is a baseline variant → use cross_pollinate or push_score, NOT inject_fresh.

OUTPUT — YAML:
ideas:
  - Name: <short snake_case>
    Title: <refined idea title>
    Short Hypothesis: <one sentence>
    Abstract: <2-3 sentences, <=200 words>
    Experiments:
      - hypothesis: <what this experiment tests>
        experiment_plan_steps: [<concrete step>, ...]
    Risk Factors and Limitations: <known risks>
    kind: fix_failure | push_score | cross_pollinate | inject_fresh
    parent_id: <omit for default parent; empty for inject_fresh>
    sibling_referenced: <id>   # REQUIRED for cross_pollinate — pick ONE
    lesson_context_idea: <tag or id>  # of these two, not both, not neither

Hard rules: AT MOST ONE proposal per kind; empty list acceptable; do not include the parent idea verbatim.
\end{Verbatim}

\tcbline

\begin{Verbatim}[fontsize=\small]
USER PROMPT:
## Benchmark task
<benchmark task description>

## This branch (<branch index>): parent idea
<parent idea's full YAML: name, title, short hypothesis, abstract, experiments, risk factors — the same the solver received>

## Evaluation
<evaluator output: score, success, one-line eval detail (crash msg, failure mode, or short success summary)>

## Sibling outcomes this iteration
<per-sibling row: branch idx, idea id, idea title, score, success, one-line eval gist — from ALL sibling branches this iter>

## Cluster context
<parent's cluster: label, theme, evaluated-member score distribution, sibling-cluster overview so the model can identify orthogonal directions>

## Prior-iter signals
<top-3 and bottom-3 ideas across the whole run so far, with iteration, branch, score, cluster label - for identifying durable winners / dead ends>

## Lessons
<consolidated lessons pool (capped at <lessons_cap>): tag, lesson text, optional context_idea id>
\end{Verbatim}
\end{tcolorbox}

\begin{tcolorbox}[promptbox={Solution Auditor Prompt Template}]
\begin{Verbatim}[fontsize=\small,breaklines=true]
You are a solution auditor for an automated research pipeline. Your job is to determine whether a solution is LEGITIMATE — i.e., it actually solves the stated task in a meaningful way rather than gaming the evaluation.

You will be given:
1. The TASK INSTRUCTION defining what problem must be solved
2. The IDEA the agent was supposed to implement
3. The SOLUTION CODE the agent produced
4. The EVALUATION METRICS returned by the evaluator

Check for these four failure modes:

**Reward Hacking**: The solution manipulates, monkey-patches, or reverse-engineers the evaluator to inflate its score without genuinely solving the task.
**Idea-Solution Mismatch**: The solution ignores the proposed idea and implements something unrelated.
**Task Mismatch**: The IDEA (and hence the CODE) solves a DIFFERENT problem than the TASK specifies (e.g. approximation instead of exact, weaker guarantee than required). Orthogonal to Idea-Solution Mismatch; both can be flagged simultaneously.
**Trivial Solution**: No-op, constants, verbatim baseline.

Respond with a JSON object (no markdown fences):
{
  "legit": true/false,
  "flags": ["reward_hacking", "idea_mismatch", "task_mismatch", "trivial"],
  "confidence": 0.0-1.0,
  "reasoning": "one paragraph explanation",
  "reconstructed_idea": { ... }   // REQUIRED iff `flags` contains "idea_mismatch"
}

Only include flags that apply. 

`reconstructed_idea` schema (matches ideator's canonical shape):
{ "title", "short_hypothesis", ["twist" iff mode="unconventional"],
  "grounded_in":[...], "addresses", "mode": "conservative"|"unconventional" }
\end{Verbatim}

\tcbline

\begin{Verbatim}[fontsize=\small]
USER PROMPT:
## TASK INSTRUCTION
<benchmark task description>

## IDEA
<the idea assigned to this branch: title, short hypothesis, abstract,
 experiments, risk factors — the same YAML the solver received>

## SOLUTION CODE
<the solver's final produced source code>

## EVALUATION METRICS
<evaluator output JSON>
\end{Verbatim}
\end{tcolorbox}

\begin{tcolorbox}[promptbox={Resource Planner Prompt Template}]
\begin{Verbatim}[fontsize=\small,breaklines=true]
You are the DYNAMIC resource planner for a parallel research-search loop. The total branch-budget is FIXED (= parallel_branches × iterations). Iter-1 is pinned to `parallel_branches` for deterministic bootstrap. Every remaining branch has the same per-branch compute budget.

**Length is your primary lever.** The `iterations` value from the launcher is used only to compute the total branch-budget — NOT a target for len(plan). Do NOT match the launcher's iteration count as a reflex.

Your job: distribute REMAINING branches across as many additional iterations as you think optimal.
  - Spend EXACTLY `remaining` branches (no over- or under-shoot).
  - Each iter you add has BETWEEN 1 AND max_branches_per_iter branches.
  - PREFER AT LEAST <B_min> BRANCHES PER ITERATION.
  - You cannot revise the frozen prefix.

## Preferred shapes (choose length from evidence, not the launcher's iterations knob)
- Long tail of 3-branch iters — refinement > parallelism.
- Iterate-and-refine (5-10 iters, 3-5 branches) — balanced.
- Broad-then-deep — heavy early, tapering to polish.
- Concentrate-and-terminate (2-3 large iters).

## Hard constraints (validated)
- sum(plan) == total_runs      # EXACTLY
- plan[0] == parallel_branches # iter-1 pinned
- 1 <= plan[i] <= max_branches_per_iter
- plan[:iter_idx] == frozen_prefix verbatim
- len(plan) > iter_idx         # at least one more iter beyond frozen prefix

## Output — JSON in ```json fences:
{ "plan": [<int>, <int>, ...],
  "rationale": "2-4 sentences on why THIS LENGTH and this distribution" }


## Benchmark task
<benchmark task description>

## Position
iter_idx=<current iter, 0-indexed>,
total_runs=<total branch budget = parallel_branches × iterations>,
parallel_branches=<launcher's -n value>,
max_branches_per_iter=<hard cap per-iter, from launcher>,
frozen_prefix=<list of per-iter branch counts already spent>

## Evidence
best_so_far=<best clean score across all evaluated branches so far>,
target=<current adaptive target score or null>,
evaluated_scores=<list of evaluated branches' raw scores, in order>,
estimated_scores=<list of unevaluated candidates' predicted scores>

## Cluster summary
<per-cluster: label, theme, n_evaluated, best_evaluated_score>

## Lessons
<consolidated lessons pool (tag + lesson text)>
\end{Verbatim}
\end{tcolorbox}

\section{Further Experiments}

\subsection{Solver Substitution with Claude Code}
\label{apdx:claude-code}

To evaluate whether the effectiveness of \textsc{AIM} is tied to a specific solver implementation and to assess how much an explicit idea management scaffold benefits existing coding agents, we conduct a solver substitution experiment. 
In this setting, we replace the default Gemini Deep Solver with Claude Code~\citep{claude_code}. 
We compare \textsc{AIM} with ScientistOne~\citep{meng2026scientistone} across all five AutoLab tasks. 
ScientistOne represents the strongest idea-driven baseline from our main experiments, configured to delegate implementation to Claude Code.
Similarly, our proposed \textsc{AIM} framework uses Claude Code as the downstream solver agent across parallel execution branches.

Table~\ref{tab:claude-code} summarizes the performance across the benchmark tasks.
Across all evaluated domains, \textsc{AIM} consistently outperforms or matches ScientistOne across all five benchmarks when both utilize the Claude Code solver.
\begin{table}[h!]
\centering
\caption{\textbf{Solver Substitution}. Comparison of \textsc{AIM} and ScientistOne with Claude Code Solvers.}
\label{tab:claude-code}
\renewcommand{\arraystretch}{1.35}
\resizebox{0.8\linewidth}{!}{%
\begin{tabular}{l | c c}
\toprule
\textbf{Task}
& \textbf{ScientistOne (Claude-Code Solver)}
& \textbf{\textsc{AIM} (Claude-Code Solver)} \\
\midrule

Flash Attention
& 81.6 \stderr{3.5}
& \textbf{84.5} \stderr{0.3} \\

Radix Sort
& 67.4 \stderr{0.6}
& \textbf{67.7} \stderr{0.2} \\

FFT Rust
& 54.6 \stderr{0.3}
& \textbf{55.4} \stderr{0.2} \\

AES128 Ctr
& 66.9 \stderr{0.3}
& \textbf{67.6} \stderr{0.6} \\

Z-order Range Scan
& \textbf{52.2} \stderr{0.1}
& \textbf{52.2} \stderr{0.3} \\

\bottomrule
\end{tabular}%
}
\end{table}


\subsection{Further Ablation Studies on the Agentic Surrogate}
\label{apdx:surrogate}

Table~\ref{tab:further_surrogate} presents further in-depth ablation studies on the Agentic Surrogate module.
We evaluate several alternative configurations to isolate the contributions of its sub-components. 
The ``without Organize'' setting bypasses clustering entirely, treating the entire list of ideas as a single unified cluster before passing them to the Estimate operator for ranking. 
The ``without Estimate'' configuration removes the estimation step, passing only the unranked cluster information from the Organize step directly to the Dispatch operator within the Agentic Acquisition module. 
The ``without Agentic Surrogate'' setting removes both the Organize and Estimate components, serving as the same baseline ablation described in the main manuscript. 
Furthermore, we test an ``Embedding-based Organize'' variant that performs K-means clustering (with an automatically determined $K$) based on idea embeddings generated by the Gemini-Embedding-001~\citep{lee2025gemini} model. 
Finally, the ``Direct Score Estimation'' configuration replaces ordinal ranking estimates with direct raw score predictions.

\begin{table}[h!]
    \centering
    \caption{Further ablation studies on the Agentic Surrogate module.}
    \label{tab:further_surrogate}
    \begin{tabular}{l c}
        \toprule
        \textbf{Method} & \textbf{Flash Attention Scores} \\
        \midrule
        Full \textsc{AIM} & 90.5 \stderr{0.8} \\
        \midrule
        \textit{without} Organize & 89.0 \stderr{0.4} \\
        \textit{without} Estimate & 87.9 \stderr{1.2} \\
        \textit{without} Agentic Surrogate (both Organize and Estimate) & 85.9 \stderr{2.7} \\
        \midrule
        Embedding-based Organize & 83.4 \stderr{2.8} \\
        Direct Score Estimation & 89.3 \stderr{1.2} \\
        \bottomrule
    \end{tabular}
\end{table}

The results from the ablation study demonstrate the critical importance of both the LLM-driven semantic organization and the ordinal estimation components.
The Full AIM achieves the highest performance with a Flash Attention score of $90.5 \pm 0.8$. 
Removing either the Organize or Estimate operators degrades performance to $89.0 \pm 0.4$ and $87.9 \pm 1.2$, respectively, while removing the Agentic Surrogate entirely results in a substantial drop to $85.9 \pm 2.7$. 

Interestingly, the Embedding-based Organize approach yields the lowest score of all configurations ($83.4 \pm 2.8$), performing even worse than the complete removal of the surrogate. 
Our interpretation of this result is that standard distance-based clustering on generic semantic embeddings fails to capture the nuanced, task-specific structural relationships that the LLM-based Organize operator can successfully identify and group. 
Additionally, substituting ordinal ranking with Direct Score Estimation ($89.3 \pm 1.2$) underperforms the Full AIM. 
This indicates that language models are generally more reliable at performing relative, ordinal comparisons of research ideas than they are at predicting absolute, uncalibrated performance scores from raw text.

\clearpage
\subsection{Time Efficiency Plots for All Tasks}
\label{apdx:full-plots}

In this section, we provide the time efficiency plot in Figure~\ref{fig:time_to_score} and Figure~\ref{fig:time_to_score_gpu} for all the 10 tasks evaluated.
Overall, our \textsc{AIM} provides an efficient framework, attaining top scores in a shorter amount of time, compared to the idea-driven and solution-driven baselines.

\begin{figure*}[h!]
    \centering

    \begin{subfigure}[t]{0.49\textwidth}
        \centering
        \includegraphics[width=\linewidth]{B_Figures/figs/flash_attention_combined.pdf}
        \caption{Flash Attention}
        \label{fig:plot2}
    \end{subfigure}
    \begin{subfigure}[t]{0.49\textwidth}
        \centering
        \includegraphics[width=\linewidth]{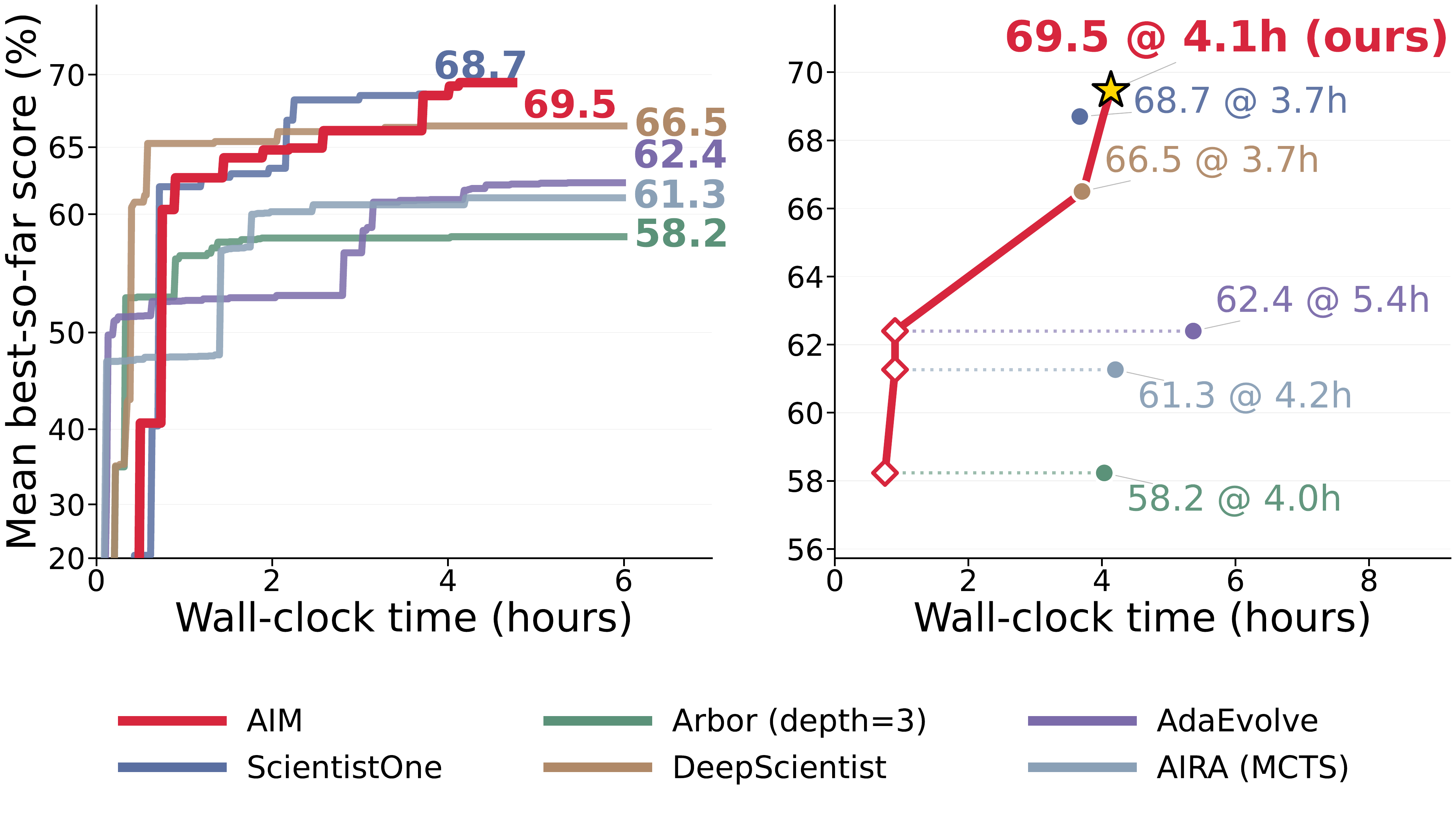}
        \caption{Radix Sort}
        \label{fig:plot2}
    \end{subfigure}

    \vspace{4mm}
    
    \begin{subfigure}[t]{0.49\textwidth}
        \centering
        \includegraphics[width=\linewidth]{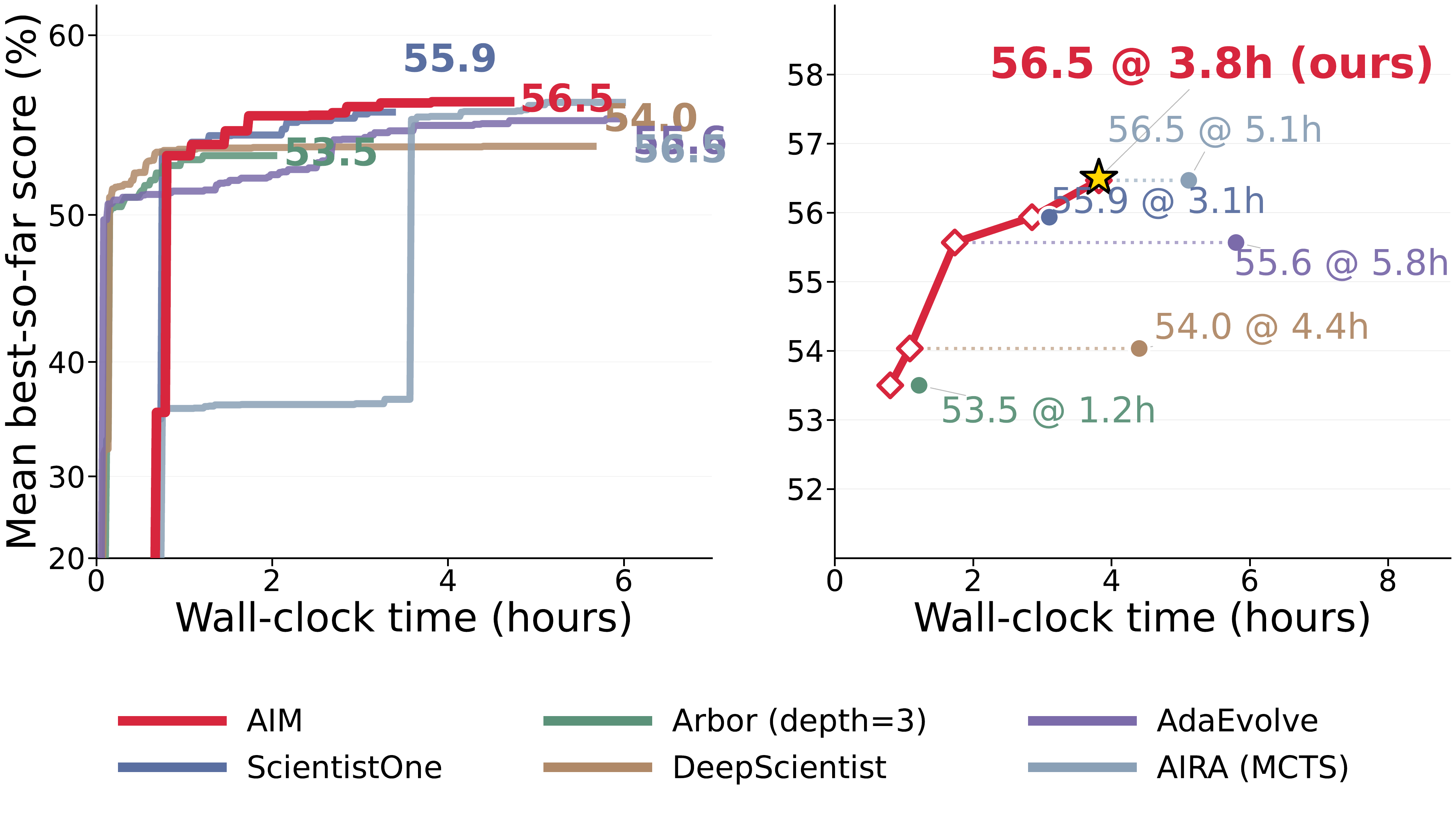}
        \caption{FFT Rust}
        \label{fig:plot3}
    \end{subfigure}
    \hfill
    \begin{subfigure}[t]{0.49\textwidth}
        \centering
        \includegraphics[width=\linewidth]{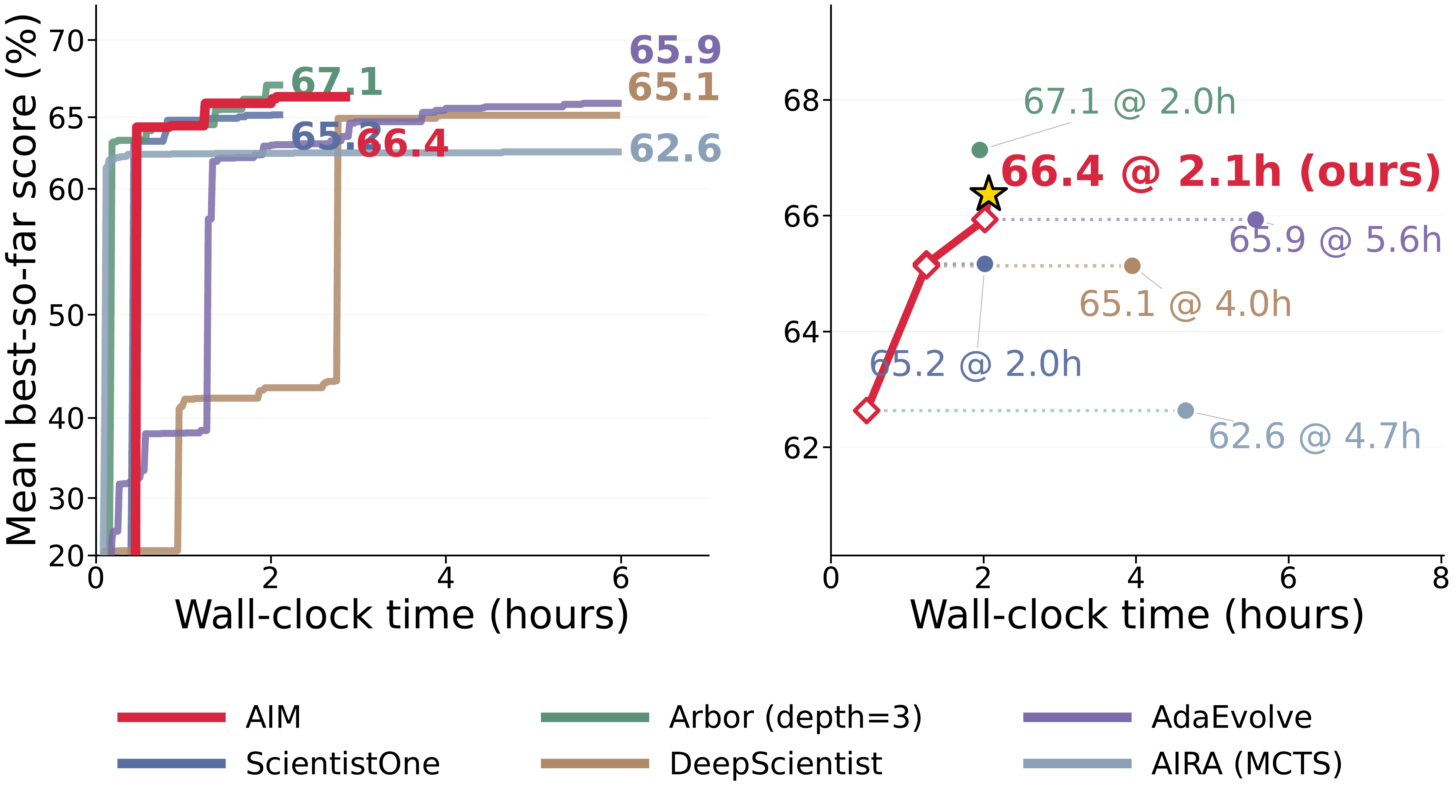}
        \caption{AES128 Ctr}
        \label{fig:plot4}
    \end{subfigure}

    \vspace{4mm}
    
    \begin{subfigure}[t]{0.49\textwidth}
        \centering
        \includegraphics[width=\linewidth]{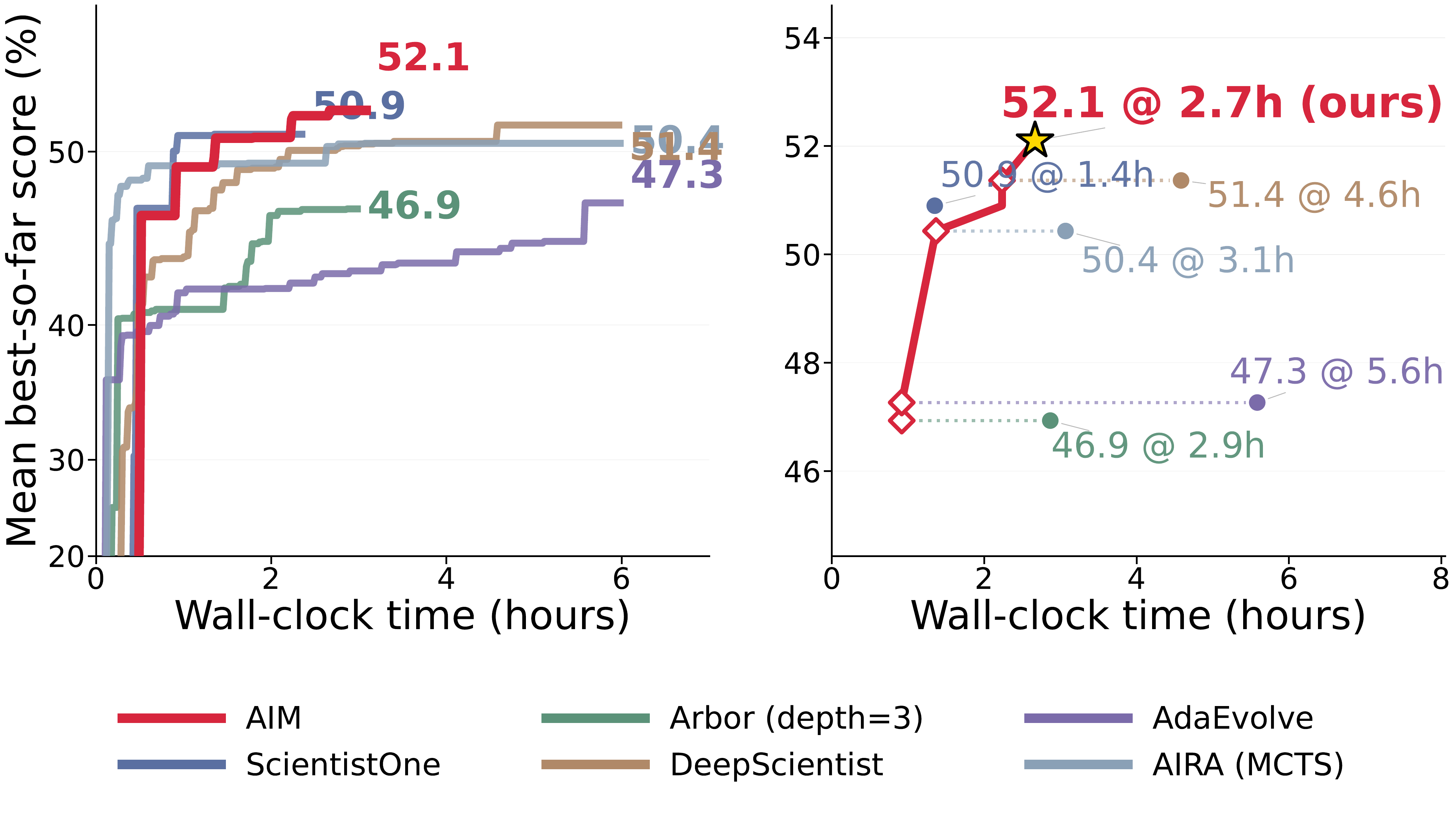}
        \caption{Z-order Range Scan}
        \label{fig:plot5}
    \end{subfigure}

    \caption{Time Efficiency Plots Across AutoLab Tasks.}
    \label{fig:time_to_score}
\end{figure*}
\begin{figure*}[h!]
    \centering

    \begin{subfigure}[t]{0.49\textwidth}
        \centering
        \includegraphics[width=\linewidth]{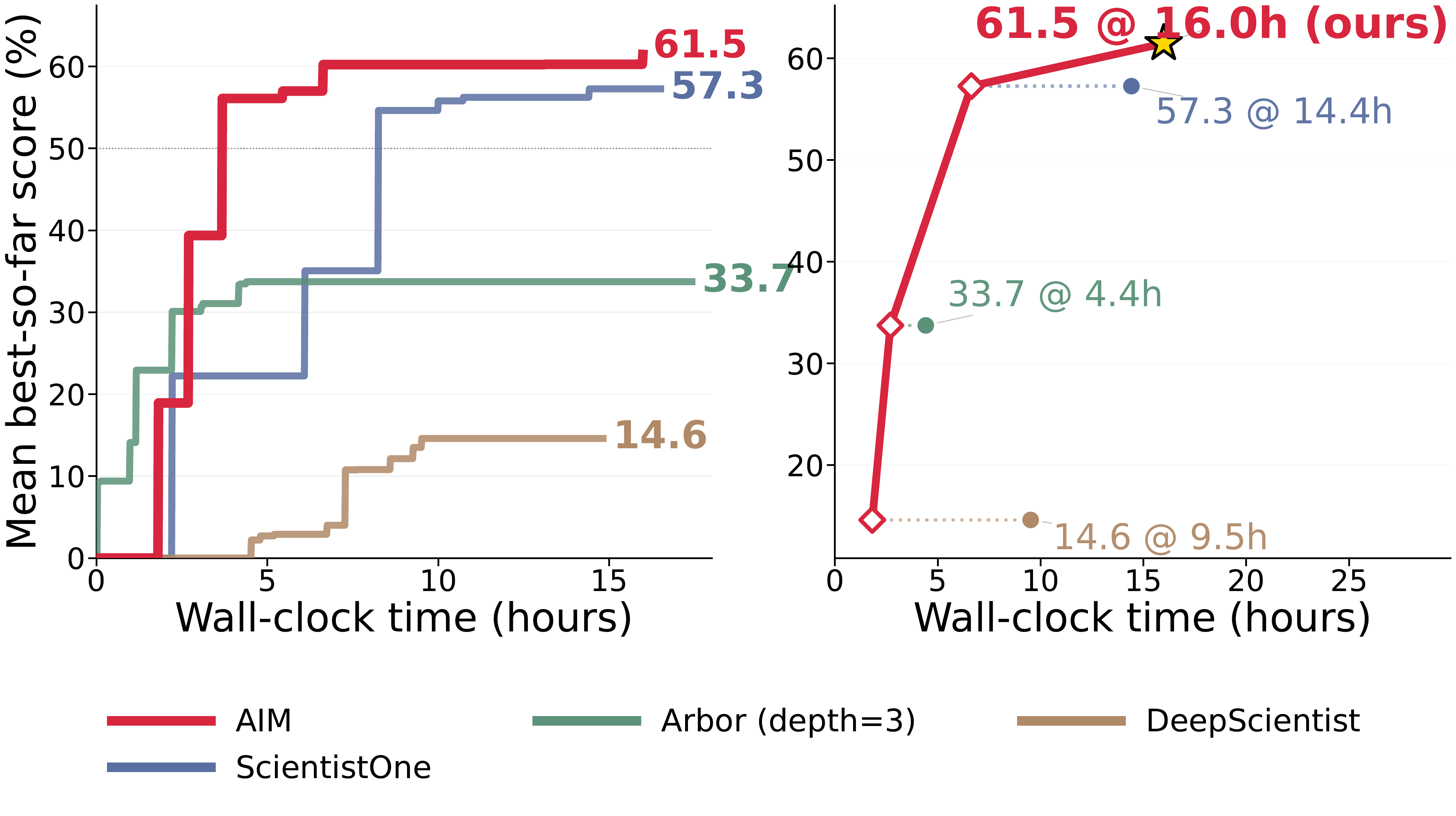}
        \caption{Moving Mnist World Model}
        \label{fig:plot2}
    \end{subfigure}
    \begin{subfigure}[t]{0.49\textwidth}
        \centering
        \includegraphics[width=\linewidth]{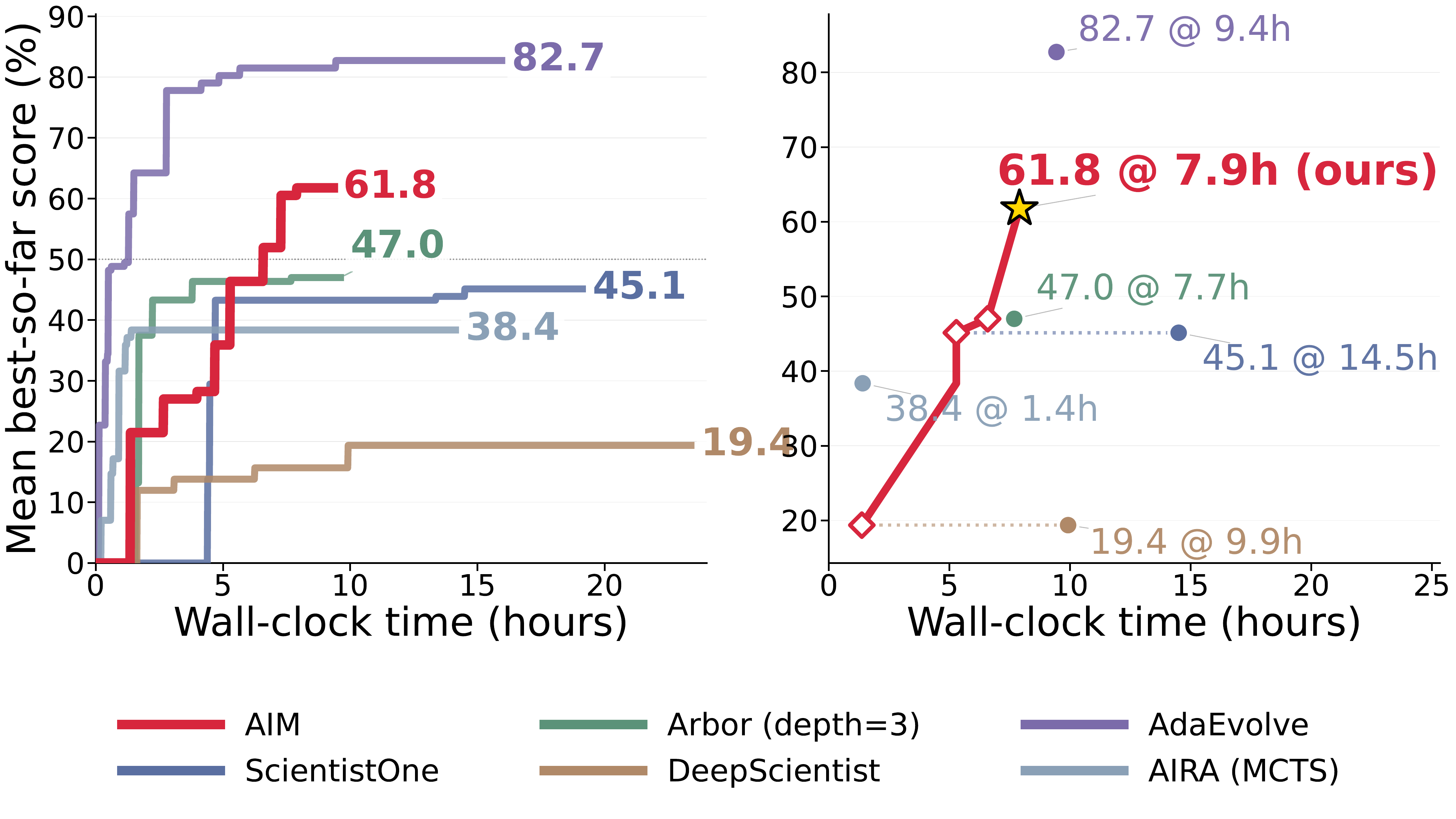}
        \caption{Data Select Ifeval}
        \label{fig:plot2}
    \end{subfigure}

    \vspace{4mm}
    
    \begin{subfigure}[t]{0.49\textwidth}
        \centering
        \includegraphics[width=\linewidth]{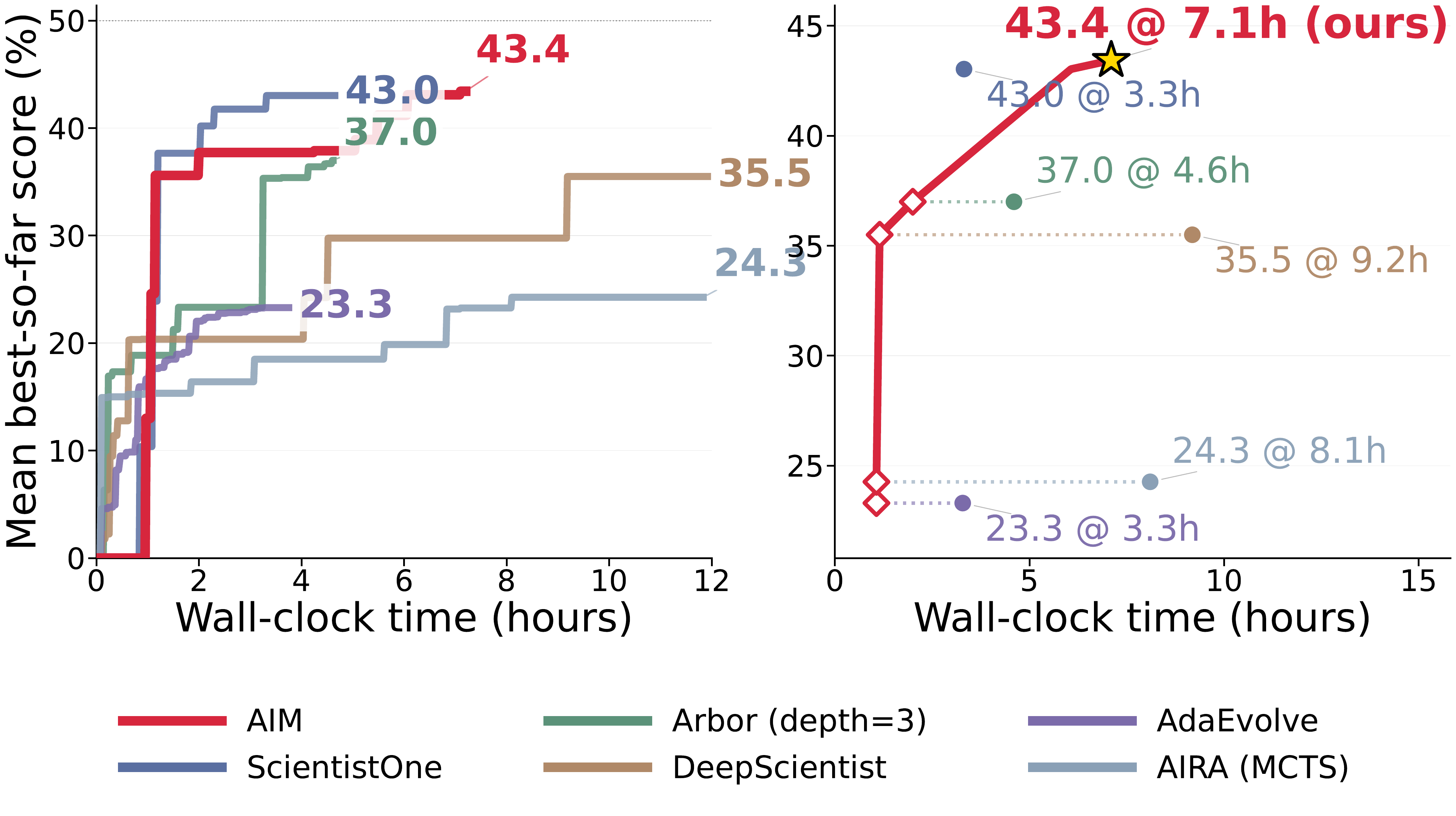}
        \caption{Huffman Canonical Decode}
        \label{fig:plot3}
    \end{subfigure}
    \hfill
    \begin{subfigure}[t]{0.49\textwidth}
        \centering
        \includegraphics[width=\linewidth]{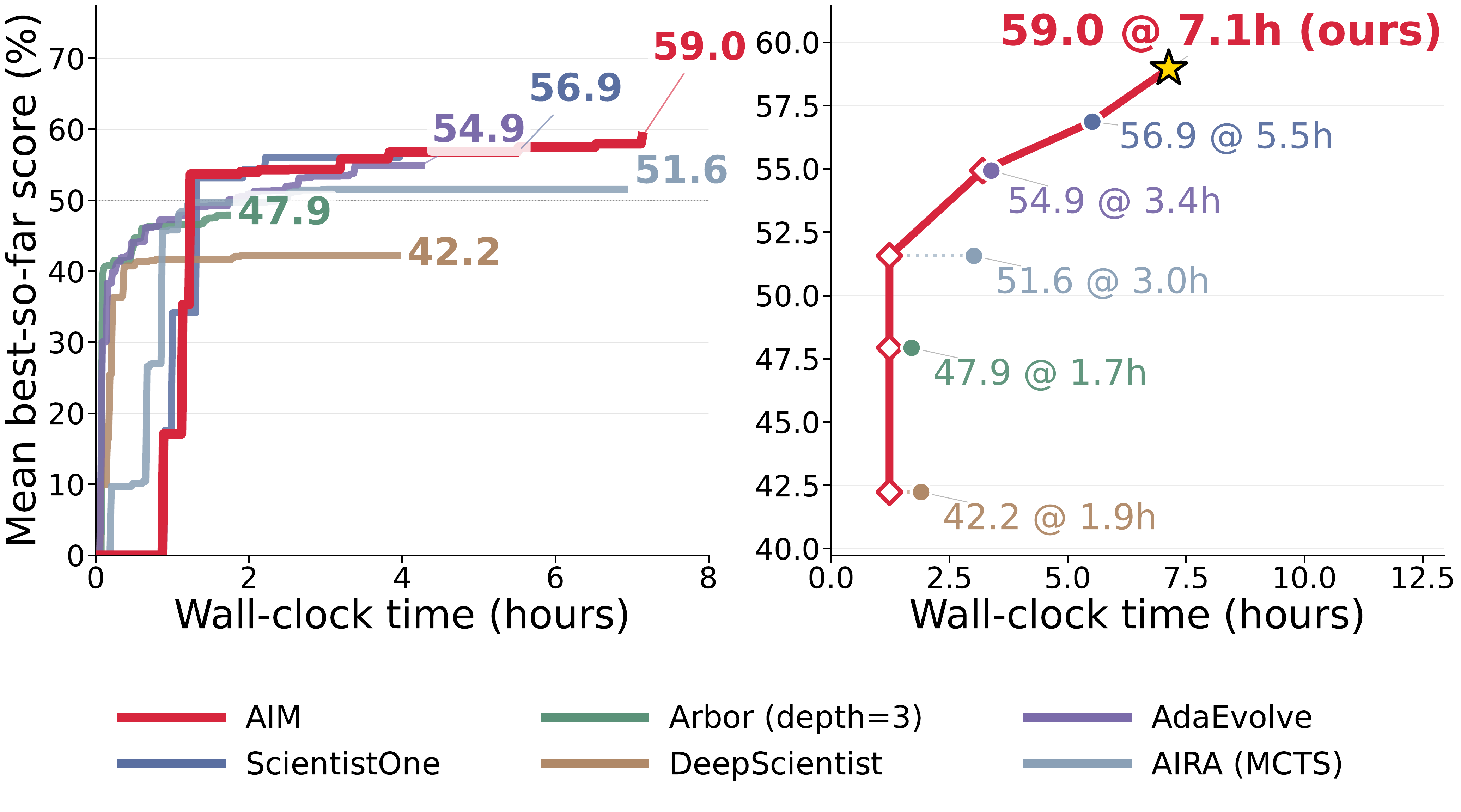}
        \caption{NTT Butterfly}
        \label{fig:plot4}
    \end{subfigure}

    \vspace{4mm}
    
    \begin{subfigure}[t]{0.49\textwidth}
        \centering
        \includegraphics[width=\linewidth]{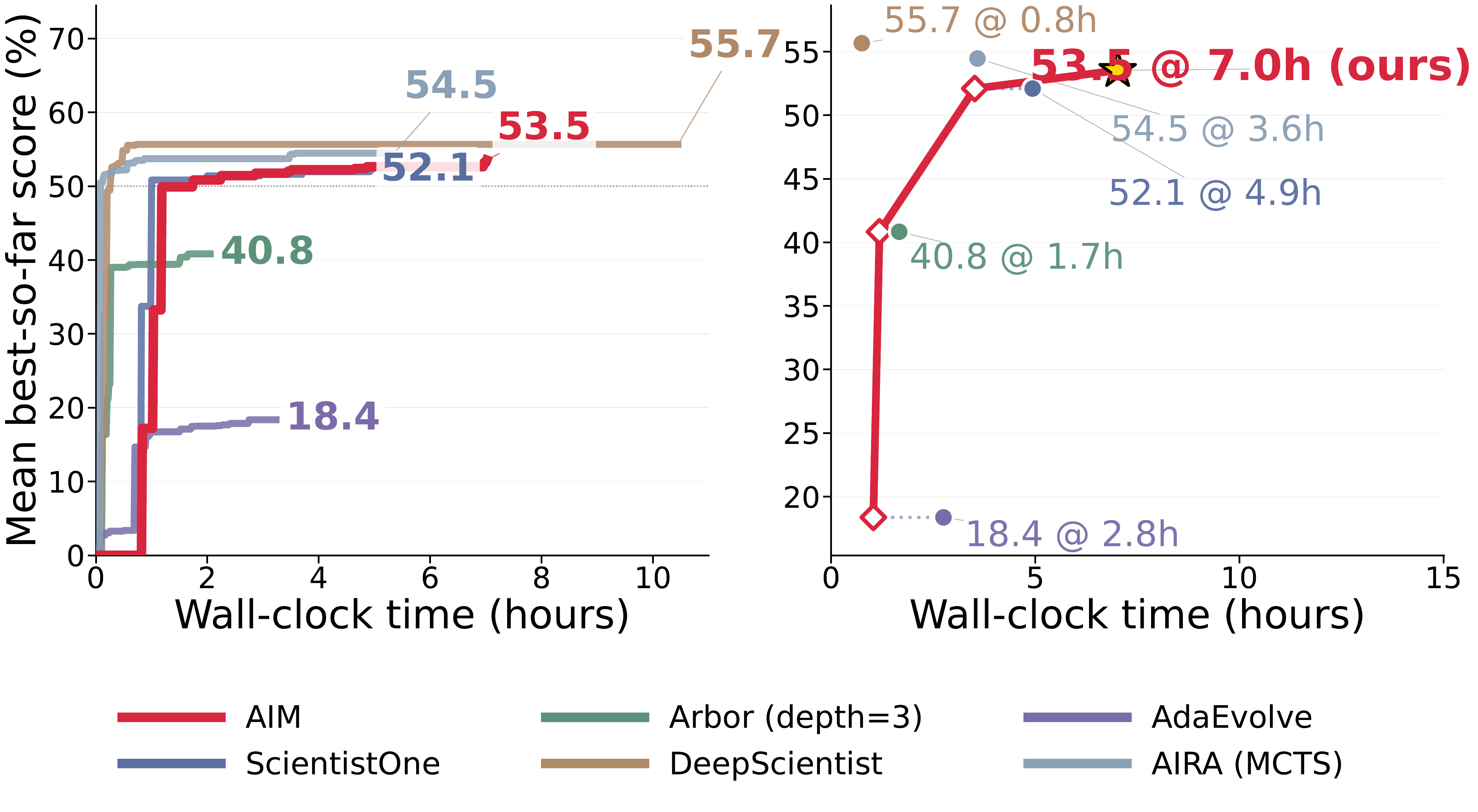}
        \caption{ICP Correspondence Step}
        \label{fig:plot5}
    \end{subfigure}

    \caption{Time Efficiency Plots Across AutoLab Model Development \& CUDA Tasks.}
    \label{fig:time_to_score_gpu}
\end{figure*}

\clearpage

In addition to the wall-clock time-to-score plots, we also provide the number of executions-to-score plots in Figure~\ref{fig:exec_to_score} and Figure~\ref{fig:exec_to_score_gpu}.
Note that we allowed five parallel workers for the baselines that enable parallel executions, and each task is budgeted to a maximum of 6 hours for System Optimization; 24 hours for Model Development; and 12 hours for CUDA tasks.
Overall, our approach tends to require more executions to reach its best score.
We conjecture this is due to AIM's tendency to first explore various directions in the first iterations.

\begin{figure*}[h!]
    \centering

    \begin{subfigure}[t]{0.49\textwidth}
        \centering
        \includegraphics[width=\linewidth]{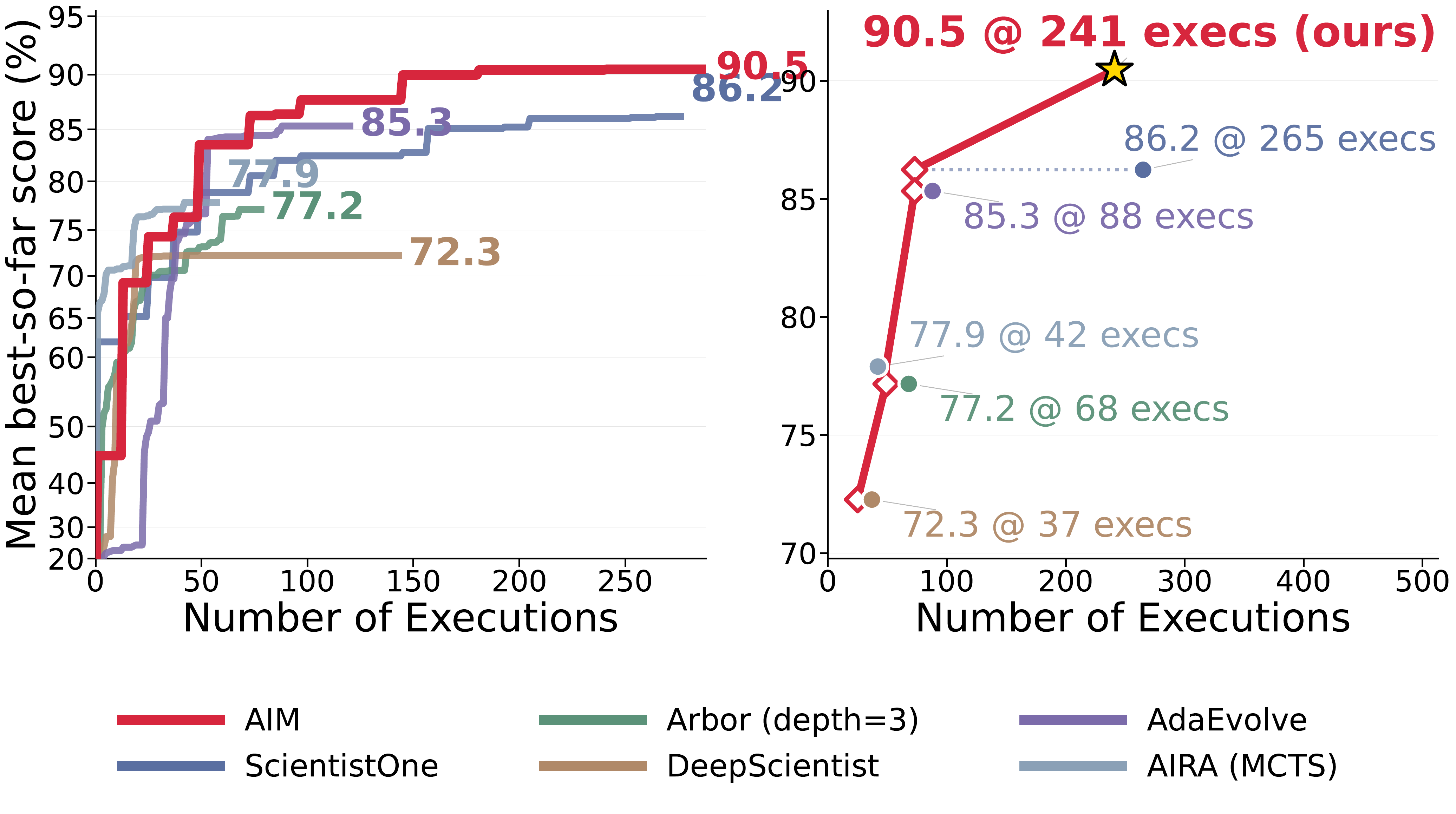}
        \caption{Flash Attention}
        \label{fig:plot2}
    \end{subfigure}
    \begin{subfigure}[t]{0.49\textwidth}
        \centering
        \includegraphics[width=\linewidth]{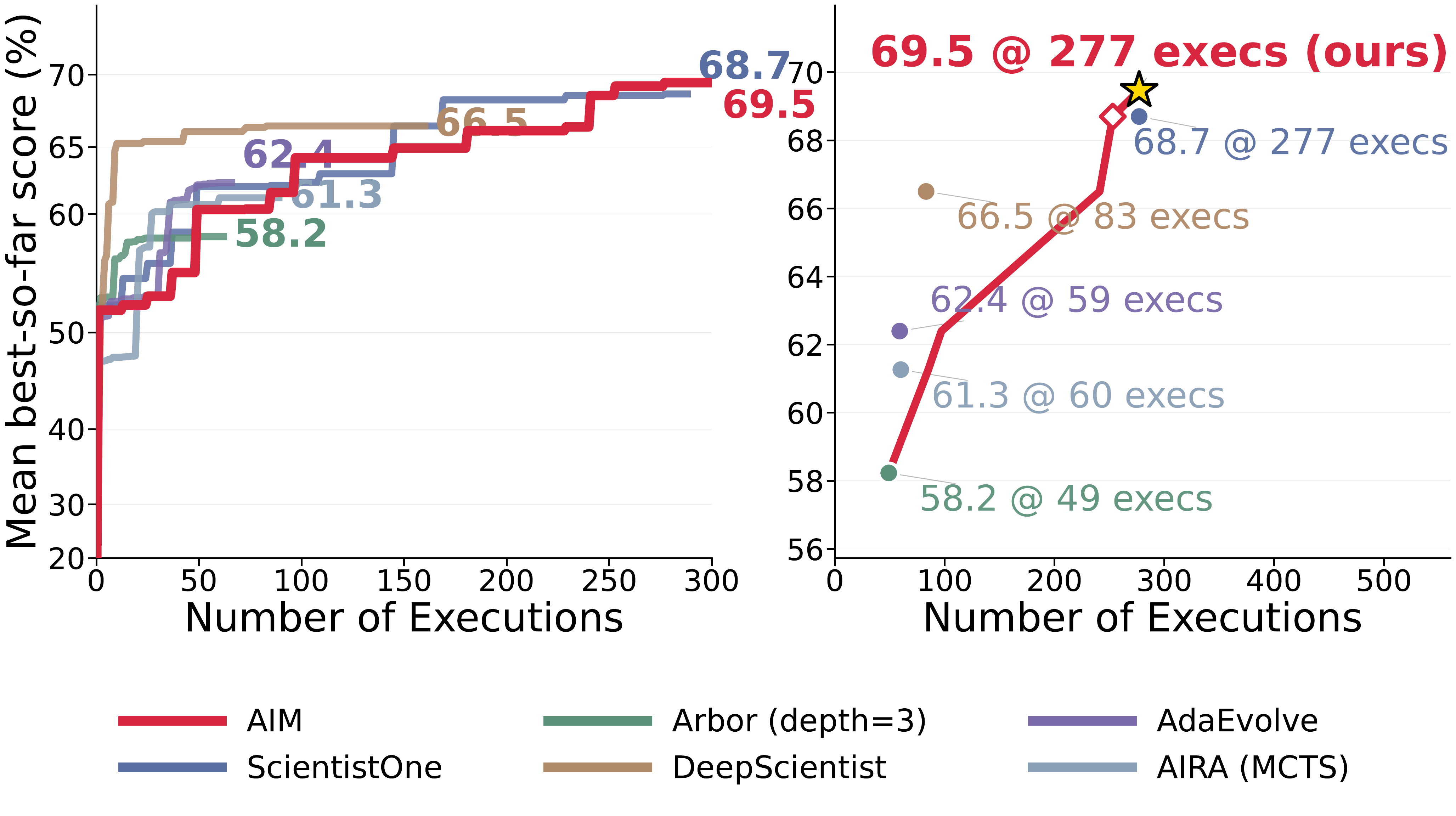}
        \caption{Radix Sort}
        \label{fig:plot2}
    \end{subfigure}

    \vspace{4mm}
    
    \begin{subfigure}[t]{0.49\textwidth}
        \centering
        \includegraphics[width=\linewidth]{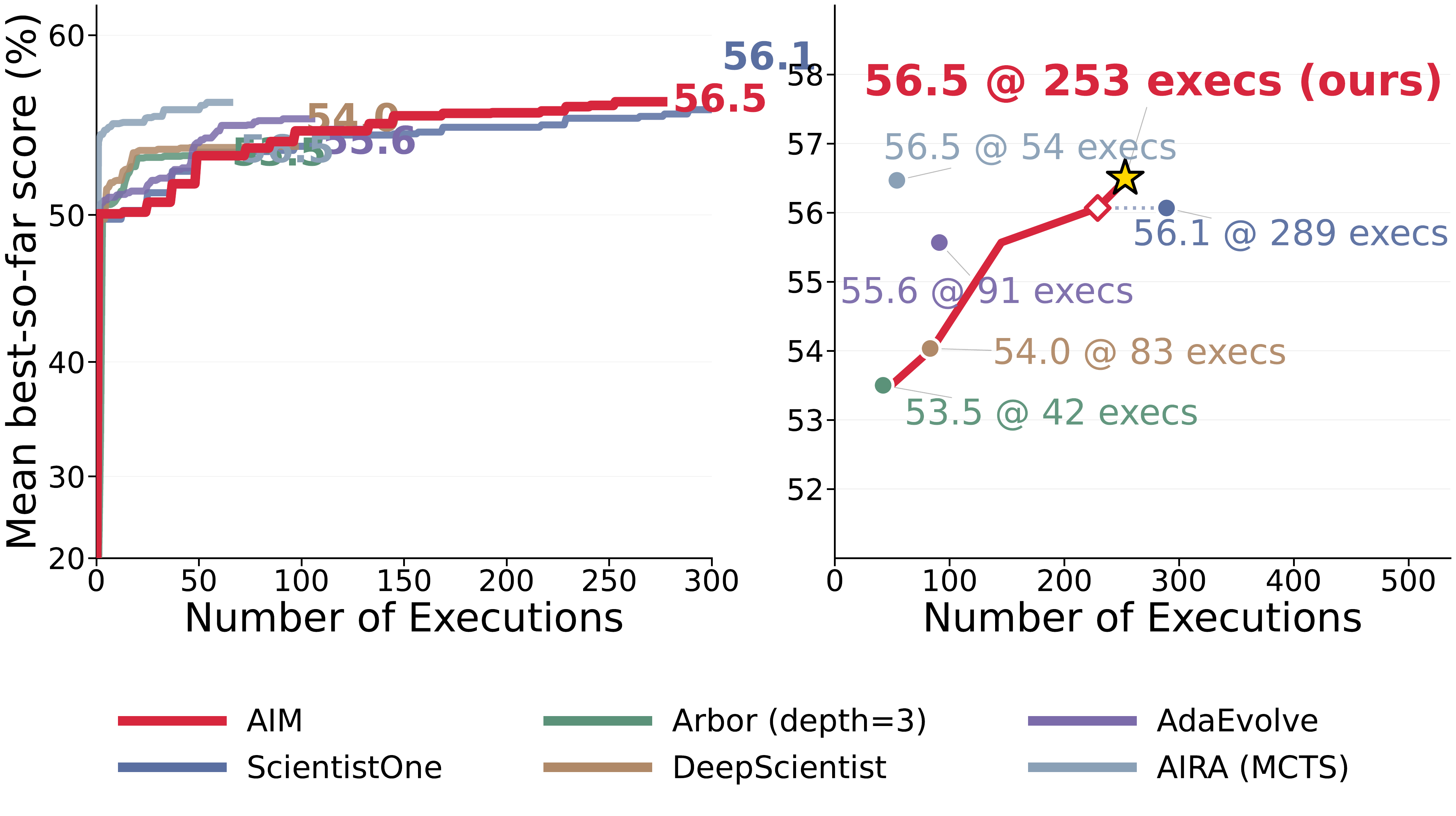}
        \caption{FFT Rust}
        \label{fig:plot3}
    \end{subfigure}
    \hfill
    \begin{subfigure}[t]{0.49\textwidth}
        \centering
        \includegraphics[width=\linewidth]{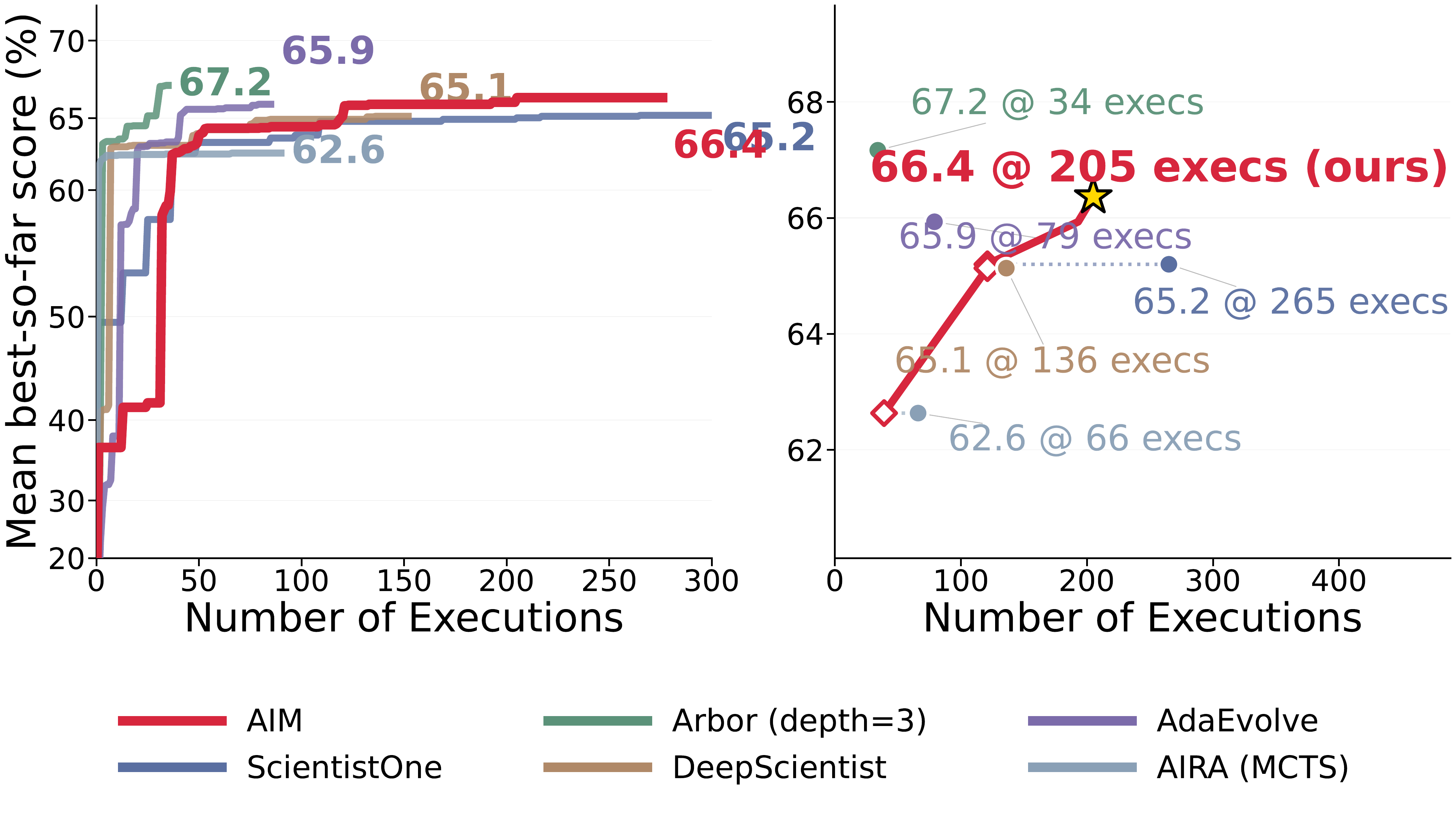}
        \caption{AES128 Ctr}
        \label{fig:plot4}
    \end{subfigure}

    \vspace{4mm}
    
    \begin{subfigure}[t]{0.49\textwidth}
        \centering
        \includegraphics[width=\linewidth]{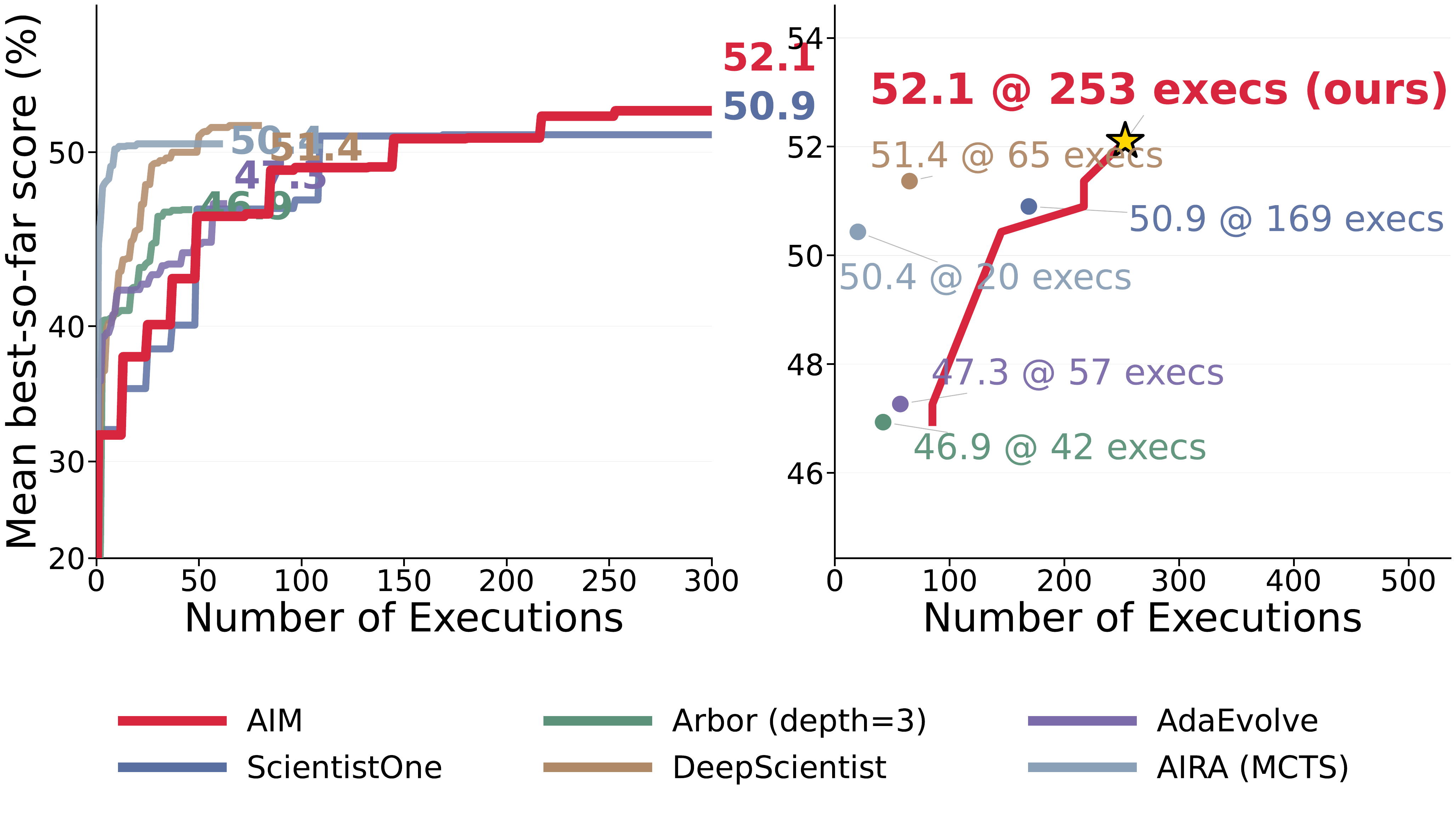}
        \caption{Z-order Range Scan}
        \label{fig:plot5}
    \end{subfigure}

    \caption{Execution-to-Score Plots Across AutoLab Tasks.}
    \label{fig:exec_to_score}
\end{figure*}

\clearpage
\begin{figure*}[h!]
    \centering

    \begin{subfigure}[t]{0.49\textwidth}
        \centering
        \includegraphics[width=\linewidth]{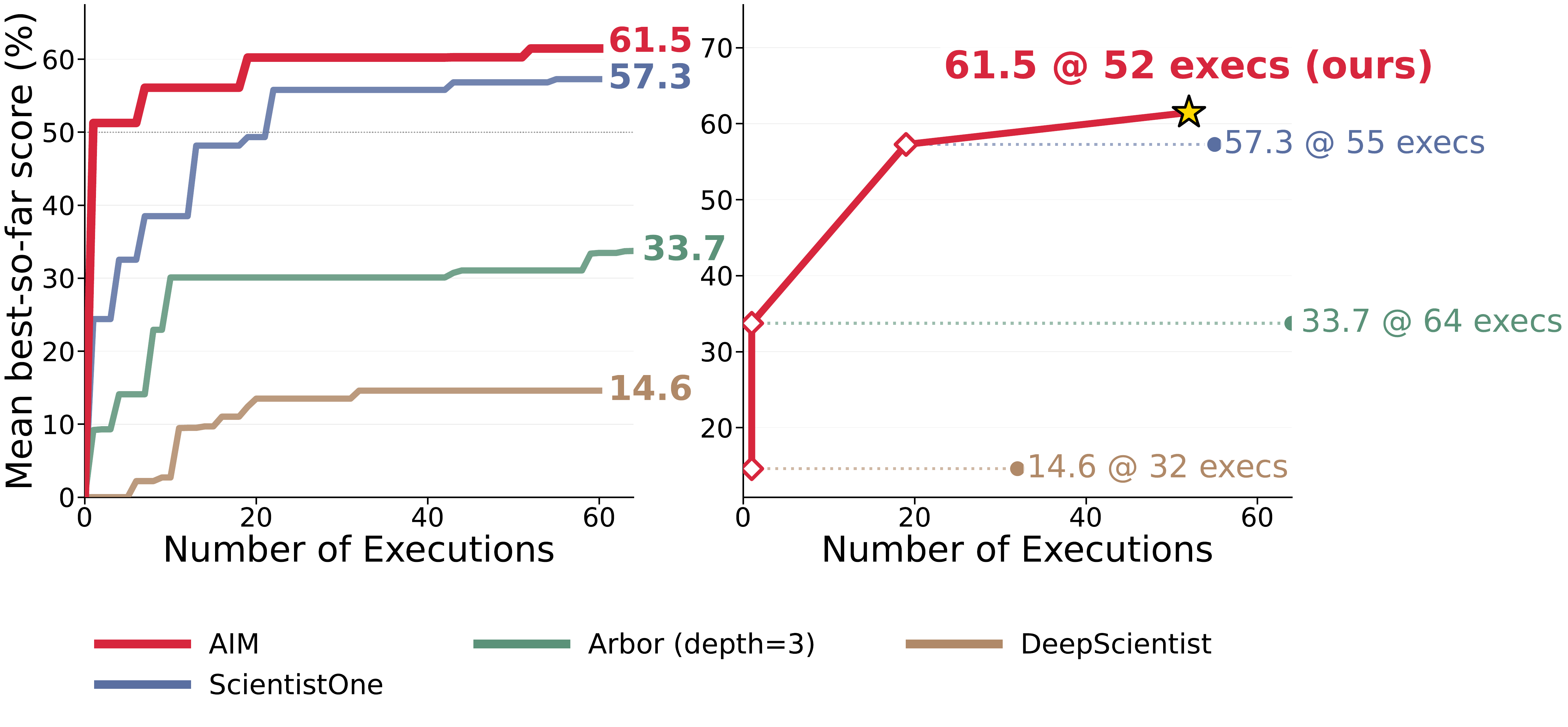}
        \caption{Moving Mnist World Model}
        \label{fig:plot2}
    \end{subfigure}
    \begin{subfigure}[t]{0.49\textwidth}
        \centering
        \includegraphics[width=\linewidth]{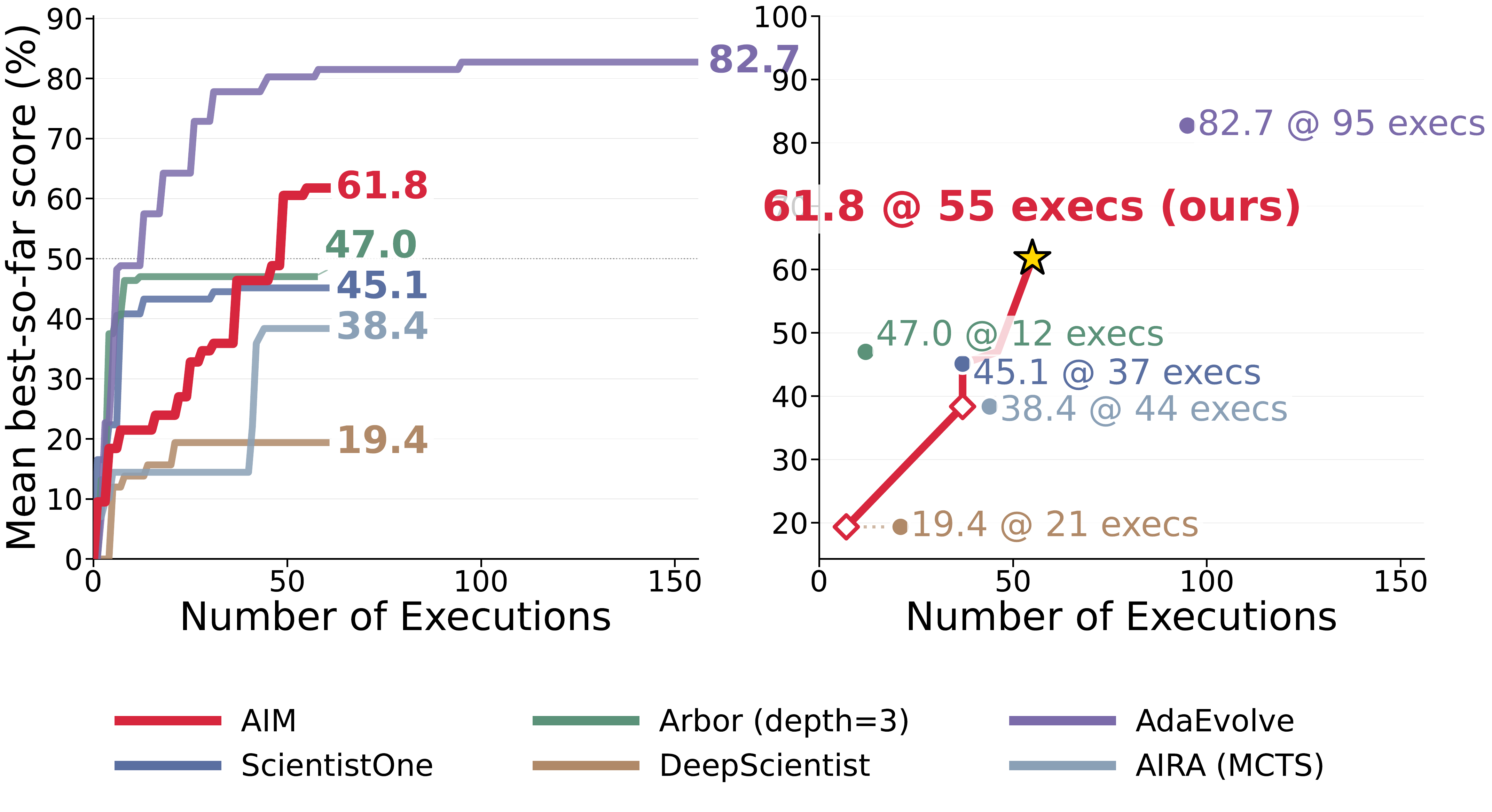}
        \caption{Data Select Ifeval}
        \label{fig:plot2}
    \end{subfigure}

    \vspace{4mm}
    
    \begin{subfigure}[t]{0.49\textwidth}
        \centering
        \includegraphics[width=\linewidth]{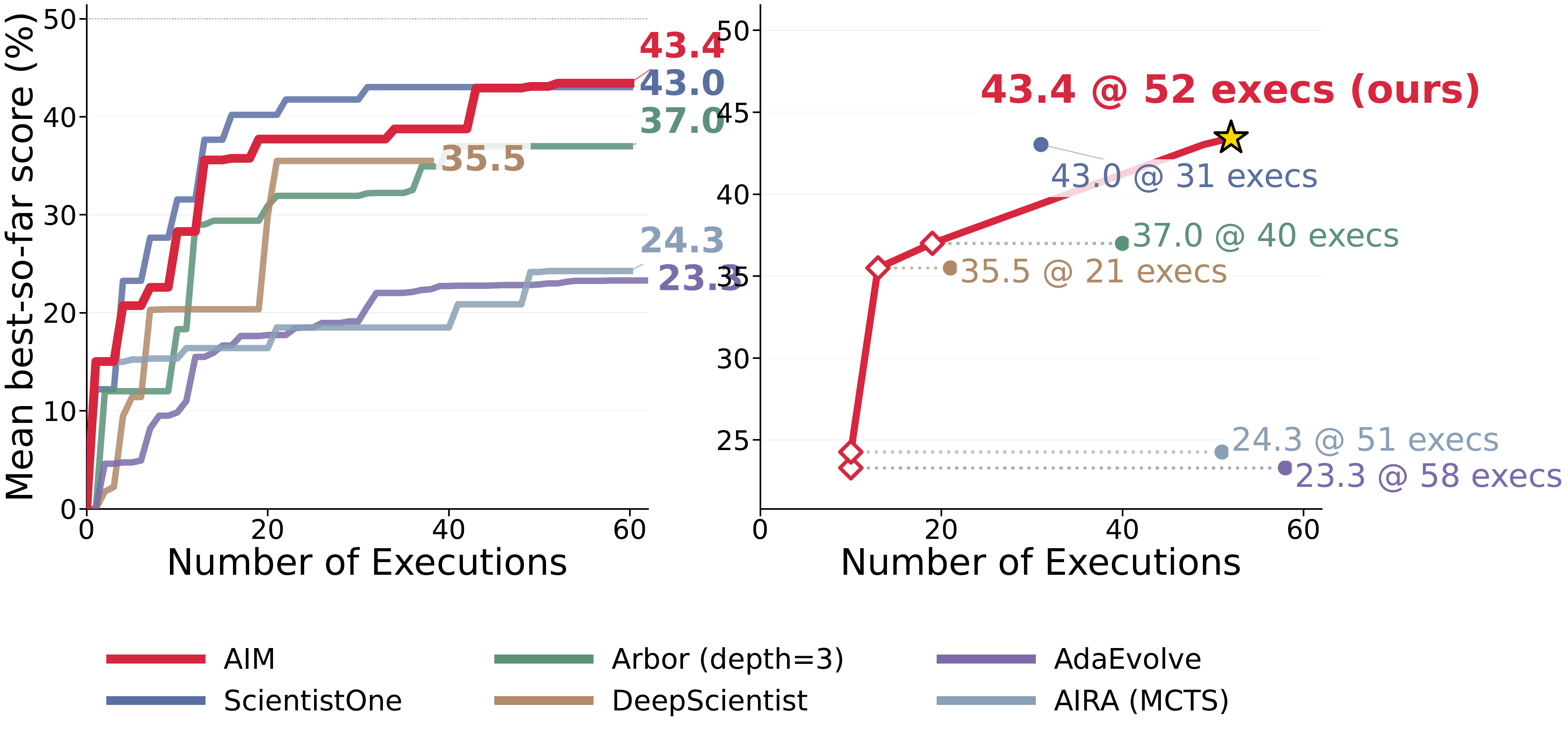}
        \caption{Huffman Canonical Decode}
        \label{fig:plot3}
    \end{subfigure}
    \hfill
    \begin{subfigure}[t]{0.49\textwidth}
        \centering
        \includegraphics[width=\linewidth]{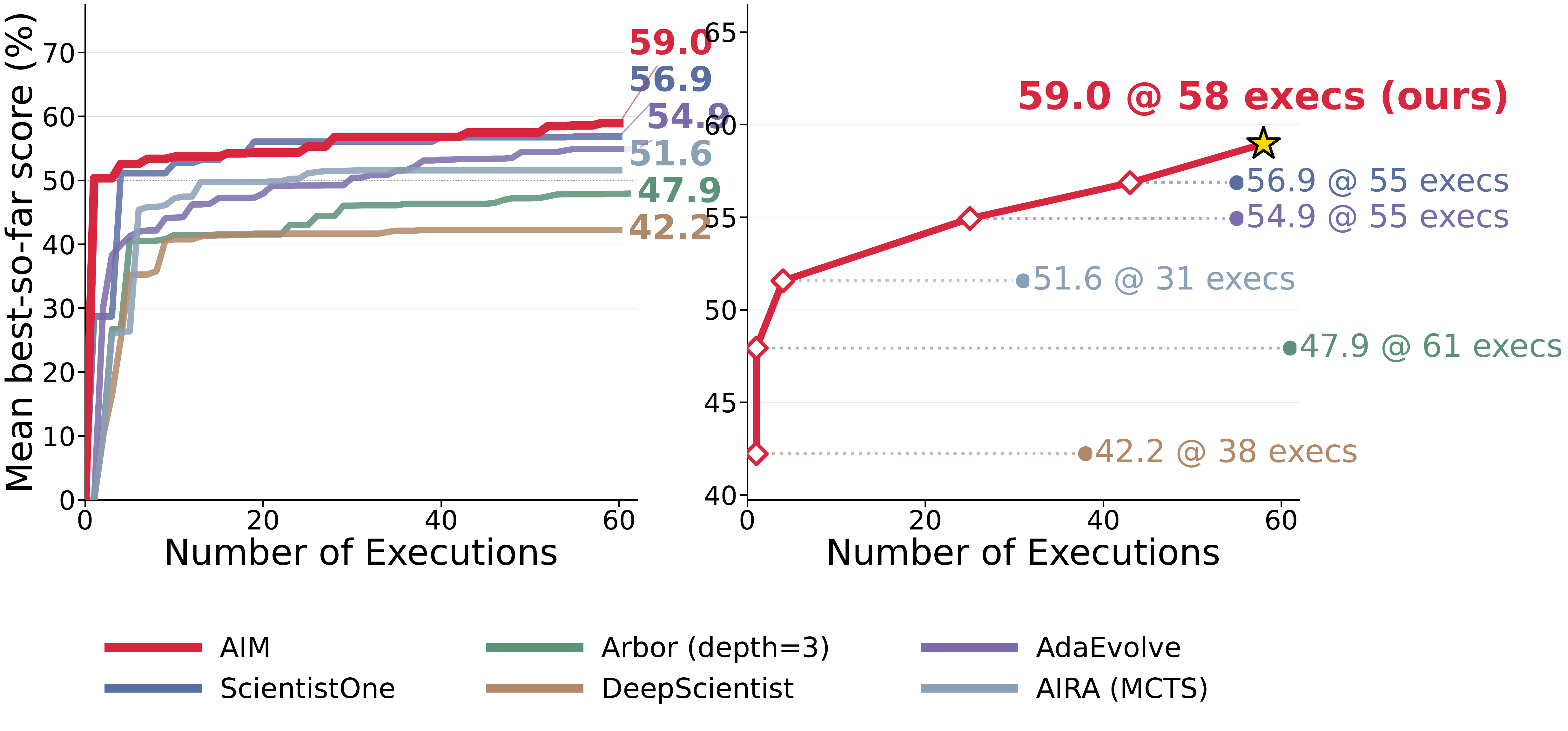}
        \caption{NTT Butterfly}
        \label{fig:plot4}
    \end{subfigure}

    \vspace{4mm}
    
    \begin{subfigure}[t]{0.49\textwidth}
        \centering
        \includegraphics[width=\linewidth]{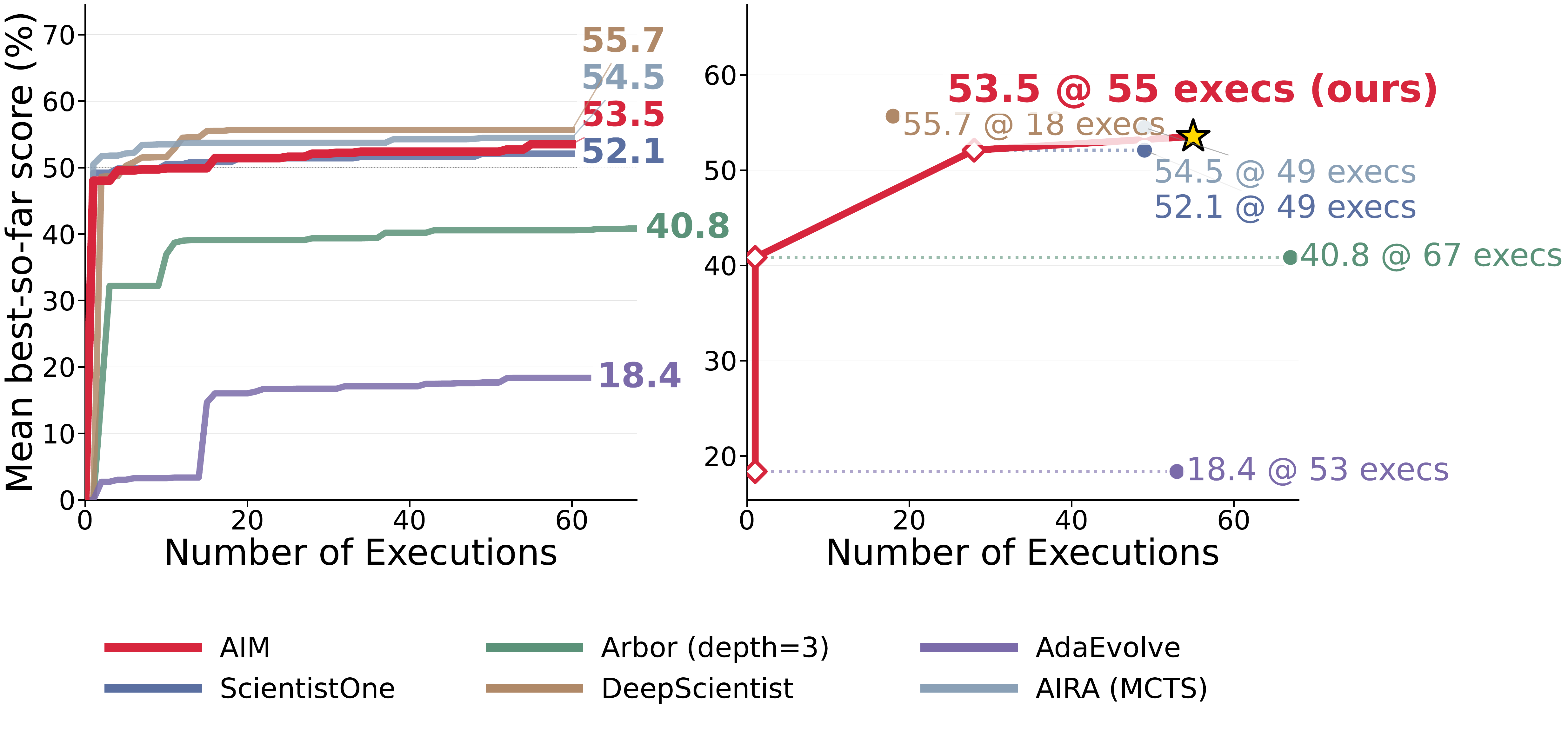}
        \caption{ICP Correspondence Step}
        \label{fig:plot5}
    \end{subfigure}

    \caption{Execution-to-Score Plots Across AutoLab Model Development \& CUDA Tasks.}
    \label{fig:exec_to_score_gpu}
\end{figure*}

\subsection{Token Cost Analysis}
Beyond wall-clock time and execution counts, we report token consumption (total prompt and generation tokens) across all LLM queries. 
Table~\ref{tab:token-cost} reports the token costs, measured by the number of LLM calls, input \& output tokens, along with the best score reached by each method.
While \textsc{AIM} generally requires a larger token budget, a great portion of it is from the large volume of context utilized in the method.
Also, considering the gains in scores, this demonstrates a trade-off between token cost and final performance.

In Figure~\ref{fig:token_cost}, we further demonstrate the score trend with respect to the token cost, comparing \textsc{AIM} with the strongest baseline, ScientistOne.
In 9 out of 10 tasks, \textsc{AIM} takes less tokens to reach the best score of ScientistOne, reducing the token cost up to 3.4$\times$ on the Data Select IFEval task.

\begin{table*}[h!]
\centering
\caption{LLM usage and cost per run on all ten AutoLab tasks (mean $\pm$ std over three runs). Output tokens include thinking token.}
\label{tab:token-cost}
\resizebox{\textwidth}{!}{%
\small\setlength{\tabcolsep}{5pt}%
\begin{tabular}{llccccc}
\toprule
Task & Metric & \textbf{AIM} (ours) & ScientistOne & Arbor & AIRA-MCTS & AIRA-EVO \\
\midrule
\multirow{4}{*}{Flash Attention} & \# LLM calls & $1740 \pm 86$ & $1550 \pm 40$ & $432 \pm 73$ & $121 \pm 19$ & $83 \pm 13$ \\
 & \# Input tokens ($\times10^{6}$) & $90.1 \pm 6.0$ & $70.3 \pm 8.8$ & $14.4 \pm 1.8$ & $0.9 \pm 0.2$ & $0.5 \pm 0.1$ \\
 & \# Output tokens ($\times10^{6}$) & $2.25 \pm 0.12$ & $1.77 \pm 0.15$ & $0.29 \pm 0.00$ & $3.17 \pm 0.06$ & $2.20 \pm 0.09$ \\
 & Best score (\%) & $\mathbf{90.5 \pm 0.8}$ & $86.2 \pm 3.1$ & $77.2 \pm 3.8$ & $77.9 \pm 0.2$ & $76.3 \pm 1.0$ \\
\midrule
\multirow{4}{*}{Radix Sort} & \# LLM calls & $1899 \pm 14$ & $1558 \pm 159$ & $286 \pm 64$ & $177 \pm 5$ & $178 \pm 0$ \\
 & \# Input tokens ($\times10^{6}$) & $121.1 \pm 6.6$ & $79.2 \pm 16.1$ & $7.3 \pm 3.3$ & $0.7 \pm 0.0$ & $0.6 \pm 0.0$ \\
 & \# Output tokens ($\times10^{6}$) & $1.34 \pm 0.05$ & $1.13 \pm 0.12$ & $0.16 \pm 0.05$ & $2.68 \pm 0.01$ & $2.23 \pm 0.05$ \\
 & Best score (\%) & $\mathbf{69.5 \pm 2.3}$ & $68.7 \pm 1.6$ & $58.2 \pm 4.2$ & $61.3 \pm 5.8$ & $66.7 \pm 2.5$ \\
\midrule
\multirow{4}{*}{AES128 Ctr} & \# LLM calls & $1844 \pm 70$ & $1762 \pm 95$ & $386 \pm 46$ & $165 \pm 33$ & $134 \pm 0$ \\
 & \# Input tokens ($\times10^{6}$) & $142.2 \pm 8.2$ & $123.5 \pm 9.6$ & $12.3 \pm 1.8$ & $1.0 \pm 0.2$ & $0.8 \pm 0.0$ \\
 & \# Output tokens ($\times10^{6}$) & $2.08 \pm 0.09$ & $1.78 \pm 0.25$ & $0.23 \pm 0.03$ & $2.91 \pm 0.16$ & $2.39 \pm 0.07$ \\
 & Best score (\%) & $66.4 \pm 0.4$ & $65.3 \pm 0.2$ & $\mathbf{67.2 \pm 0.1}$ & $62.6 \pm 0.3$ & $62.8 \pm 0.5$ \\
\midrule
\multirow{4}{*}{FFT Rust} & \# LLM calls & $1639 \pm 320$ & $1642 \pm 121$ & $365 \pm 50$ & $119 \pm 32$ & $90 \pm 0$ \\
 & \# Input tokens ($\times10^{6}$) & $70.6 \pm 29.7$ & $77.0 \pm 18.8$ & $8.2 \pm 1.8$ & $0.7 \pm 0.2$ & $0.5 \pm 0.0$ \\
 & \# Output tokens ($\times10^{6}$) & $1.64 \pm 0.02$ & $1.55 \pm 0.18$ & $0.22 \pm 0.04$ & $2.94 \pm 0.04$ & $2.26 \pm 0.27$ \\
 & Best score (\%) & $\mathbf{56.5 \pm 0.5}$ & $55.9 \pm 0.4$ & $53.5 \pm 1.1$ & $56.4 \pm 0.1$ & $56.2 \pm 0.4$ \\
\midrule
\multirow{4}{*}{Z-order Range Scan} & \# LLM calls & $1696 \pm 90$ & $1689 \pm 109$ & $275 \pm 99$ & $133 \pm 17$ & $97 \pm 25$ \\
 & \# Input tokens ($\times10^{6}$) & $103.1 \pm 9.1$ & $108.2 \pm 14.5$ & $7.4 \pm 1.9$ & $1.2 \pm 0.1$ & $0.8 \pm 0.2$ \\
 & \# Output tokens ($\times10^{6}$) & $1.90 \pm 0.03$ & $1.71 \pm 0.12$ & $0.18 \pm 0.04$ & $2.67 \pm 0.11$ & $2.15 \pm 0.18$ \\
 & Best score (\%) & $\mathbf{52.1 \pm 1.1}$ & $50.9 \pm 1.2$ & $46.9 \pm 2.1$ & $50.4 \pm 1.3$ & $51.1 \pm 0.8$ \\
\midrule
\multirow{4}{*}{MM World Model} & \# LLM calls & $4335 \pm 175$ & $3057 \pm 1880$ & $678 \pm 366$ & -- & -- \\
 & \# Input tokens ($\times10^{6}$) & $370.3 \pm 42.3$ & $281.7 \pm 195.3$ & $18.9 \pm 10.0$ & -- & -- \\
 & \# Output tokens ($\times10^{6}$) & $1.00 \pm 0.04$ & $0.60 \pm 0.36$ & $0.25 \pm 0.13$ & -- & -- \\
 & Best score (\%) & $\mathbf{61.5 \pm 0.9}$ & $57.3 \pm 6.3$ & $33.7 \pm 11.9$ & -- & -- \\
\midrule
\multirow{4}{*}{Data Select IFEval} & \# LLM calls & $6222 \pm 389$ & $12784 \pm 57$ & $447 \pm 321$ & $120 \pm 0$ & -- \\
 & \# Input tokens ($\times10^{6}$) & $631.6 \pm 126.2$ & $1524.0 \pm 18.8$ & $9.3 \pm 7.9$ & $0.4 \pm 0.1$ & -- \\
 & \# Output tokens ($\times10^{6}$) & $1.38 \pm 0.20$ & $2.74 \pm 0.06$ & $0.17 \pm 0.10$ & $0.81 \pm 0.51$ & -- \\
 & Best score (\%) & $\mathbf{61.8 \pm 3.2}$ & $45.1 \pm 1.1$ & $47.0 \pm 16.1$ & $38.4 \pm 5.9$ & -- \\
\midrule
\multirow{4}{*}{Huffman Decode} & \# LLM calls & $4102 \pm 261$ & $3906 \pm 88$ & $711 \pm 196$ & $119 \pm 2$ & -- \\
 & \# Input tokens ($\times10^{6}$) & $360.3 \pm 30.1$ & $339.8 \pm 9.9$ & $28.0 \pm 9.1$ & $0.7 \pm 0.0$ & -- \\
 & \# Output tokens ($\times10^{6}$) & $2.83 \pm 0.24$ & $2.60 \pm 0.16$ & $0.61 \pm 0.12$ & $4.04 \pm 0.69$ & -- \\
 & Best score (\%) & $\mathbf{43.4 \pm 1.3}$ & $43.0 \pm 2.6$ & $37.0 \pm 2.3$ & $24.3 \pm 10.8$ & -- \\
\midrule
\multirow{4}{*}{NTT Butterfly} & \# LLM calls & $4033 \pm 141$ & $3905 \pm 141$ & $361 \pm 34$ & $120 \pm 0$ & -- \\
 & \# Input tokens ($\times10^{6}$) & $386.0 \pm 17.1$ & $390.5 \pm 6.8$ & $10.6 \pm 1.8$ & $0.8 \pm 0.0$ & -- \\
 & \# Output tokens ($\times10^{6}$) & $2.87 \pm 0.25$ & $2.76 \pm 0.29$ & $0.27 \pm 0.08$ & $2.86 \pm 0.56$ & -- \\
 & Best score (\%) & $\mathbf{59.0 \pm 1.3}$ & $56.9 \pm 1.3$ & $47.9 \pm 5.0$ & $51.6 \pm 6.8$ & -- \\
\midrule
\multirow{4}{*}{ICP Corr.\ Step} & \# LLM calls & $4298 \pm 209$ & $3931 \pm 161$ & $423 \pm 71$ & $120 \pm 0$ & -- \\
 & \# Input tokens ($\times10^{6}$) & $371.8 \pm 16.9$ & $351.2 \pm 7.3$ & $11.5 \pm 1.1$ & $0.9 \pm 0.0$ & -- \\
 & \# Output tokens ($\times10^{6}$) & $2.51 \pm 0.22$ & $2.15 \pm 0.18$ & $0.24 \pm 0.02$ & $2.37 \pm 0.24$ & -- \\
 & Best score (\%) & $53.5 \pm 0.5$ & $52.1 \pm 1.0$ & $40.8 \pm 10.2$ & $\mathbf{54.5 \pm 0.9}$ & -- \\
\bottomrule
\end{tabular}%
}
\end{table*}

\begin{figure}[h!]
    \centering
    \includegraphics[width=\linewidth]{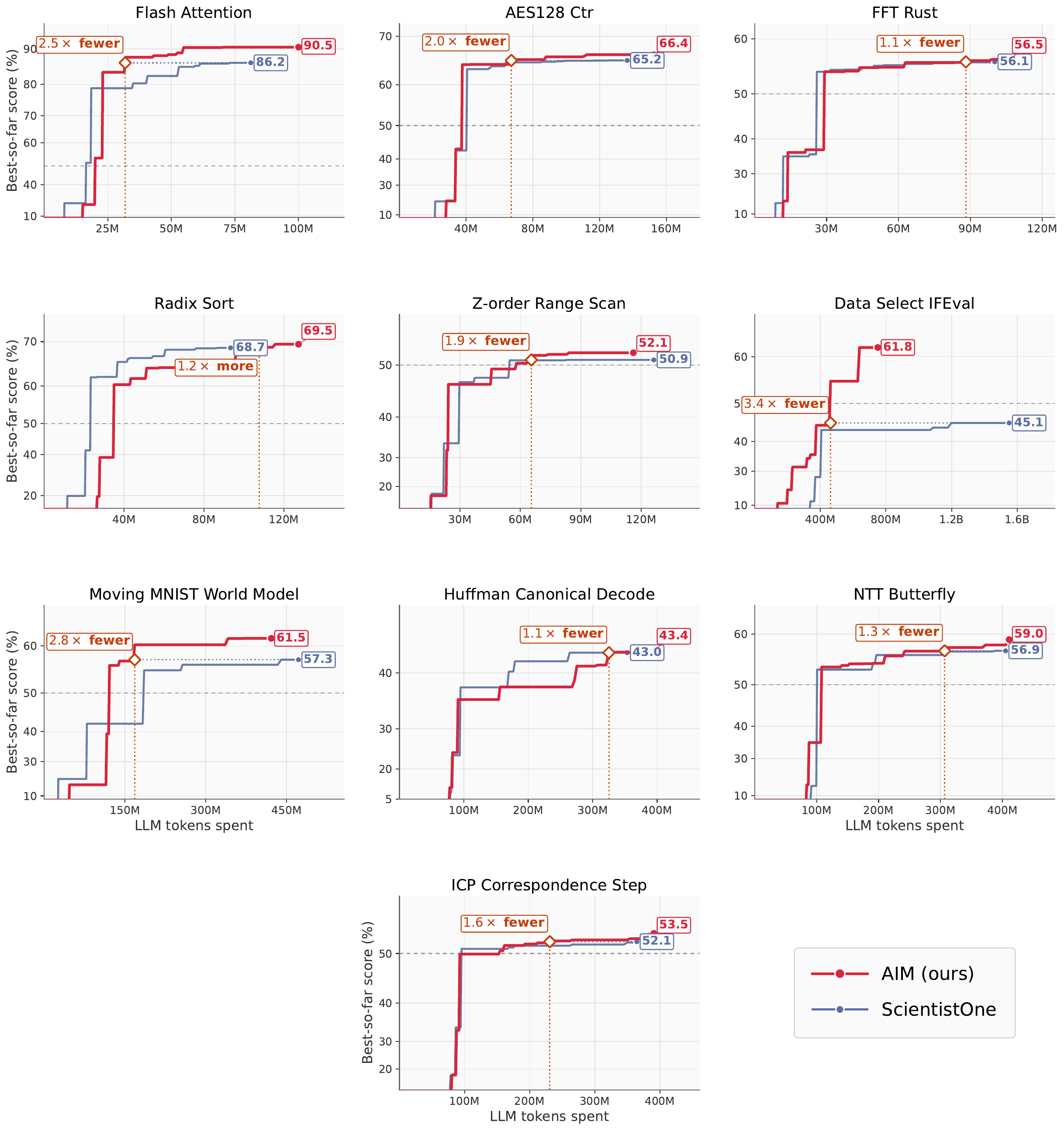}
    \caption{Token Cost Comparison between \textsc{AIM} and ScientistOne.}
    \label{fig:token_cost}
\end{figure}

\clearpage
\section{Further Analyses}
\label{apdx:experiments}

\subsection{More Qualitative Examples}
\label{apdx:morequal}


\begin{tcolorbox}[promptbox={Example 1 -- Flash Attention, iteration 1: cold start, nothing evaluated yet}]
\footnotesize
\noindent\textbf{(1) Clusters and rank estimates} \quad {\scriptsize pool = 10 ideas}

\smallskip
\setlength{\tabcolsep}{3pt}\renewcommand{\arraystretch}{1.05}
\begin{tabular}{@{}c c l c c c@{}}
\toprule
Rank & Cl. & Label & \#ideas & \#evaluated & best score\\
\midrule
\textbf{1} & [1] & Memory, Tiling \& Data Layout & 4 & 0 & -- \\
\textbf{2} & [2] & Advanced SIMD \& AVX-512 & 3 & 0 & -- \\
\textbf{3} & [0] & Exponentiation \& Math Optimization & 3 & 0 & -- \\
\bottomrule
\end{tabular}

\smallskip
\noindent\textit{``The primary bottleneck is memory bandwidth due to the $O(n^2)$ intermediate matrix; Tiling (Cluster 1) is mandatory [...]. SIMD (Cluster 2) handles the inner-loop dot products and is the next priority, while scalar math optimizations (Cluster 0) provide smaller marginal gains.''}

\tcbline

\noindent\textbf{(2) Actions and idea selection}

\smallskip
\noindent\textbf{b0} $\cdot$ cluster [1] \colorbox{teal!15}{\tiny\textsc{explore}} $\cdot$ idea \colorbox{orange!20}{\tiny\textsc{exploit}} $\rightarrow$ \textbf{Asymmetric Register-Only Tiling} {\tiny\texttt{rank 1/4}}\\
\textit{``We pick one of the tied top-ranked ideas: pinning Q and the accumulators in AVX registers while streaming K and V directly targets the fundamental memory-bandwidth bottleneck for $d{=}64$.''}

\smallskip
\noindent\textbf{b3} $\cdot$ cluster [2] \colorbox{teal!15}{\tiny\textsc{explore}} $\cdot$ idea \colorbox{teal!15}{\tiny\textsc{explore}} $\rightarrow$ \textbf{AVX-512 Micro-Kernel Dispatch} {\tiny\texttt{rank 1/3}}\\
\textit{``Although the prior step assigned an explore action, the strict Iter-1 constraint requires exploiting the highest-ranked predictions since no empirical scores exist yet.''}

\smallskip
\noindent\textbf{Outcome:} \textbf{b0} 0.764 (best) $\cdot$ \textbf{b1} 0.720 $\cdot$ \textbf{b2} 0.753 $\cdot$ \textbf{b3} 0.741 $\cdot$ \textbf{b4} 0.730
\end{tcolorbox}

\begin{tcolorbox}[promptbox={Example 2 -- Flash Attention, iteration 2: scores and lineage drive the choice}]
\footnotesize
\noindent\textbf{(1) Clusters and rank estimates} \quad {\scriptsize pool = 22 ideas}

\smallskip
\setlength{\tabcolsep}{3pt}\renewcommand{\arraystretch}{1.05}
\begin{tabular}{@{}c c l c c c@{}}
\toprule
Rank & Cl. & Label & \#ideas & \#evaluated & best score\\
\midrule
\textbf{1} & [2] & Memory Layout \& Cache Prefetching & 6 & 1 & \textbf{0.764} \\
\textbf{1} & [3] & AVX-512 \& Wide Micro-Kernels & 6 & 2 & \textbf{0.764} \\
\textbf{2} & [0] & Fast Math \& Exp Approximations & 5 & 2 & 0.753 \\
\textbf{2} & [4] & Register Unrolling \& ILP & 5 & 1 & 0.720 \\
\textbf{3} & [1] & Dynamic Sparsity \& Early Exit & 4 & 0 & -- \\
\bottomrule
\end{tabular}

\smallskip
\noindent\textit{``Clusters 2 and 3 share the highest evaluation scores (0.7639) [...]. Cluster 1 is unproven and ranked last, as dynamic sparsity usually introduces branching that undermines dense CPU SIMD performance.''}

\tcbline

\noindent\textbf{(2) Actions and idea selection}

\smallskip
\noindent\textbf{b0} $\cdot$ cluster [3] \colorbox{orange!20}{\tiny\textsc{exploit}} $\cdot$ idea \colorbox{orange!20}{\tiny\textsc{exploit}} $\rightarrow$ \textbf{Fully Unrolled AVX-512 Dot Products} {\tiny\texttt{rank 2/4}}\\
\textit{``We select this idea despite it being rank 2/4. Its parent achieved the highest overall score by introducing block tiling; this child fully unrolls the $D{=}64$ dot products into $4\times$ AVX-512 FMAs, applying a proven `simd' lesson from a sibling.''}

\smallskip
\noindent\textbf{b3} $\cdot$ cluster [1] \colorbox{teal!15}{\tiny\textsc{explore}} $\cdot$ idea \colorbox{teal!15}{\tiny\textsc{explore}} $\rightarrow$ \textbf{Dynamic-Threshold Sparsity} {\tiny\texttt{rank 3/4}}\\
\textit{``For an explore action we avoid the most promising member (rank 1/4) and pick a lower-confidence idea, validating the core premise of data-dependent sparsity before investing in exact underflow-bound checks.''}

\smallskip
\noindent\textbf{Outcome:} \textbf{b0} 0.846 $\cdot$ \textbf{b1} 0.850 $\cdot$ \textbf{b2} 0.889 (new best) $\cdot$ \textbf{b3} 0.450
\end{tcolorbox}

\clearpage
\begin{tcolorbox}[promptbox={Example 3 -- NTT Butterfly, iteration 3: exploit two proven clusters, probe one untouched}]
\footnotesize
\noindent\textbf{(1) Clusters and rank estimates} \quad {\scriptsize pool = 31 ideas}

\smallskip
\setlength{\tabcolsep}{3pt}\renewcommand{\arraystretch}{1.05}
\begin{tabular}{@{}c c l c c c@{}}
\toprule
Rank & Cl. & Label & \#ideas & \#evaluated & best score\\
\midrule
\textbf{1} & [0] & Global Topology \& Matrix Restructuring & 10 & 4 & \textbf{0.557} \\
\textbf{2} & [1] & Field Arithmetic \& PTX & 11 & 2 & 0.547 \\
\textbf{3} & [3] & Inner-Block Memory \& Warp Shuffles & 6 & 1 & 0.538 \\
\textbf{4} & [2] & Twiddle Factor Management & 4 & 0 & -- \\
\bottomrule
\end{tabular}

\smallskip
\noindent\textit{``Cluster 0 is ranked first because the distilled lessons strongly emphasize 2D matrix factorization and Bailey's 6-step as the critical bottleneck [...]. Cluster 2 has no evaluated signal yet and is generally a secondary optimization.''}

\tcbline

\noindent\textbf{(2) Actions and idea selection}

\smallskip
\noindent\textbf{b1} $\cdot$ cluster [1] \colorbox{orange!20}{\tiny\textsc{exploit}} $\cdot$ idea \colorbox{orange!20}{\tiny\textsc{exploit}} $\rightarrow$ \textbf{PTX Arithmetic Fusion} {\tiny\texttt{rank 1/9}}\\
\textit{``A rank-1 idea that combines the best arithmetic optimization (PTX lazy reduction) with the highest-scoring memory topology from the entire run (coalesced 2D transpose, 0.557).''}

\smallskip
\noindent\textbf{b2} $\cdot$ cluster [2] \colorbox{teal!15}{\tiny\textsc{explore}} $\cdot$ idea \colorbox{teal!15}{\tiny\textsc{explore}} $\rightarrow$ \textbf{Twiddle Cache in Constant/Texture Memory} {\tiny\texttt{rank 3/4}}\\
\textit{``This idea has the highest within-cluster rank number (3/4, lower confidence from the ranker). Testing constant/texture caching provides orthogonal information gain compared to compute-based regeneration or workspace precomputation.''}

\smallskip
\noindent\textbf{Outcome:} \textbf{b0} 0.557 $\cdot$ \textbf{b1} 0.559 (new best) $\cdot$ \textbf{b2} 0.460
\end{tcolorbox}

\begin{tcolorbox}[promptbox={Example 4 -- Data Select IFEval, iteration 3: double down on the leading cluster}]
\footnotesize
\noindent\textbf{(1) Clusters and rank estimates} \quad {\scriptsize pool = 33 ideas}

\smallskip
\setlength{\tabcolsep}{3pt}\renewcommand{\arraystretch}{1.05}
\begin{tabular}{@{}c c l c c c@{}}
\toprule
Rank & Cl. & Label & \#ideas & \#evaluated & best score\\
\midrule
\textbf{1} & [3] & Metadata Stratification \& Representations & 7 & 1 & \textbf{0.377} \\
\textbf{2} & [2] & Deterministic Heuristics \& Lexical Diversity & 9 & 3 & 0.285 \\
\textbf{3} & [0] & Base-Model IFD \& Curriculums & 8 & 2 & 0.193 \\
\textbf{4} & [1] & Generative Probing \& LLM-as-a-Judge & 9 & 2 & 0.137 \\
\bottomrule
\end{tabular}

\smallskip
\noindent\textit{``Cluster 3 achieves the highest evaluated score (0.3770) and aligns with strong lessons advocating for metadata distribution and length proxies. [...] Clusters 0 and 1 have much weaker top scores, with lessons explicitly warning against the generative probing approaches found in Cluster 1.''}

\tcbline

\noindent\textbf{(2) Actions and idea selection}

\smallskip
\noindent\textbf{b0} $\cdot$ cluster [3] \colorbox{orange!20}{\tiny\textsc{exploit}} $\cdot$ idea \colorbox{orange!20}{\tiny\textsc{exploit}} $\rightarrow$ \textbf{Source-Balanced IO-Length Stratification} {\tiny\texttt{rank 1/6}}\\
\textit{``A direct refinement of the best-performing idea so far: it iterates on the successful source-balancing strategy by incorporating output length to filter out terse responses.''}

\smallskip
\noindent\textbf{b1} $\cdot$ cluster [3] \colorbox{orange!20}{\tiny\textsc{exploit}} $\cdot$ idea \colorbox{teal!15}{\tiny\textsc{explore}} $\rightarrow$ \textbf{Unsupervised TF-IDF + KMeans Stratification} {\tiny\texttt{rank 5/6}}\\
\textit{``To execute an explore action within this top-performing cluster, we select a high-rank idea (5/6) that introduces a completely different mechanism: data-driven semantic boundaries instead of native `source' metadata.''}

\smallskip
\noindent\textbf{Outcome:} \textbf{b0} 0.119 $\cdot$ \textbf{b1} 0.377 (ties best) $\cdot$ \textbf{b2} 0.230
\end{tcolorbox}

\clearpage
\begin{tcolorbox}[promptbox={Example 5 -- Radix Sort, iteration 3: tie-breaking by lineage}]
\footnotesize
\noindent\textbf{(1) Clusters and rank estimates} \quad {\scriptsize pool = 37 ideas}

\smallskip
\setlength{\tabcolsep}{3pt}\renewcommand{\arraystretch}{1.05}
\begin{tabular}{@{}c c l c c c@{}}
\toprule
Rank & Cl. & Label & \#ideas & \#evaluated & best score\\
\midrule
\textbf{1} & [0] & LSD Pass \& Fused Histograms & 13 & 7 & \textbf{0.644} \\
\textbf{1} & [2] & Alternative Paradigms \& Adaptive Sorts & 10 & 2 & \textbf{0.644} \\
\textbf{2} & [1] & Memory Hierarchy \& Write-Combining & 10 & 3 & 0.568 \\
\textbf{3} & [3] & Scatter Buffering \& Pipelining & 4 & 1 & 0.544 \\
\bottomrule
\end{tabular}

\smallskip
\noindent\textit{``Clusters 0 and 2 are tied for Rank 1 as they share the highest best-evaluated score (0.6443) [...]. Cluster 3 ranks last due to its lowest peak score, although with only 1 evaluated idea it retains some exploration value.''}

\tcbline

\noindent\textbf{(2) Actions and idea selection}

\smallskip
\noindent\textbf{b0} $\cdot$ cluster [0] \colorbox{orange!20}{\tiny\textsc{exploit}} $\cdot$ idea \colorbox{orange!20}{\tiny\textsc{exploit}} $\rightarrow$ \textbf{Fused 11-Bit Radix with Loop-Unrolled Prefetching} {\tiny\texttt{rank 1/6}}\\
\textit{``There is a tie at rank 1/6 between two prefetch-based refinements. I chose this one because its parent achieved the highest overall score in the run (0.6443), making it the most promising trajectory to aggressively exploit.''}

\smallskip
\noindent\textbf{b1} $\cdot$ cluster [0] \colorbox{orange!20}{\tiny\textsc{exploit}} $\cdot$ idea \colorbox{teal!15}{\tiny\textsc{explore}} $\rightarrow$ \textbf{Restricted Pointers \& Unrolled 12-12-8 Scatter} {\tiny\texttt{rank 5/6}}\\
\textit{``While evaluated members have focused on prefetching and fused histograms, this idea tests whether compiler-level optimizations like \texttt{restrict} pointers and manual unrolling can significantly accelerate the 12-12-8 radix scatter.''}

\smallskip
\noindent\textbf{Outcome:} \textbf{b0} 0.661 (new best) $\cdot$ \textbf{b1} 0.536 $\cdot$ \textbf{b2} 0.570 $\cdot$ \textbf{b3} 0.546
\end{tcolorbox}

\subsection{Qualitative Mechanisms of the Organize Operator}
\label{apdx:organize}

One strong advantage of idea-driven approaches is its interpretability in the research trajectory; it is easy to follow the search trajectory summarized by the idea traces.
As a qualitative analysis, we examine how the Organize operator restructures the discovered idea pool as new candidates and experimental evidence become available. 
In Figure~\ref{fig:organize-heatmap}, we provide example heatmaps of the best score from each cluster explored, and show how the exploration evolves throughout iterations.

\begin{figure}[h!]
    \centering

    \begin{subfigure}{\linewidth}
        \centering
        \includegraphics[width=\linewidth]{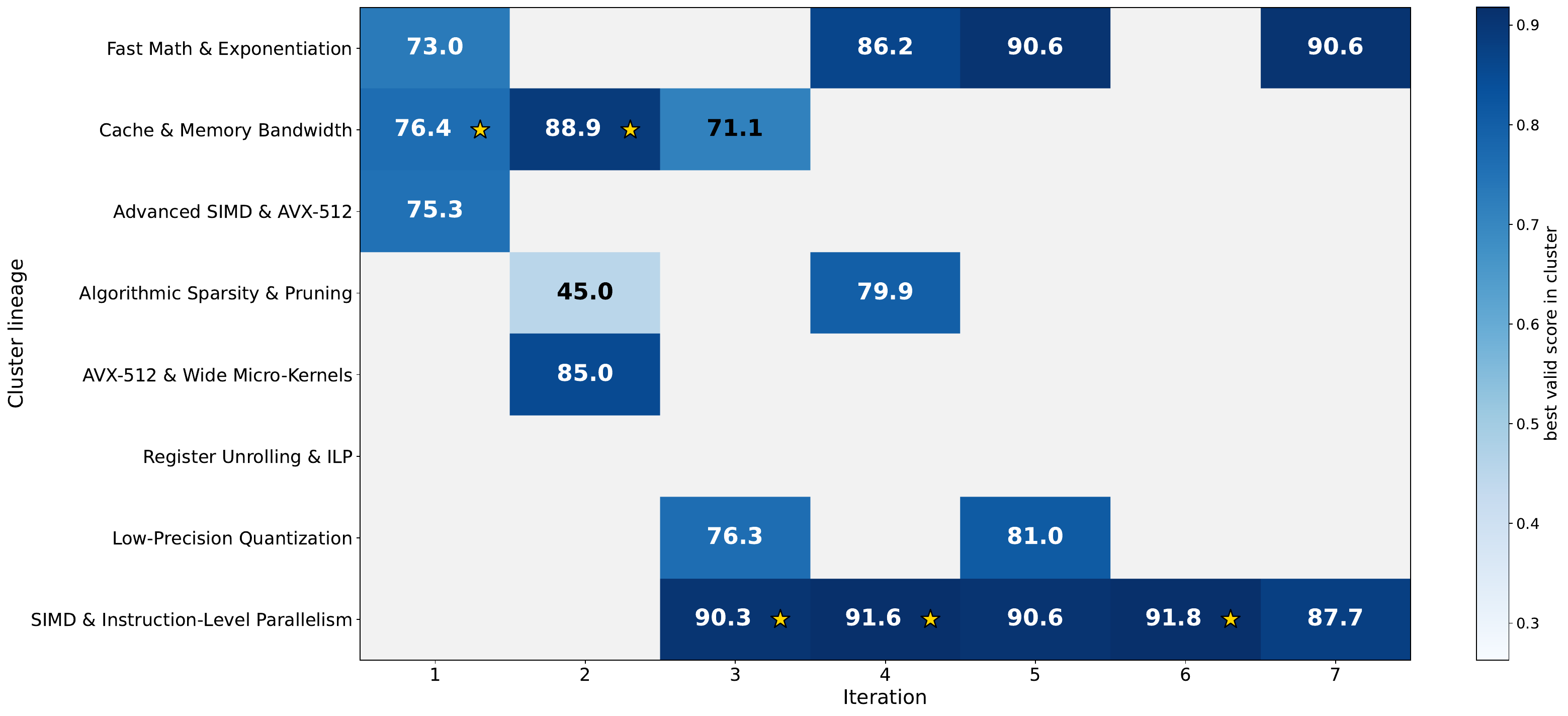}
        \caption{Flash Attention: Cluster-wise Best Score Progression.}
        \label{fig:flash_attention}
    \end{subfigure}

    \vspace{2mm}

    \begin{subfigure}{\linewidth}
        \centering
        \includegraphics[width=\linewidth]{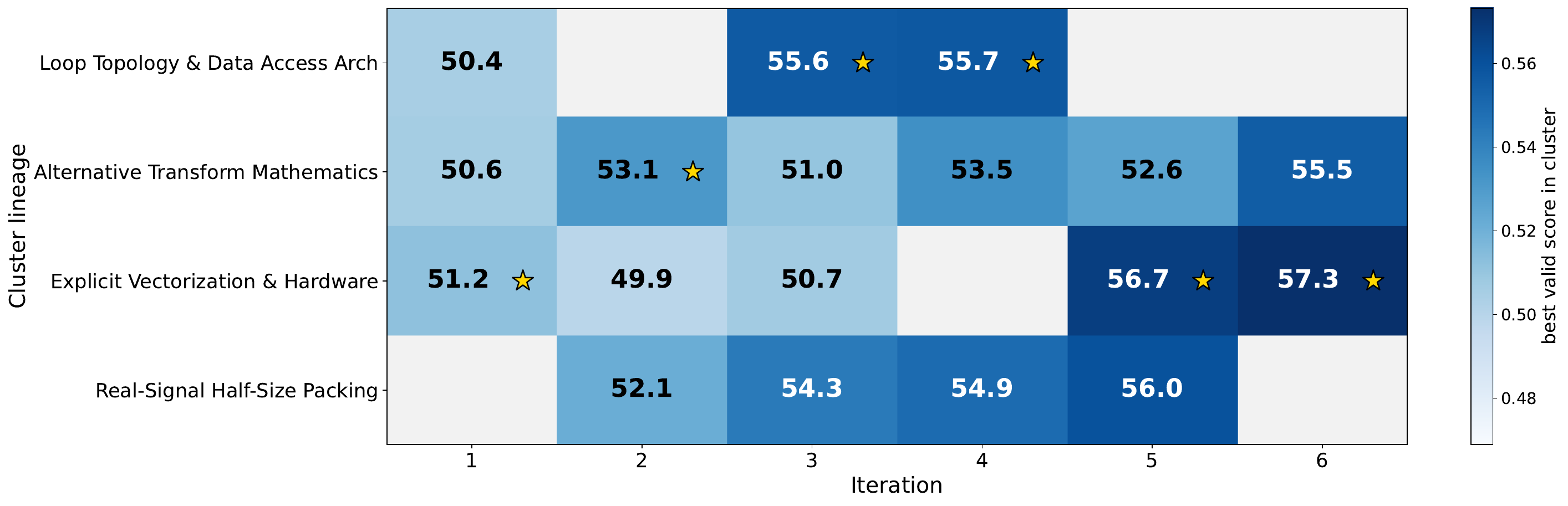}
        \caption{FFT Rust: Cluster-wise Best Score Progression.}
        \label{fig:fft_rust}
    \end{subfigure}

    \caption{\textbf{Agentic Surrogate: ORGANIZE}. Heatmap of the best score of each cluster explored in each iteration. The star symbol marks the point when and where the new best scoring solution was discovered.}
    \label{fig:organize-heatmap}
\end{figure}

On one example run on Flash Attention~(Figure~\ref{fig:organize-heatmap} (a)), the search initially explores several directions, including memory optimization, mathematical approximation, and vectorization. 
As evidence accumulates, SIMD and instruction-level parallelism emerges as a consistently strong direction, while fast exponentiation remains competitive in later iterations. 
Overall, 8 different cluster themes were proposed, while 7 of them were explored throughout the iterations.
Note that the ideas comprising each cluster theme is not mutually exclusive; the list of ideas are re-clustered every iteration, so the same idea could have been regrouped into a different cluster in the subsequent rounds.
On FFT Rust~(Figure~\ref{fig:organize-heatmap} (b)), on the other hand, the leading direction changes repeatedly among transform reformulation, loop and cache topology, real-signal packing, and explicit vectorization. 
Compared to the Flash Attention task that spanned 8 semantic clusters, FFT Rust had a narrower breadth of search with 4 different clusters explored. 
This difference in trend largely depends on the trait of the task in question.

\begin{figure}[h!]
    \centering

    \begin{subfigure}{\linewidth}
        \centering
        \includegraphics[width=\linewidth]{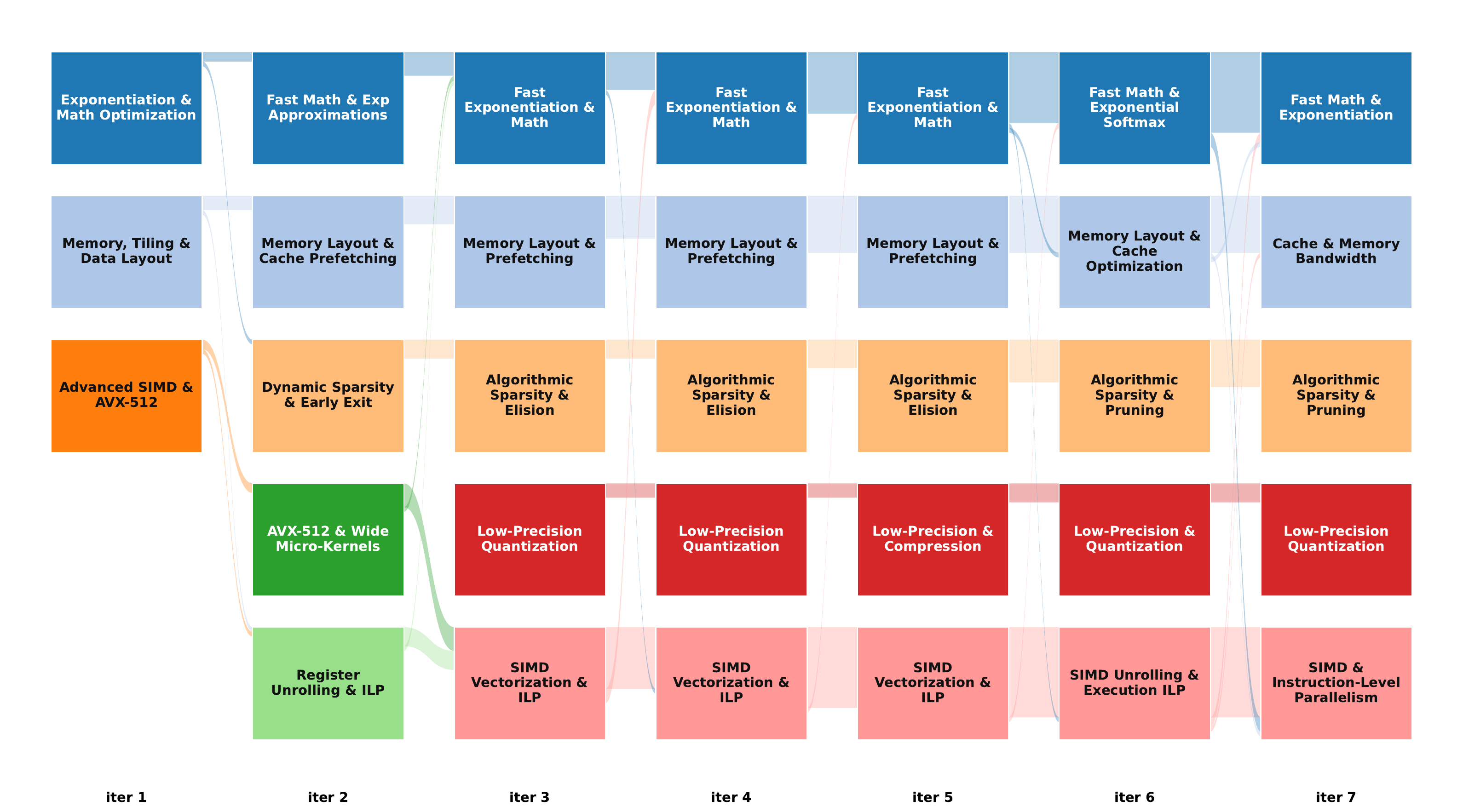}
        \caption{Flash Attention: Alluvial Plot.}
        \label{fig:flash_attention}
    \end{subfigure}

    \vspace{2mm}

    \begin{subfigure}{\linewidth}
        \centering
        \includegraphics[width=\linewidth]{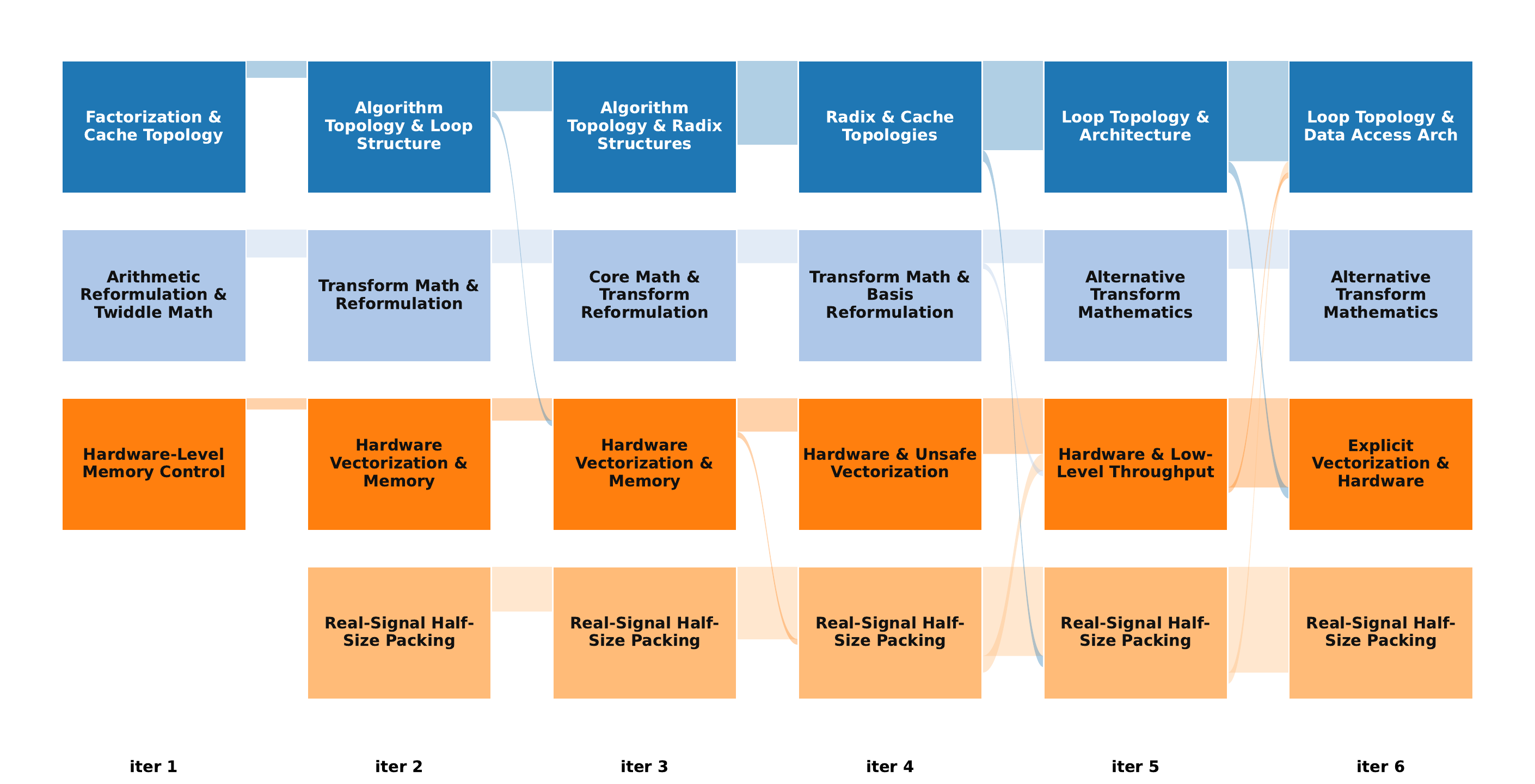}
        \caption{FFT Rust: Alluvial Plot.}
        \label{fig:fft_rust}
    \end{subfigure}

    \caption{\textbf{Agentic Surrogate: ORGANIZE}. Alluvial plots showing how the ideas flow and how they are re-organized throughout iterations.}
    \label{fig:alluvial}
\end{figure}

The alluvial diagrams in Figure~\ref{fig:alluvial} further visualize how clusters persist, split, merge, and change names across, showing the fine-grained flow and structure of ideas evolving across iterations.
Despite these structural revisions at every iteration, coherent semantic trajectories remain visible. 
For example, the broad SIMD direction in Flash Attention develops into more specialized directions involving wide micro-kernels, register unrolling, low-precision computation, and instruction-level parallelism. 
Similarly, the FFT Rust task progressively refines broad hardware- and topology-oriented clusters into explicit vectorization and data-access strategies. 
\textit{Thus, the Organize operator preserves semantic continuity while allowing the representation to adapt beyond a fixed list, tree, or generation lineage}.

\clearpage
\subsection{Examining the Preciseness of the Estimate Operator}
\label{apdx:estimate}

We next examine whether the ordinal promisingness estimates produced by the Estimate operator agree with subsequently observed verifier scores. 
For each cluster, we measure the Spearman correlation between the cluster-level rank estimates and the actual rank returned by the executions, and average them across iterations.

As seen in Figure~\ref{fig:spearman}, the estimates are positively correlated with observed performance across all five tasks, with mean cluster-wise correlations ranging from $0.56$ to $0.91$ and an overall average of approximately $0.71$. 
The strongest agreement occurs on Z-order Range Scan and Flash Attention, while the remaining tasks retain moderate positive correlations. 
These results indicate that the agent can extract useful ranking signals from the organized idea map, evaluation history, and implementation lessons without predicting calibrated reward values.

\begin{figure}[h!]
    \centering
    \includegraphics[width=0.8\linewidth]{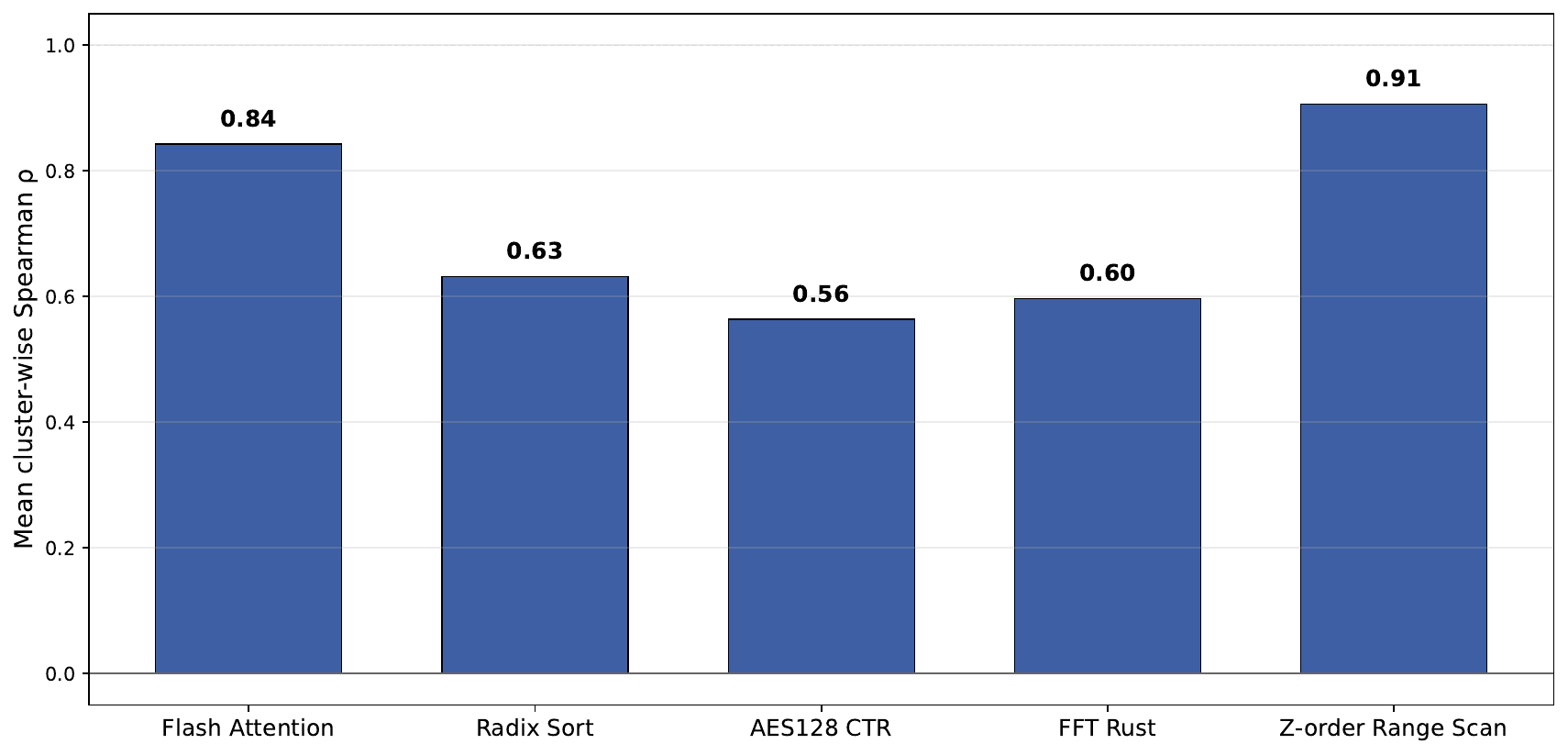}
    \caption{Cluster-level Ordinal Estimate's Average Spearman Correlation.}
    \label{fig:spearman}
\end{figure}

The iteration-wise results in Figure~\ref{fig:estimate_trend} show how these estimates change as evidence accumulates. 
Correlation generally improves after the initial iterations: for example, the estimate on Radix Sort recovers from a poor initial ranking to nearly perfect agreement in later iterations, while Flash Attention and Z-order Range Scan maintain strong agreement over much of the search. 
The trajectories are not strictly monotonic, however, particularly for AES128-CTR and FFT.
We conjecture this might be because (1) there are multiple competing clusters that have similar expected performance, and/or (2) there might be implementation noise~(i.e., minor implementation details can change the measured performance).
Nevertheless, the estimates remain informative enough to ground subsequent exploration and exploitation decisions in observed research progress.

\begin{figure*}[h!]
    \centering

    \begin{subfigure}[t]{0.32\textwidth}
        \centering
        \includegraphics[width=\linewidth]{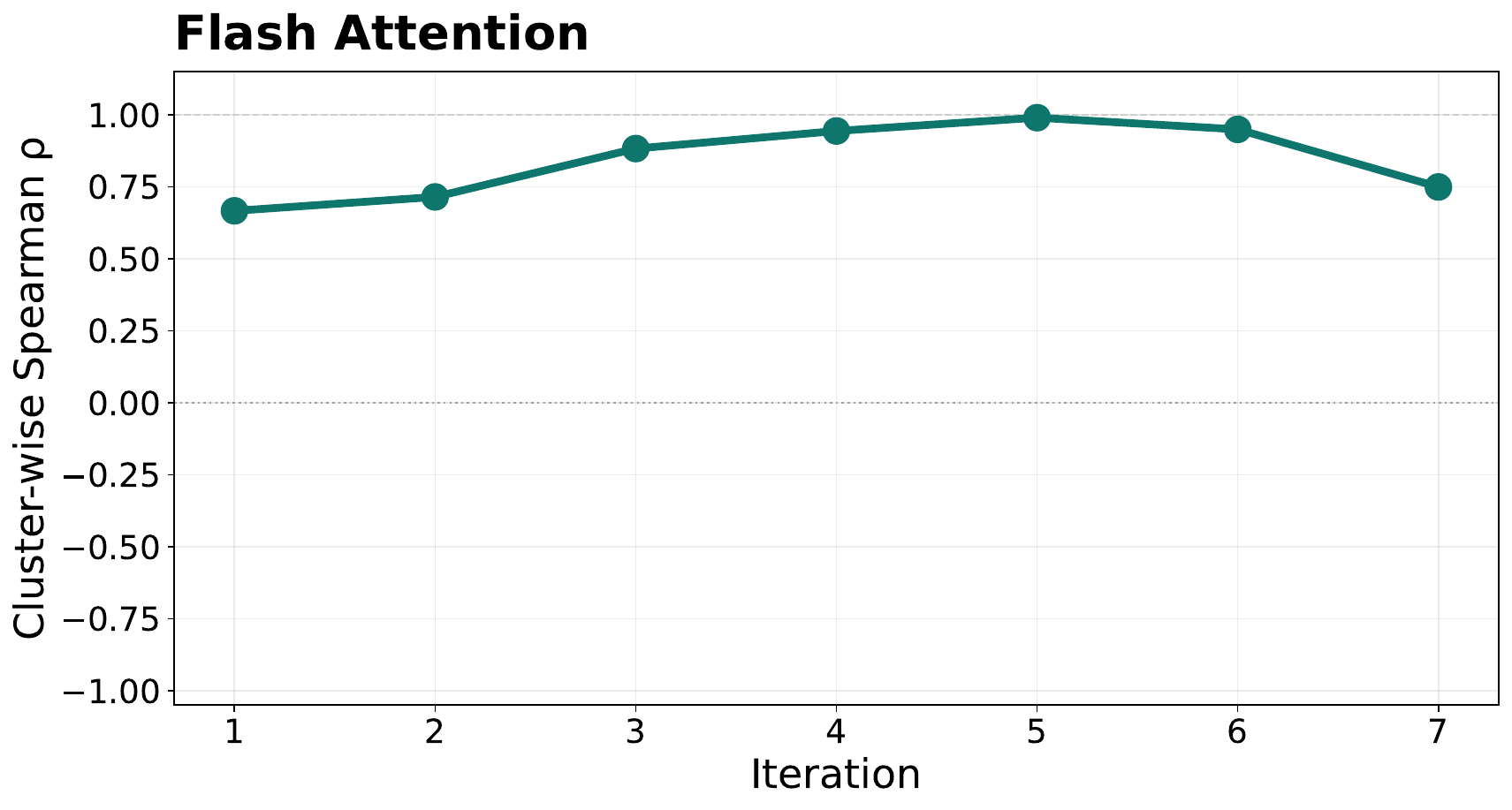}
        \caption{Flash Attention}
        \label{fig:plot1}
    \end{subfigure}
    \hfill
    \begin{subfigure}[t]{0.32\textwidth}
        \centering
        \includegraphics[width=\linewidth]{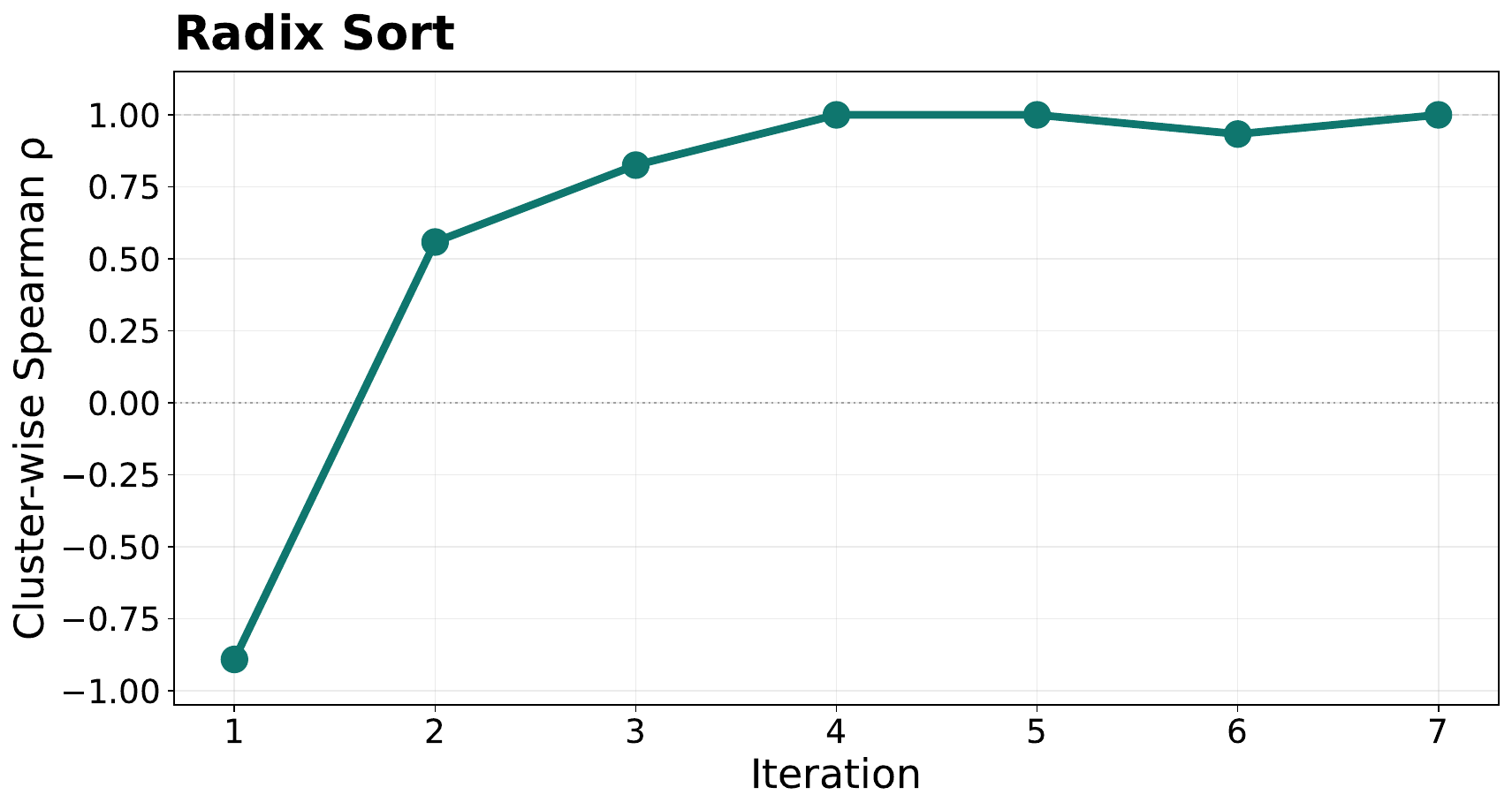}
        \caption{Radix Sort}
        \label{fig:plot2}
    \end{subfigure}
    \hfill
    \begin{subfigure}[t]{0.32\textwidth}
        \centering
        \includegraphics[width=\linewidth]{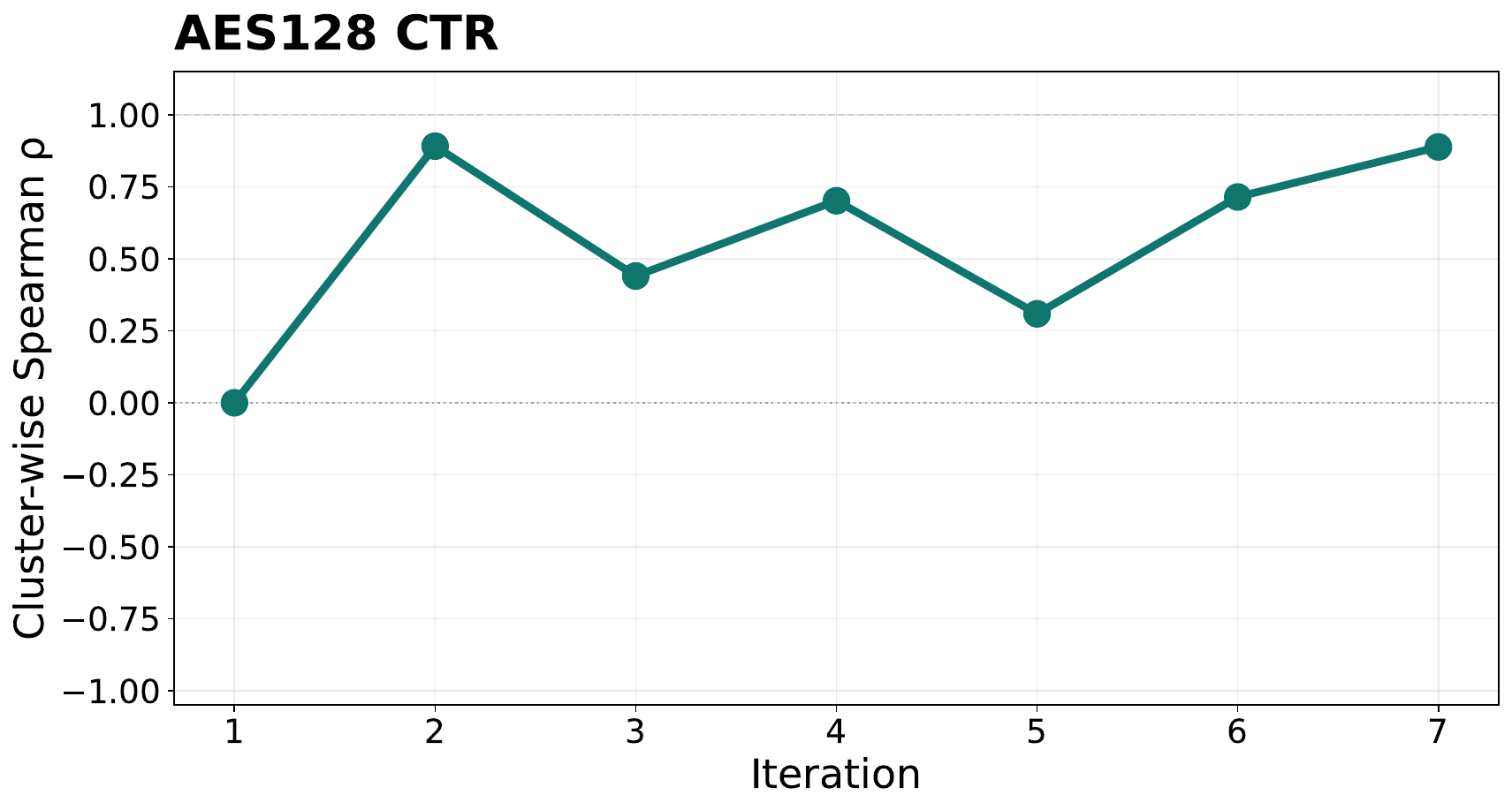}
        \caption{AES128 Ctr}
        \label{fig:plot3}
    \end{subfigure}

    \vspace{2mm}

    \begin{subfigure}[t]{0.32\textwidth}
        \centering
        \includegraphics[width=\linewidth]{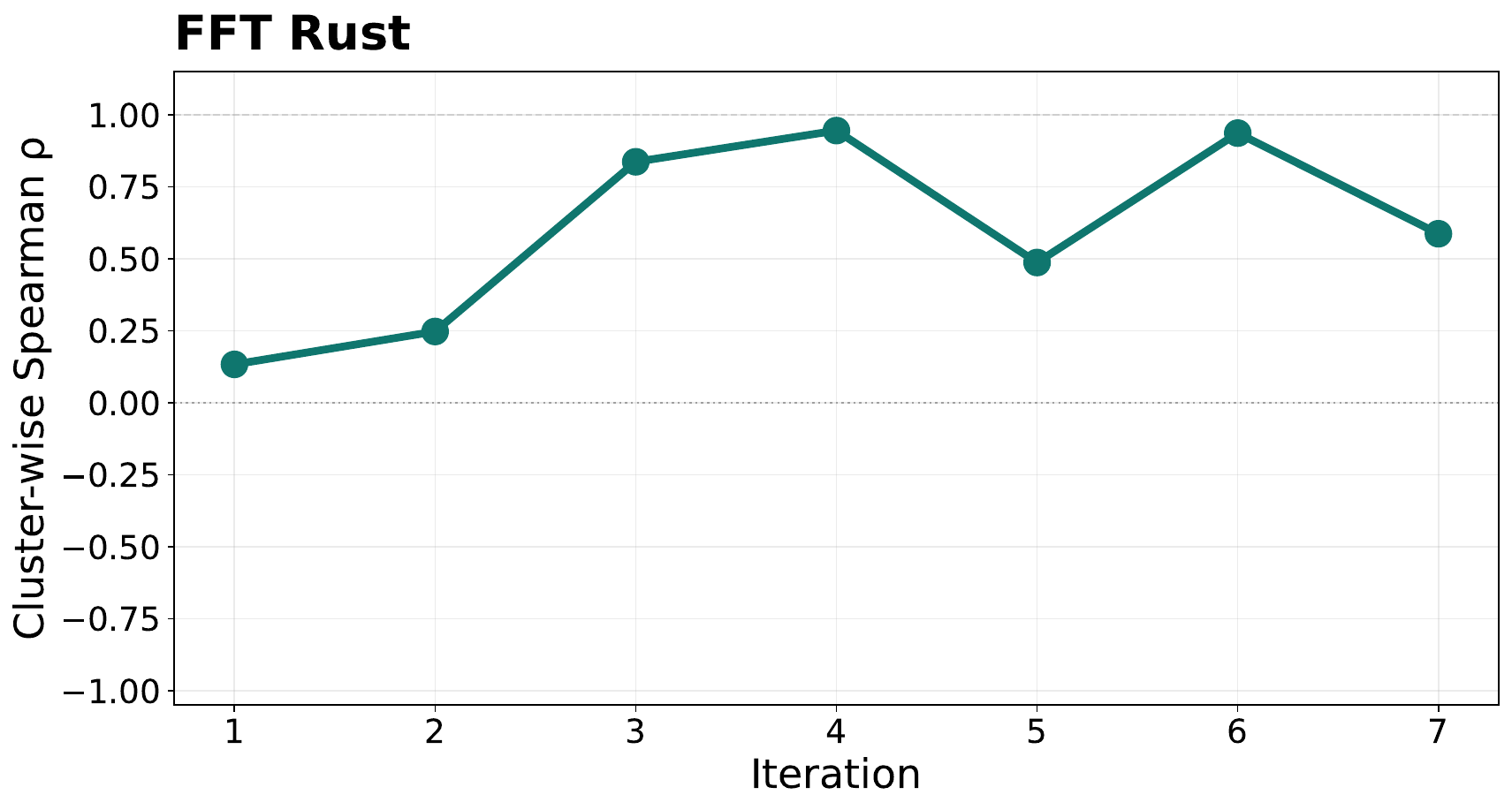}
        \caption{FFT Rust}
        \label{fig:plot4}
    \end{subfigure}
    \hspace{0.02\textwidth}
    \begin{subfigure}[t]{0.32\textwidth}
        \centering
        \includegraphics[width=\linewidth]{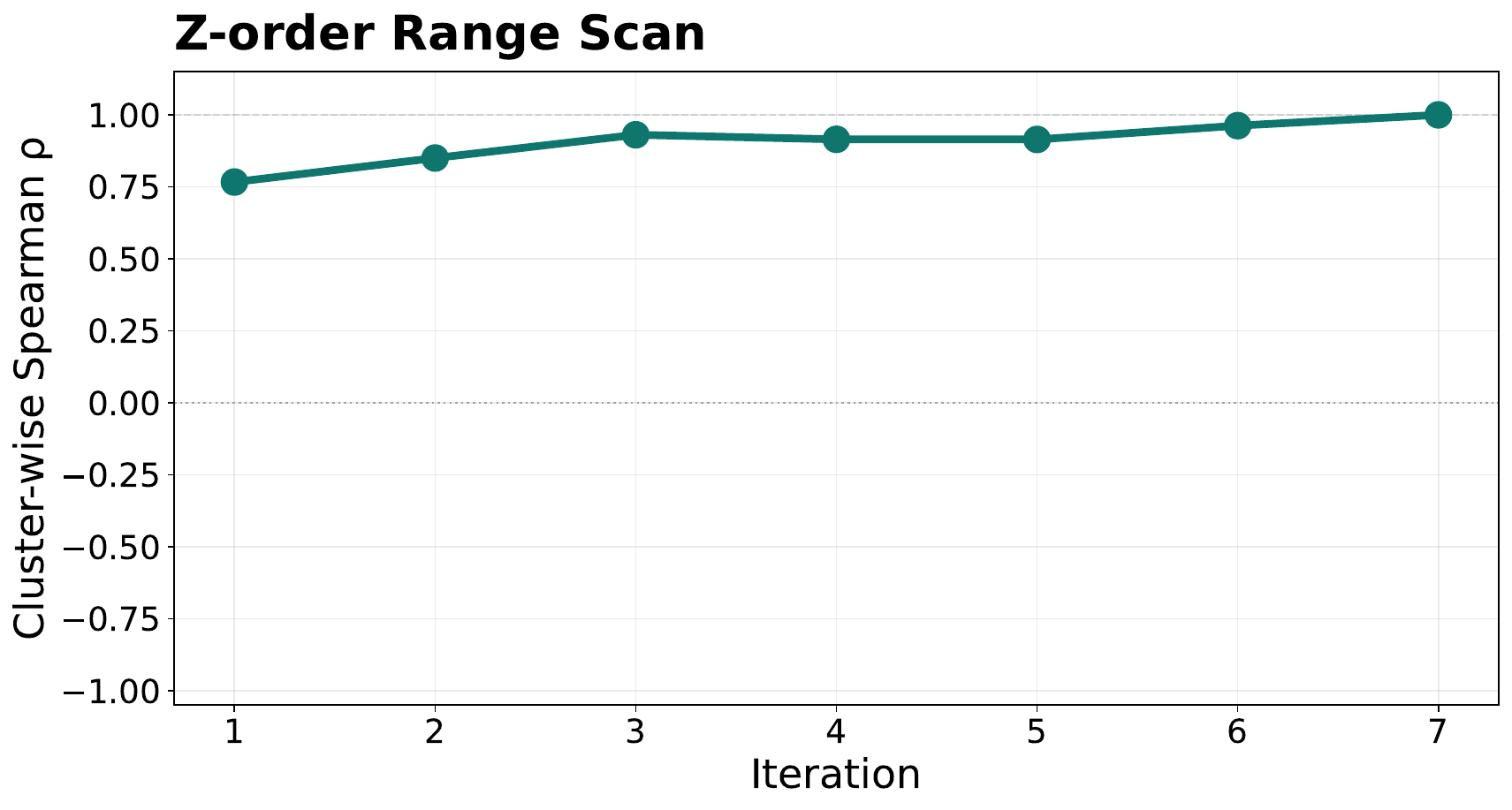}
        \caption{Z-order Range Scan}
        \label{fig:plot5}
    \end{subfigure}

    \caption{Trend of Spearman Correlation across Iterations.}
    \label{fig:estimate_trend}
\end{figure*}

\subsection{Dispatch Operator Action Analysis}
\label{apdx:dispatch}

The Dispatch operator makes idea selection decisions at two levels for each Solver branch: whether to explore or exploit across clusters, and whether to explore or exploit among ideas within the selected cluster. 
Motivated by recent findings that LLM agents often fail to explore systematically~\citep{krishnamurthy2024can,pan2026large,park2026exploration,choi2026multi}, \textsc{AIM} requires the agent to explicitly verbalize these decisions. 
Unlike approaches that rely on externally specified search heuristics~\citep{yamada2025ai,toledo2026ai} or direct LLM-based candidate selection~\citep{weng2026deepscientist}, \textsc{AIM} conditions its actions on the agent's explicit assessment of the current idea map and promisingness estimates. 
This makes the exploration--exploitation policy evidence-grounded and interpretable.

Figure~\ref{fig:dispatch_dist} shows the distribution of the resulting two-level actions, represented as (\emph{cluster-level action}, \emph{idea-level action}), across iterations of the five tasks. 
Early iterations generally allocate more branches to actions involving exploration, whereas exploitation becomes increasingly prevalent in later iterations. 
This shift is intuitive: when evidence is sparse, the agent broadens coverage across clusters and ideas; as experimental evidence accumulates, it increasingly concentrates resources on directions estimated to be promising. 
The task-dependent variation in these distributions further indicates that the search policy is adapted to research progress rather than fixed in advance.

\begin{figure*}[h!]
    \centering

    \begin{subfigure}[t]{0.48\textwidth}
        \centering
        \includegraphics[width=\linewidth]{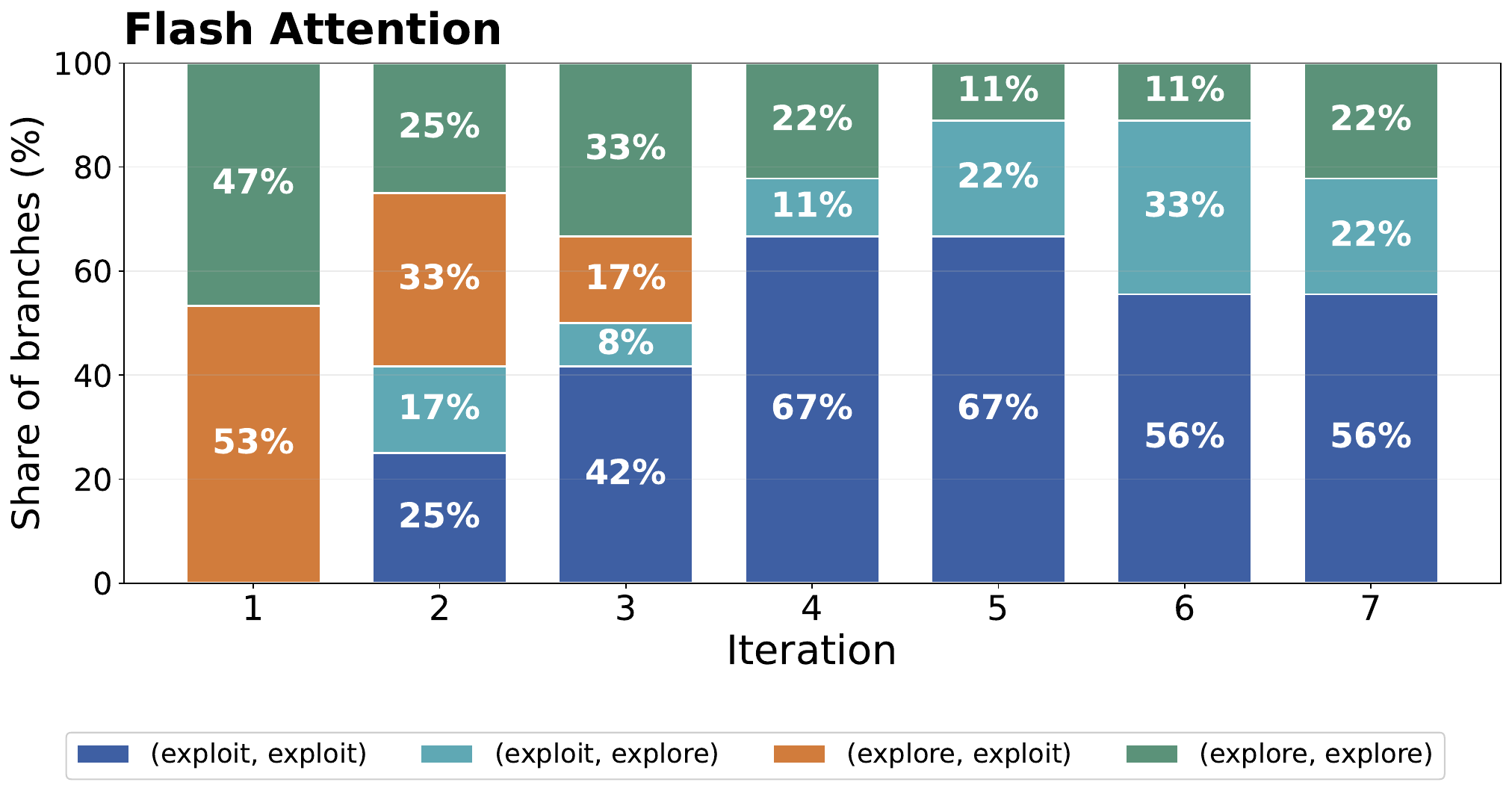}
        \caption{Flash Attention}
        \label{fig:plot1}
    \end{subfigure}
    \hfill
    \begin{subfigure}[t]{0.48\textwidth}
        \centering
        \includegraphics[width=\linewidth]{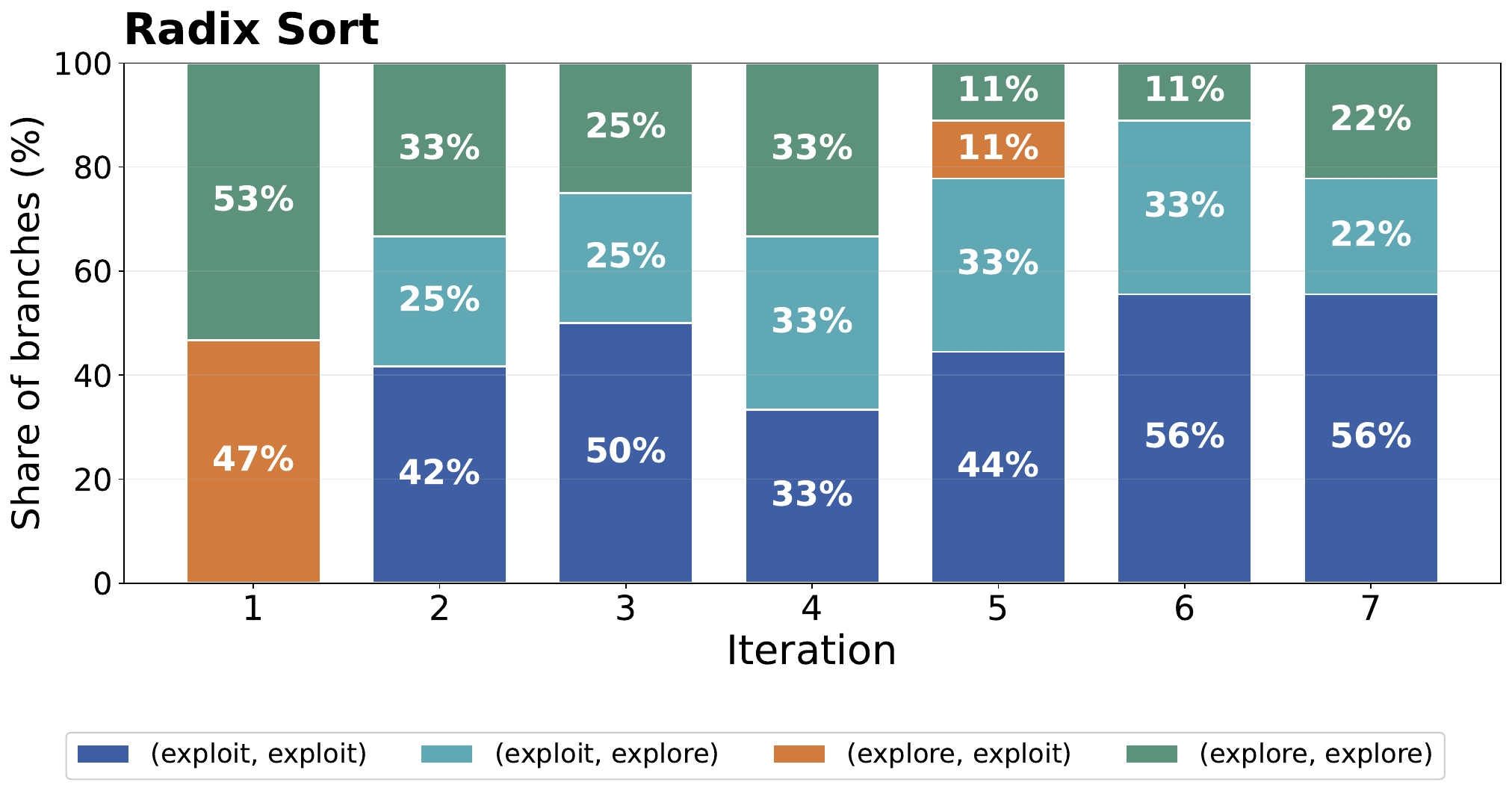}
        \caption{Radix Sort}
        \label{fig:plot2}
    \end{subfigure}
    
    \vspace{2mm}
    
    \begin{subfigure}[t]{0.48\textwidth}
        \centering
        \includegraphics[width=\linewidth]{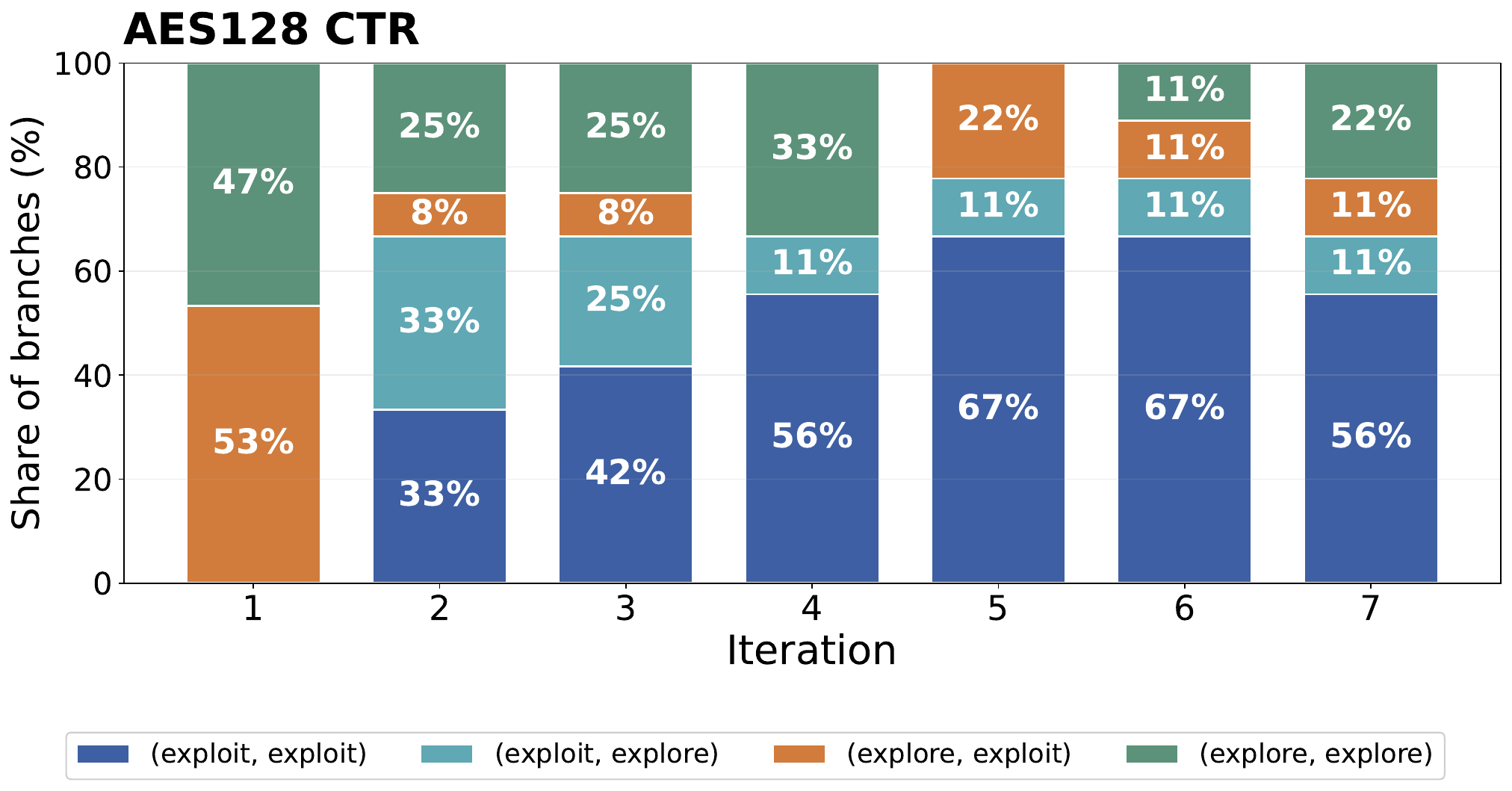}
        \caption{AES128 Ctr}
        \label{fig:plot3}
    \end{subfigure}
    \hfill
    \begin{subfigure}[t]{0.48\textwidth}
        \centering
        \includegraphics[width=\linewidth]{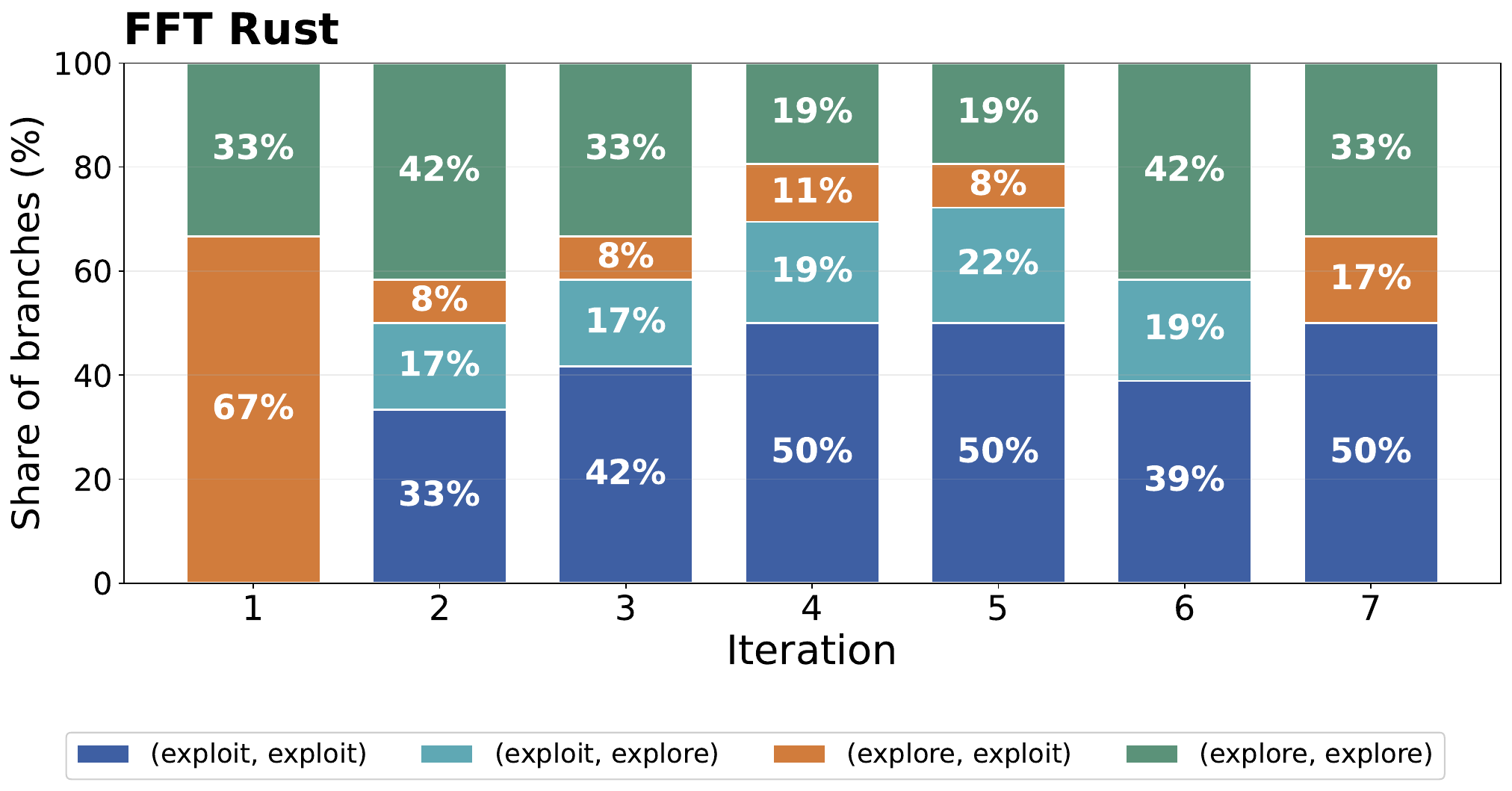}
        \caption{FFT Rust}
        \label{fig:plot4}
    \end{subfigure}

    \vspace{2mm}
    
    \begin{subfigure}[t]{0.48\textwidth}
        \centering
        \includegraphics[width=\linewidth]{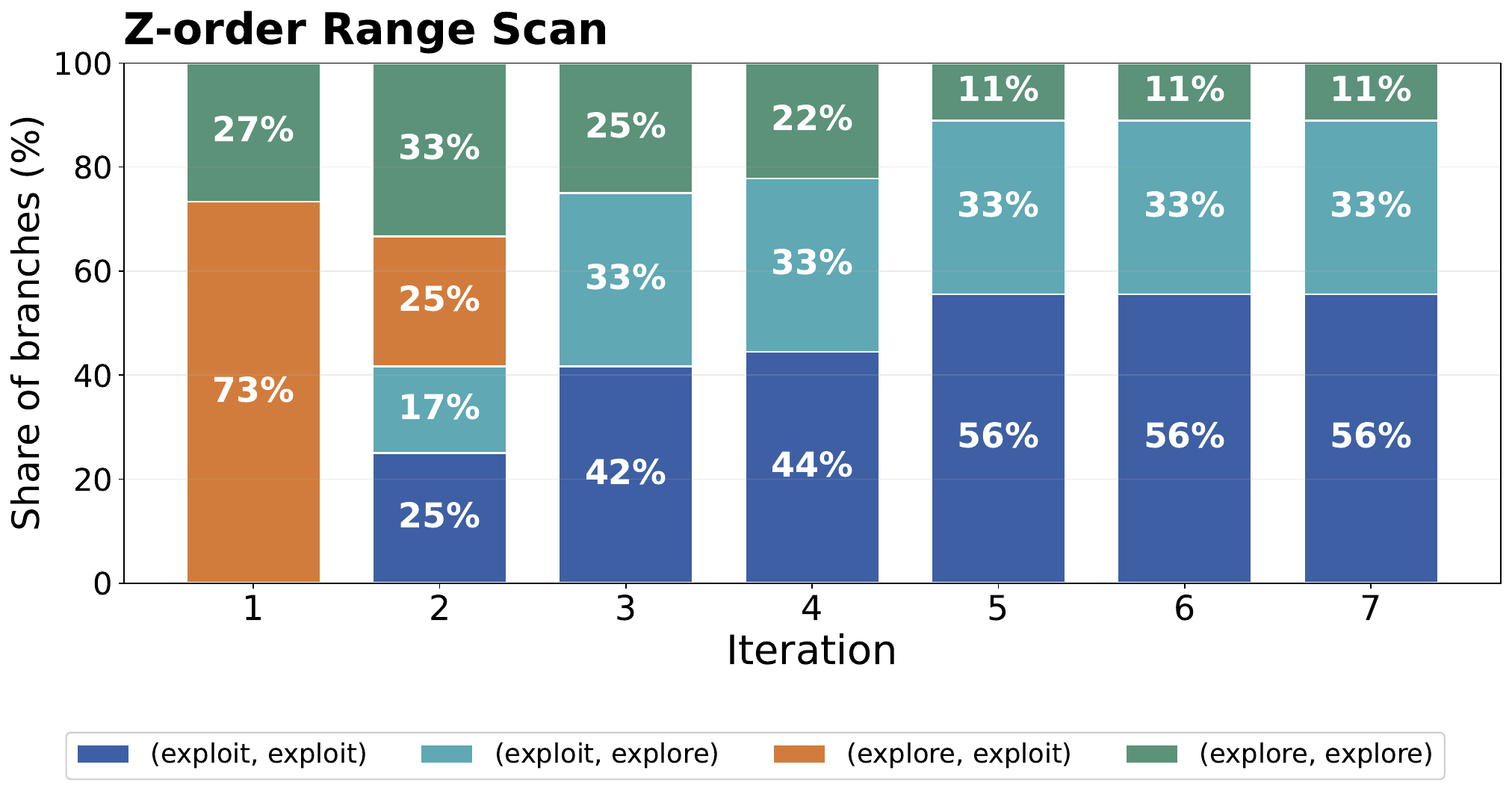}
        \caption{Z-order Range Scan}
        \label{fig:plot5}
    \end{subfigure}

    \caption{Explore/Exploit Action Distribution Across Branches.}
    \label{fig:dispatch_dist}
\end{figure*}

\clearpage
\subsection{Mode Distribution of the Expand Operator}
\label{apdx:expand}

We analyze how the Expand operator uses its four generation modes throughout the search. 
As an example, Figure~\ref{fig:expand_dist} (a) reports the number of ideas added at each Flash Attention iteration and the aggregate mode distribution across all five tasks.
Overall, the three modes excluding `fix-error' shows similar rates of usage.

As general analysis, we also provide the distribution of modes per task in Figure~\ref{fig:expand_dist} (b).
The cross-task distributions are mostly consistent across tasks.
The cross-pollination mode accounts for $35$--$41\%$ of generated ideas, novel idea generation for $29$--$32\%$, and score-guided refinement for $24$--$31\%$. 
Error-guided repair constitutes only a small remaining fraction. 
These results show that the Expand operator does not collapse onto a single generation strategy; it continually combines the refinement and recombination of existing evidence with the introduction of previously unexplored directions.

\begin{figure*}[h!]
    \centering

    \begin{subfigure}[t]{0.8\textwidth}
        \centering
        \includegraphics[width=\linewidth]{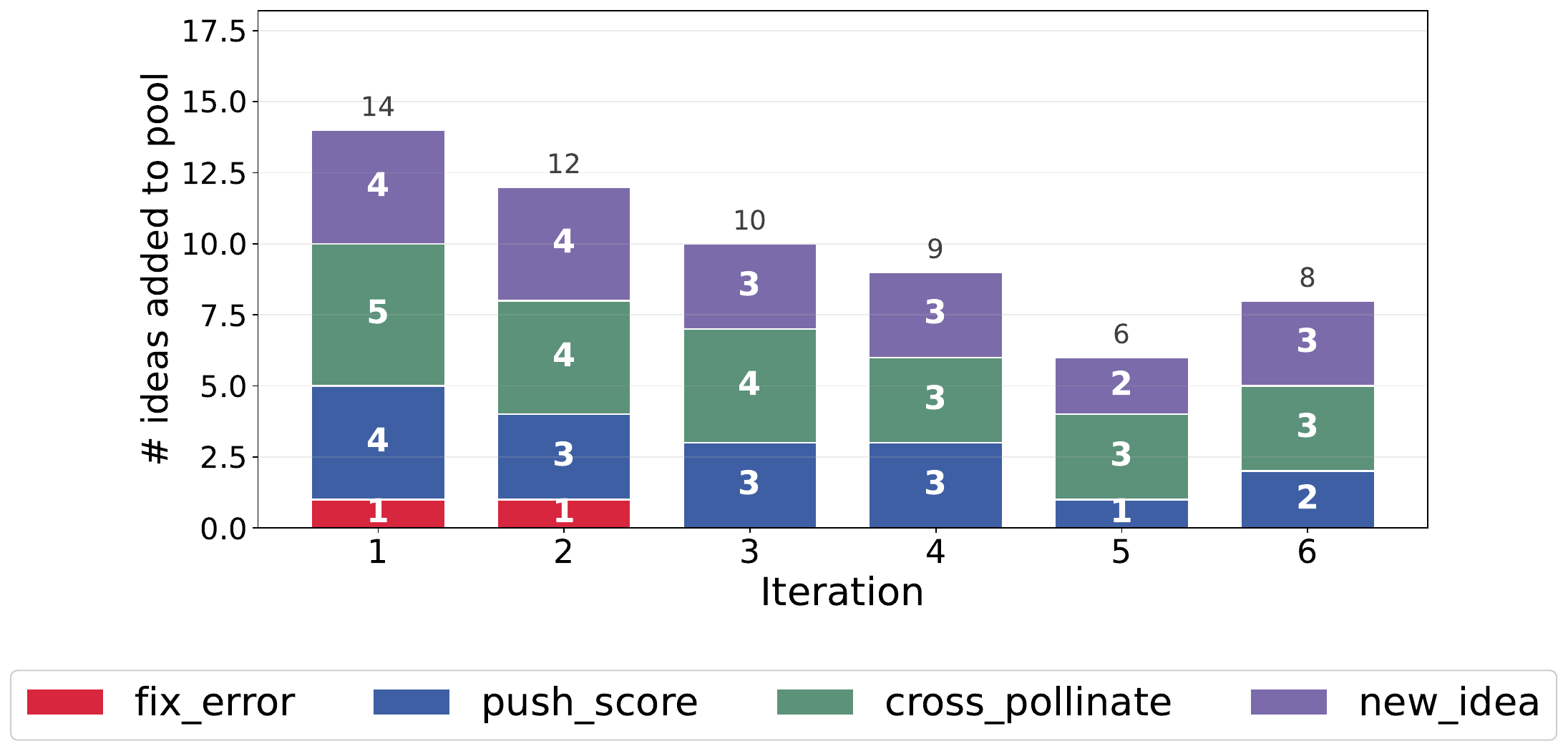}
        \caption{Flash Attention Task Mode Distribution}
        \vspace{5mm}
        \label{fig:fa_expand}
    \end{subfigure}
    \begin{subfigure}[t]{0.8\textwidth}
        \centering
        \includegraphics[width=\linewidth]{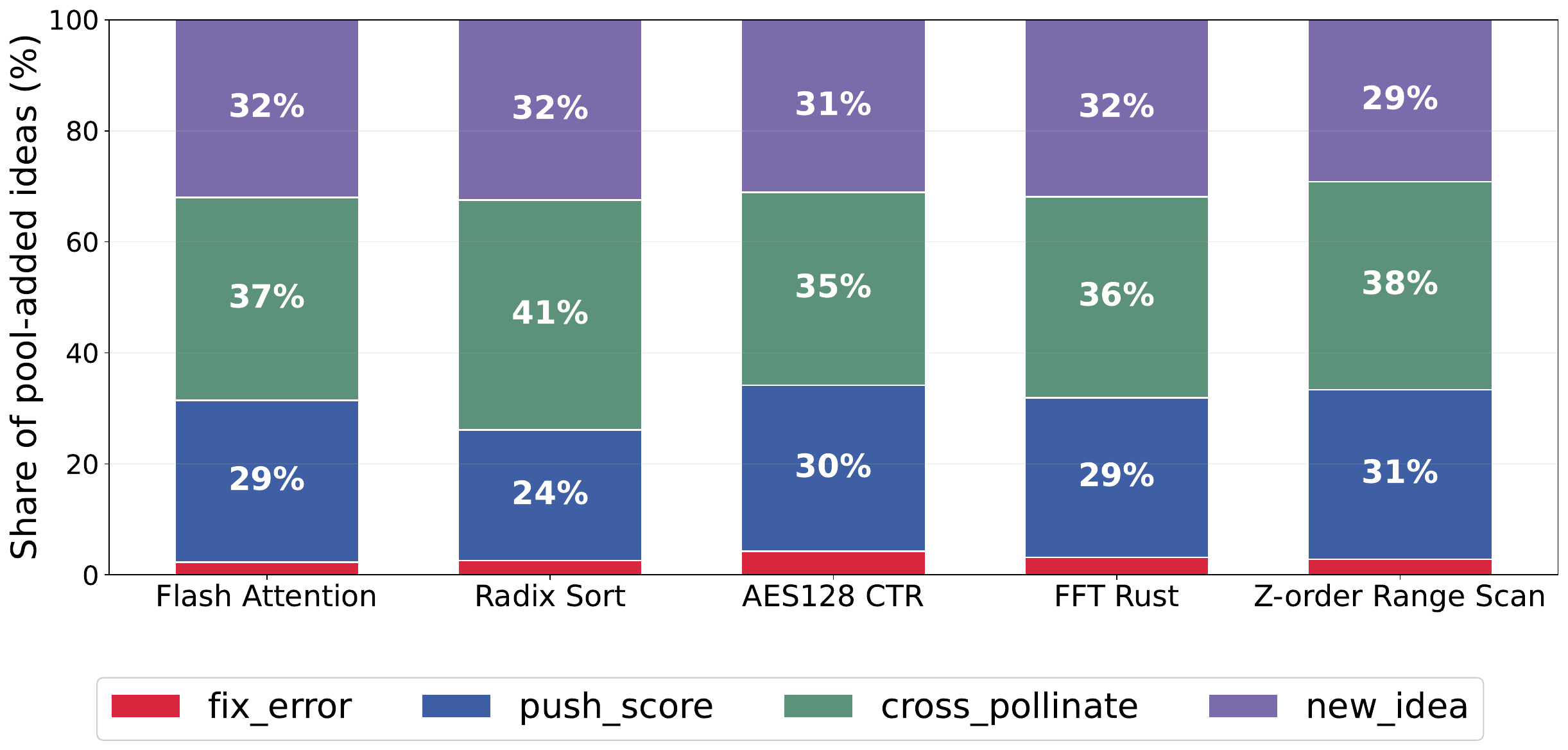}
        \caption{Cross-task Mode Distribution}
        \label{fig:all_expand}
    \end{subfigure}

    \caption{Expand operator mode distributions.}
    \label{fig:expand_dist}
\end{figure*}

\subsection{Effect of Solution Auditor Idea Reconstruction}
\label{apdx:auditor}

In the Solution Auditor component, a solution that was flagged solely with ``idea-mismatch'' has its idea reconstructed.
This reconstruction is intended to reduce the misalignment between the actual solution evaluated, and the corresponding idea that triggered the solution.
To verify if this feedback cycle is working properly, we show in Figure~\ref{fig:auditor} the ratio of \texttt{idea-mismatch} flags before and after the idea reconstruction is performed.
As intended, the ratio of the \texttt{idea-mismatch} flag is greatly reduced across the tasks, ensuring that the (idea, solution) pairs are much more aligned.
\begin{figure}[h!]
    \centering
    \includegraphics[width=0.9\linewidth]{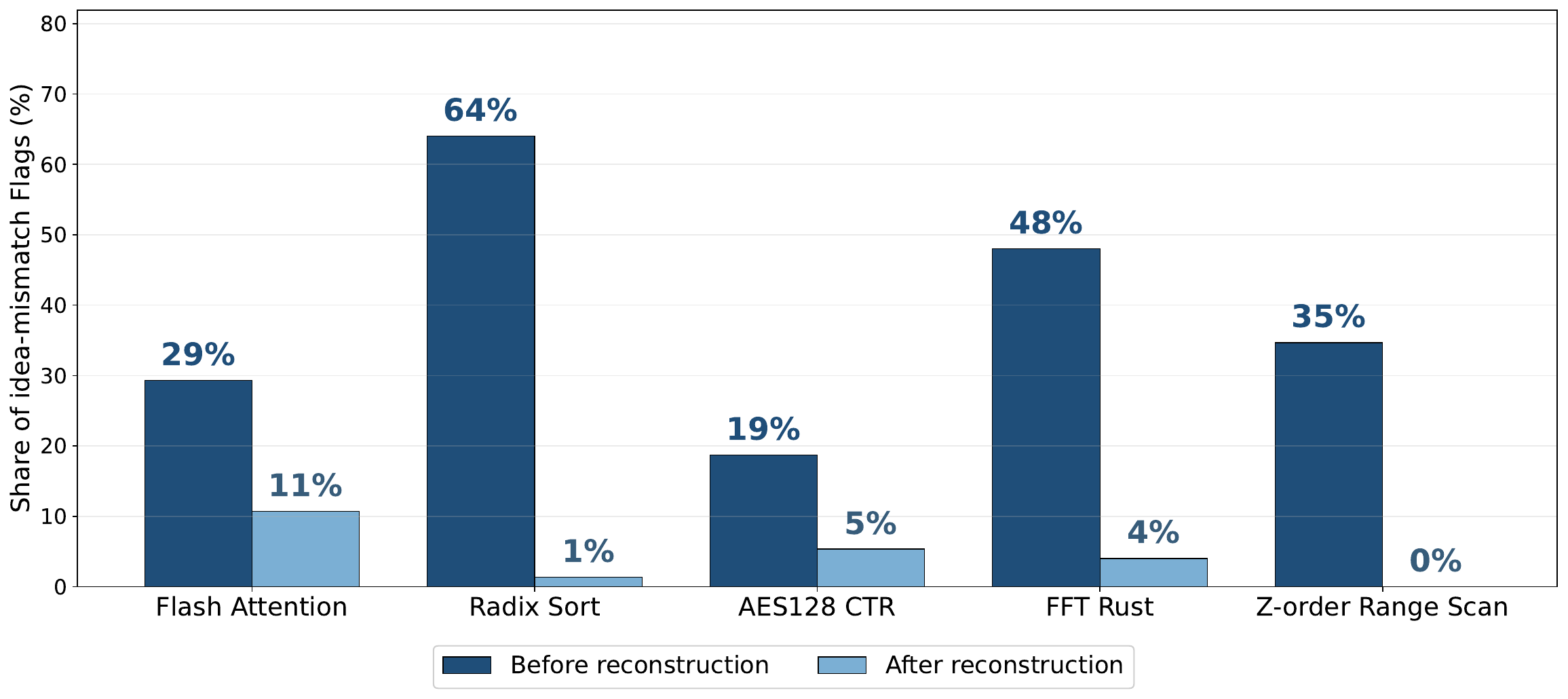}
    \caption{Decrease in \texttt{idea-mismatch} flags after Solution Auditor idea reconstruction.}
    \label{fig:auditor}
\end{figure}

\paragraph{Qualitative Analysis.}
As direct scrutiny of how idea reconstruction affects the process, we here provide couple of examples of the reconstruction process from the Flash Attention task.
For each example, we demonstrate the initially assigned idea, the auditor's verdict, and the reconstructed idea.

\subsubsection*{Example 1: the code relied on the technique the idea set out to avoid}
\label{example1}
\begin{tcolorbox}[promptbox={Example 1 --- Assigned Idea}]
\begin{Verbatim}[fontsize=\small,breaklines=true]
Title: AVX-512 Flash Attention 64x64 Tiling with Compiler Exp

Short Hypothesis: Expanding the tile size to 64x64 using AVX-512 intrinsics will improve L1 cache utilization and FLOP throughput over the C1 reference; furthermore, relying on standard exp() rather than a custom 5-term AVX-512 minimax exponential avoids FMA execution port contention in the heavily unrolled register-blocked loop, improving net throughput.

Abstract: The naive C0 baseline evaluates scaled dot-product attention in 0.75s by materializing a 128 MB FP64 score matrix. The C1 reference reduces this to ~0.10s using Flash Attention tiling (Br=Bc=32), online softmax, and AVX2 intrinsics. We propose further minimizing this runtime toward the theoretical ~0.05s floor by deploying AVX-512 intrinsics and expanding the tile size to 64x64, better utilizing L1 cache capacity and doubling SIMD lane width.

However, migrating to an aggressively unrolled 64x64 AVX-512 micro-kernel places immense pressure on CPU FMA execution ports. Although the literature typically suggests replacing scalar exp() with a vectorised 5-term minimax approximation to speed up the softmax rescale, we hypothesize that executing this polynomial in AVX-512 competes directly with the QK^T and SV dot-product FMAs.

By isolating the exp() computation from the dense FMA pipeline -- specifically by reverting to standard compiler-managed exp() or strategically placing it outside the innermost unrolled FMA block -- we can prevent execution port saturation. We expect this combined strategy (AVX-512, 64x64 tiling, and contention-aware exp() selection) to push performance substantially closer to the hardware's practical 0.05s floor.

Experiments:
1. Hypothesis: Expanding the tile size from 32x32 to 64x64 and migrating from AVX2 to AVX-512 intrinsics will improve throughput over the C1 reference implementation.
   Step 1: Port the existing C1 reference AVX2 kernel to use AVX-512 intrinsics (_mm512_fmadd_ps, etc.), initially keeping Br=Bc=32.
   Step 2: Expand the loop bounds and register blocking to support a Br=64, Bc=64 tile size.
   Step 3: Measure execution time and verify exact numerical correctness against the C0/C1 outputs. If register spilling degrades performance, evaluate asymmetric tiles (e.g., 64x32).
2. Hypothesis: Using standard exp() for softmax normalisation yields better overall throughput than a custom 5-term AVX-512 minimax exponential by avoiding FMA port contention in the unrolled inner loop.
   Step 1: Implement a 5-term AVX-512 minimax polynomial approximation for exp() (e.g., based on Schraudolph 1999) inside the 64x64 tiled micro-kernel.
   Step 2: Create an ablative variant that replaces the custom minimax exp() with calls to the standard math library exp().
   Step 3: Compare both variants' wall-clock runtimes. Profiling should confirm whether the FMA-heavy custom exp() reduces overall instruction throughput due to port contention compared to the standard exp().
\end{Verbatim}
\end{tcolorbox}

\begin{tcolorbox}[promptbox={Example 1 --- Auditor Verdict}]
\begin{Verbatim}[fontsize=\small,breaklines=true]
flags: ["idea_mismatch"]   confidence: 1.0

The idea explicitly proposes avoiding custom minimax exponential approximations in favor of the standard compiler exp() to prevent FMA execution port saturation in the heavily unrolled loop. However, the solution code contradicts this core hypothesis entirely by implementing and utilizing a custom AVX-512 minimax polynomial approximation (fast_exp512) for the softmax calculation. Because the implementation fundamentally relies on the technique the idea was explicitly aiming to ablate and avoid, there is a clear idea-solution mismatch.
\end{Verbatim}
\end{tcolorbox}

\begin{tcolorbox}[promptbox={Example 1 --- Reconstructed Idea}]
\begin{Verbatim}[fontsize=\small,breaklines=true]
Title: AVX-512 Flash Attention 64x64 Tiling with Minimax Exp

Short Hypothesis: Applying AVX-512 intrinsics with 64x64 tiling and a custom minimax polynomial for exp() will maximize L1 cache usage and throughput.

Abstract: The naive baseline evaluates scaled dot-product attention slowly by materializing a full score matrix. We optimize this by applying Flash Attention with a 64x64 tile size and AVX-512 intrinsics to fully utilize L1 cache capacity and double the SIMD lane width. To avoid sequential bottlenecks, we compute the softmax normalization using a custom, fully vectorised AVX-512 minimax polynomial approximation, isolating it from the innermost dot-product loops.

Experiments:
1. Hypothesis: A 64x64 AVX-512 Flash Attention kernel with a custom minimax exp() will yield higher performance than the C1 reference.
   Step 1: Implement the Flash Attention inner loops using AVX-512 intrinsics with a 64x64 tile size.
   Step 2: Implement a custom AVX-512 minimax exponential approximation for the softmax normalization.
   Step 3: Benchmark execution time against the baseline to verify throughput improvements.
\end{Verbatim}
\end{tcolorbox}

\noindent Without reconstruction, the pool would have recorded ``compiler
\texttt{exp} beats minimax \texttt{exp}'' as the best-supported claim in the
run, the opposite of what was tested, and the allocator would have exploited
that claim in later iterations.

\subsubsection*{Example 2: the secondary component survived, the central claim did not}

\begin{tcolorbox}[promptbox={Example 2 --- Assigned Idea}]
\begin{Verbatim}[fontsize=\small,breaklines=true]
Title: AVX-512 Sequence-Dim Vectorization with Minimax Exp

Short Hypothesis: Vectorizing over the sequence dimension (Br=16) rather than the feature dimension will eliminate costly horizontal sums while preserving the parent's highly accurate minimax exp speedup.

Abstract: Standard scaled dot-product attention computes an O(n^2) score matrix. For n=4096 and d=64, this 128 MB FP64 allocation thrashes the L1/L2 caches, resulting in severe latency (Baseline C0: ~0.75 s). Flash Attention algorithms bypass this by tiling the computation and maintaining running softmax statistics in fast cache (Reference C1: ~0.10 s). However, C1's reliance on AVX2 feature-dimension vectorization necessitates expensive horizontal sums, and its scalar exp() calls in the normalizer act as a serial bottleneck.

We propose a redesigned micro-kernel that pivots vectorization to the sequence dimension. By setting the row block size to Br=16, we perfectly map the computation to the 16 lanes of an AVX-512 register. This enables computing dot products for 16 distinct queries against a broadcasted key simultaneously using purely vertical _mm512_fmadd_ps operations -- completely eliminating horizontal reductions. We integrate this structural change with a vectorized 5-7 term minimax exp() approximation with ldexp range reduction to process the online-softmax rescale entirely within the AVX-512 pipeline.

By removing horizontal dependencies and replacing scalar math with wide SIMD operations, we expect to bridge the gap between the C1 reference (~0.10 s) and the theoretical order-of-magnitude floor (~0.05 s), all while adhering strictly to single-threaded FP32 numerical correctness requirements.

Experiments:
1. Hypothesis: An AVX-512 Flash Attention micro-kernel vectorized over the sequence dimension (Br=16) with a 5-7 term minimax exp approximation will outperform the C1 Reference by removing horizontal sums and scalar bottlenecks.
   Step 1: Write an AVX-512 micro-kernel in solve.c using a Br=16 tile size, reformulating QK^T and PV multiplication so that 16 queries are processed vertically in parallel without horizontal add instructions.
   Step 2: Implement a vectorized 5-term minimax exp() approximation using AVX-512 intrinsics for the online softmax rescaling block.
   Step 3: Benchmark wall-clock runtime against Baseline C0 and Reference C1, and validate that the output strictly matches the exact FP32 result.
2. Hypothesis: Vectorizing the micro-kernel over the sequence dimension (Br=16) yields a performance improvement independent of the minimax exp approximation due to the elimination of horizontal reductions.
   Step 1: Create an ablated version of the AVX-512 sequence-dim kernel that temporarily disables the vectorized minimax exp, falling back to sequential scalar exp() calls for the softmax step.
   Step 2: Measure the performance of this ablated kernel against the Reference C1 to isolate the latency saved purely by eliminating horizontal SIMD sums.
   Step 3: Compare this ablated timing to the fully optimized kernel from Experiment 1 to quantify the fraction of speedup attributable to the vectorized minimax exp().
\end{Verbatim}
\end{tcolorbox}

\begin{tcolorbox}[promptbox={Example 2 --- Auditor Verdict}]
\begin{Verbatim}[fontsize=\small,breaklines=true]
flags: ["idea_mismatch"]   confidence: 1.0

The idea proposes vectorizing over the query sequence dimension (Br=16) to perfectly map independent queries to lanes and completely eliminate horizontal reductions during the softmax step. However, the code vectorizes QK^T over the key sequence dimension (processing 1 query against 64 keys at a time) and explicitly uses horizontal reductions (_mm512_reduce_max_ps and _mm512_reduce_add_ps) across the SIMD lanes to compute the softmax statistics. While the code does successfully implement the proposed AVX-512 minimax exp() approximation, its overall vectorization strategy contradicts the core architectural claim of the idea.
\end{Verbatim}
\end{tcolorbox}

\begin{tcolorbox}[promptbox={Example 2 --- Reconstructed Idea}]
\begin{Verbatim}[fontsize=\small,breaklines=true]
Title: AVX-512 Key-Sequence and Feature Vectorization with Minimax Exp

Short Hypothesis: Vectorizing QK^T over keys and PV over features with an AVX-512 minimax exp() will accelerate attention by maximizing FMA throughput.

Abstract: Standard scaled dot-product attention computes an O(n^2) score matrix, causing cache thrashing for large sequences. We propose an AVX-512 micro-kernel using 64x64 blocks that vectorizes the QK^T multiplication over the key sequence dimension, computing dot products for 4 queries against 64 keys concurrently via broadcasted query features. The softmax normalizer applies a vectorized 5-term minimax exp() approximation and uses horizontal SIMD reductions for running max and sum statistics. Finally, the PV step vectorizes over the feature dimension, broadcasting attention weights to compute the final output.

Experiments:
1. Hypothesis: An AVX-512 Flash Attention kernel vectorized over key sequence and feature dimensions with a minimax exp() will outperform baseline implementations.
   Step 1: Write an AVX-512 kernel with Br=64 and Bc=64, vectorizing QK^T over keys and PV over features.
   Step 2: Implement a vectorized 5-term minimax exp() approximation using AVX-512 intrinsics and horizontal reductions for the softmax step.
   Step 3: Benchmark wall-clock runtime against the baseline and validate FP32 numerical correctness.
\end{Verbatim}
\end{tcolorbox}
\noindent This is the subtle case: a user would conclude that eliminating horizontal sums is key, when the kernel that earned that reward performs horizontal sums. 
Reconstruction keeps the credit on the parts that were actually executed.

Reconstruction is a correction of \emph{attribution}.
Because every downstream meta stage consumes scores through idea attributions, uncorrected mismatches would teach the surrogate and the allocator the wrong lessons about which mechanisms work, and would do so
most strongly for the highest-scoring branches. 

\section{Theoretical Analyses}
\label{apdx:theory}

\subsection{A Note on the Finiteness of Idea Space $\mathcal{X}$}

In our theoretical analysis, we define the Semantic Coverage of an idea set.
Since comparison across methods will make sense only when the search space is finite, we set a proposition on the finiteness of the idea space.

\begin{proposition}[\textbf{Finiteness of the Idea Space}]
\label{prop:finite-idea-space}
For a fixed research task $\tau$, suppose each admissible research idea is
represented as a sequence of tokens from a finite vocabulary $\Sigma$, with
maximum description length $L<\infty$.
Then the corresponding idea space $\mathcal{X}$ is finite.
\end{proposition}

\begin{proof}
Let $\Sigma$ denote the finite token vocabulary and let
$\Sigma^{\leq L}$ denote the set of all token sequences of length at most $L$:
\begin{equation}
    \Sigma^{\leq L}
    =
    \bigcup_{\ell=0}^{L}\Sigma^\ell.
\end{equation}
Since $\Sigma$ is finite,
\begin{equation}
    |\Sigma^\ell|
    =
    |\Sigma|^\ell.
\end{equation}
Therefore,
\begin{equation}
    |\Sigma^{\leq L}|
    =
    \sum_{\ell=0}^{L}|\Sigma|^\ell
    <
    \infty.
\end{equation}

For a fixed task $\tau$, define the admissible idea space as
\begin{equation}
    \mathcal{X}
    =
    \left\{
        x\in\Sigma^{\leq L}
        :
        x \text{ constitutes an admissible research idea for } \tau
    \right\}.
\end{equation}
By construction,
\begin{equation}
    \mathcal{X}
    \subseteq
    \Sigma^{\leq L}.
\end{equation}
Since every subset of a finite set is finite,
\begin{equation}
    |\mathcal{X}|
    <
    \infty.
\end{equation}
Thus, the admissible idea space for task $\tau$ is finite.
\end{proof}

\subsection{Proof of Propositions}
\label{apdx:proofs}

\begin{proof}[\textbf{Proof of Proposition~\ref{thm:semantic-coverage}}]
Assumption~\ref{assump:semantic-decomposition} induces an equivalence relation
over executable solutions:
\begin{equation}
    z \sim z'
    \quad\Longleftrightarrow\quad
    \pi(z)=\pi(z').
\end{equation}
Let
\begin{equation}
    q:\mathcal{Z}\rightarrow\mathcal{Z}/{\sim}
\end{equation}
denote the corresponding quotient map, where $q(z)=[z]$ is the semantic
equivalence class containing $z$.
For any evaluated solution set $\mathcal{E}\subseteq\mathcal{Z}$, its semantic
coverage can therefore be written as
\begin{equation}
    C(\mathcal{E})
    =
    |q(\mathcal{E})|.
\end{equation}

Consider first an arbitrary search procedure that evaluates
\begin{equation}
    \mathcal{E}
    =
    \{z_1,\ldots,z_m\},
    \qquad
    m\leq N.
\end{equation}
Since $q$ is a function, taking its image cannot increase the cardinality of a
finite set. Hence,
\begin{equation}
    C
    =
    |q(\mathcal{E})|
    \leq
    |\mathcal{E}|
    =
    m
    \leq
    N.
    \label{eq:single-quotient-bound}
\end{equation}
Intuitively, quotient contraction may merge multiple evaluated solutions into
the same semantic class, but can never create additional semantic classes.

Now consider the idea-driven procedure, which allocates its $N$ experiment
executions to $N$ distinct ideas
\begin{equation}
    x_1,\ldots,x_N,
    \qquad
    x_i\neq x_j
    \quad\text{for }i\neq j.
\end{equation}
Let
\begin{equation}
    z_i=g(x_i)
\end{equation}
be the corresponding executable solutions.
By Assumption~\ref{assump:faithful-realization},
\begin{equation}
    \pi(z_i)=x_i.
\end{equation}
Therefore, for every $i\neq j$,
\begin{equation}
    \pi(z_i)\neq\pi(z_j),
\end{equation}
and hence
\begin{equation}
    z_i\not\sim z_j.
\end{equation}
Thus, no two of the $N$ solutions are contracted into the same equivalence
class, giving
\begin{equation}
    C^{\text{idea}}
    =
    \left|
        q\!\left(
            \{z_1,\ldots,z_N\}
        \right)
    \right|
    =
    N.
    \label{eq:two-quotient-coverage}
\end{equation}

Combining \eqref{eq:single-quotient-bound} and
\eqref{eq:two-quotient-coverage}, we obtain
\begin{equation}
    C
    \leq
    N
    =
    C^{\text{idea}},
\end{equation}
which proves the result.
\end{proof}

\begin{proof}[\textbf{Proof of Proposition~\ref{thm:coverage-benefit}}]
Let a search procedure cover $C$ distinct semantic directions.
Under Assumption~\ref{assump:semantic-uncertainty}, conditional on not having
yet encountered an $\varepsilon$-optimal direction, the competitive directions
remain exchangeable among the unexplored directions.

After $i$ distinct non-competitive directions have been explored, there remain
$K-i$ unexplored directions, of which $G_\varepsilon$ are
$\varepsilon$-optimal.
Therefore, the conditional probability that the next explored direction is also
non-competitive is
\begin{equation}
    \frac{K-G_\varepsilon-i}{K-i}.
\end{equation}
Hence, for $C\leq K-G_\varepsilon$, the probability of failing to encounter any
$\varepsilon$-optimal direction after covering $C$ distinct directions is
\begin{align}
    1-P_\varepsilon(C)
    &=
    \prod_{i=0}^{C-1}
    \frac{K-G_\varepsilon-i}{K-i} \\
    &=
    \frac{\binom{K-G_\varepsilon}{C}}
         {\binom{K}{C}}.
\end{align}
Thus,
\begin{equation}
    P_\varepsilon(C)
    =
    1-
    \frac{\binom{K-G_\varepsilon}{C}}
         {\binom{K}{C}}.
\end{equation}
If $C>K-G_\varepsilon$, at least one competitive direction must necessarily
have been covered, and therefore $P_\varepsilon(C)=1$, which is consistent
with the same combinatorial expression under the convention
$\binom{K-G_\varepsilon}{C}=0$.

Now fix a target success probability $1-\delta$ and define
\begin{equation}
    C_\delta(K,G_\varepsilon)
    =
    \min
    \left\{
        C :
        P_\varepsilon(C)\geq 1-\delta
    \right\}.
\end{equation}

For fixed $K$ and $C$, each factor
\begin{equation}
    \frac{K-G_\varepsilon-i}{K-i}
    =
    1-\frac{G_\varepsilon}{K-i}
\end{equation}
is non-increasing in $G_\varepsilon$.
Hence, $P_\varepsilon(C)$ is non-decreasing in $G_\varepsilon$, implying
\begin{equation}
    G_\varepsilon^{(1)}
    \leq
    G_\varepsilon^{(2)}
    \quad\Longrightarrow\quad
    C_\delta(K,G_\varepsilon^{(1)})
    \geq
    C_\delta(K,G_\varepsilon^{(2)}).
\end{equation}
Thus, for a fixed number of admissible directions, sparser competitive
directions require broader semantic coverage.

Similarly, for fixed $G_\varepsilon$ and $C$, each factor
\begin{equation}
    1-\frac{G_\varepsilon}{K-i}
\end{equation}
is non-decreasing in $K$.
Therefore, increasing the number of admissible directions while keeping the
number of competitive directions fixed decreases $P_\varepsilon(C)$ and
increases the semantic coverage required to attain the same success probability.

Finally, the failure probability satisfies
\begin{align}
    1-P_\varepsilon(C)
    &=
    \prod_{i=0}^{C-1}
    \left(
        1-\frac{G_\varepsilon}{K-i}
    \right) \\
    &\leq
    \left(
        1-\frac{G_\varepsilon}{K}
    \right)^C \\
    &\leq
    \exp\!\left(
        -\frac{C G_\varepsilon}{K}
    \right).
\end{align}
Consequently, the sufficient condition
\begin{equation}
    C
    \geq
    \frac{K}{G_\varepsilon}
    \ln\frac{1}{\delta}
\end{equation}
guarantees $P_\varepsilon(C)\geq 1-\delta$.
Thus, the sufficient semantic coverage scales with
$K/G_\varepsilon$, completing the proof.
\end{proof}

\subsection{Convex Hull Area Box-and-Whisker Plots}
\label{apdx:boxplots}

\begin{figure}[h!]
    \centering
    \includegraphics[width=\linewidth]{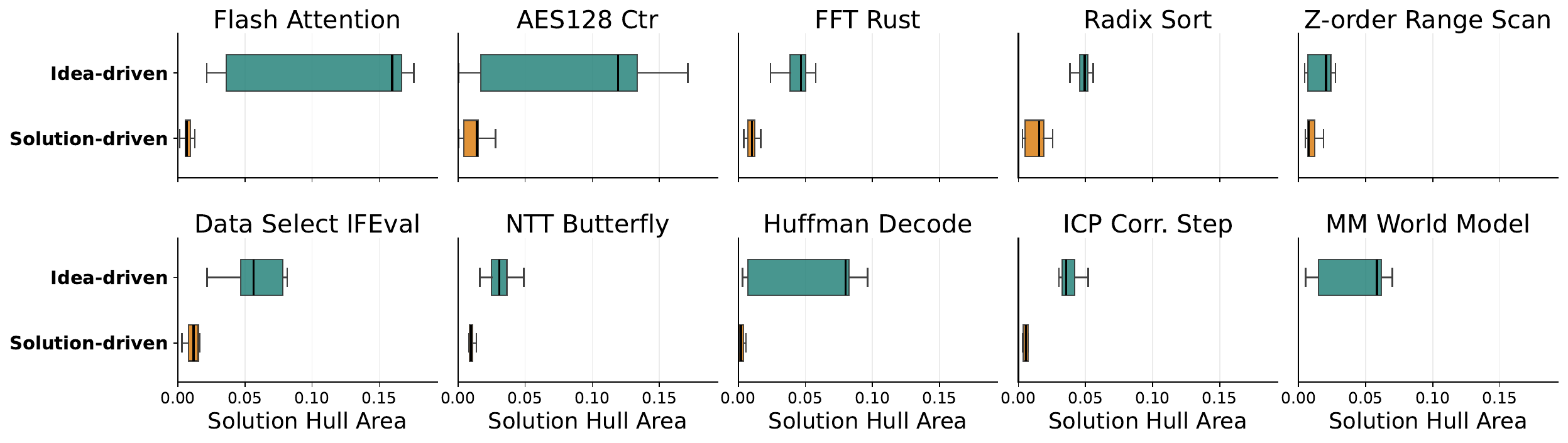}
    \caption{Idea-driven approaches generally show broader coverage of solutions in the embedding space. Embedding cosine similarity-based supplementary analysis is in Appendix~\ref{apdx:cosine}.}
    \label{fig:boxplot}
\end{figure}

\clearpage
\subsection{Convex Hull Visualization}
\label{apdx:convex-hull}
In Figure~\ref{fig:convexhull}, we provide visual examples of the convex hulls described in Section~\ref{sec:theory}.

\begin{figure}[h!]
    \centering
    \includegraphics[width=\linewidth]{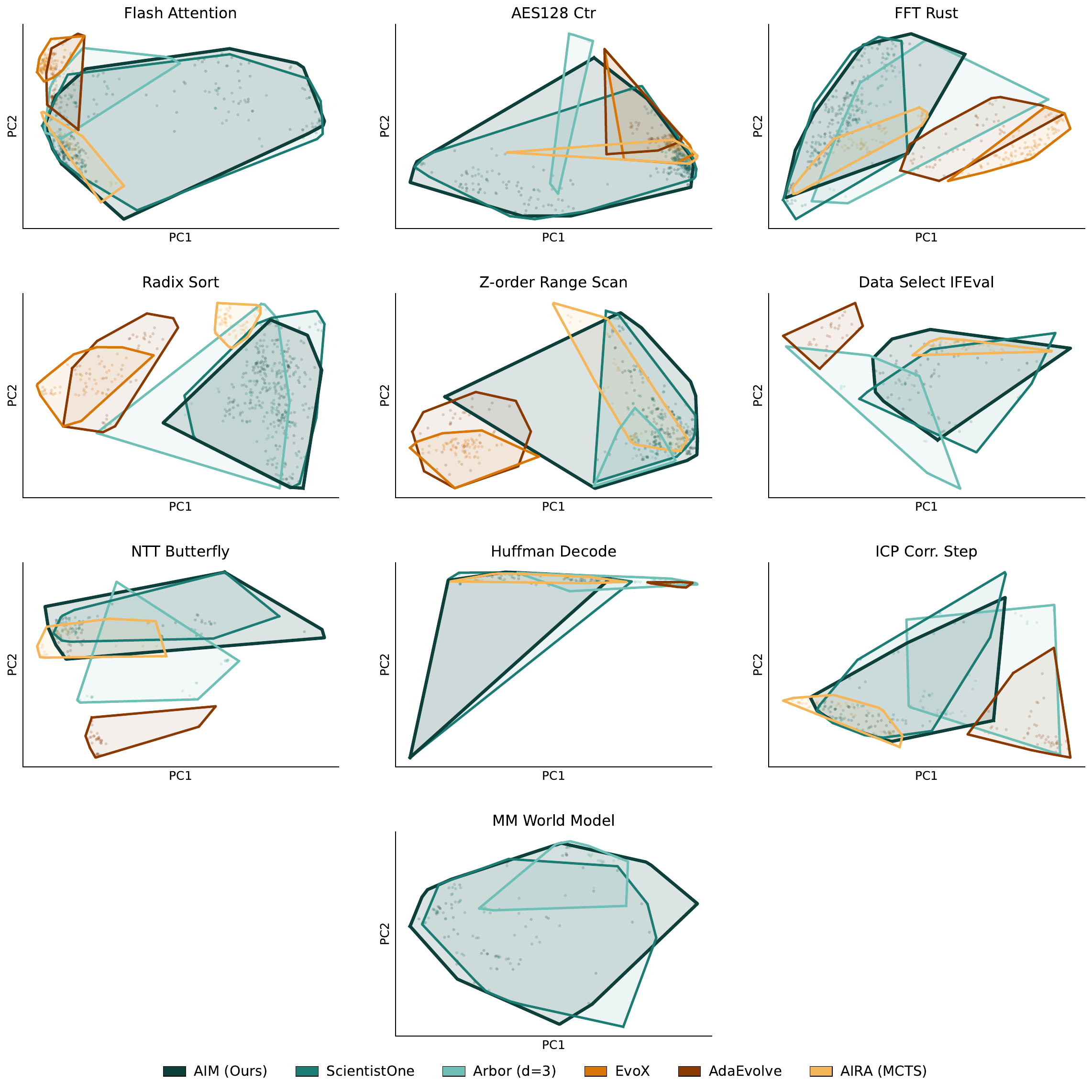}
    \caption{Convex Hull Visualization}
    \label{fig:convexhull}
\end{figure}

\subsection{Supplementary Cosine-Similarity-based Diversity Analysis}
\label{apdx:cosine}
To supplement the convex-hull area computation in Figure~\ref{fig:theory}, we further provide a comparative analysis between the solution-driven and idea-driven approaches' cosine similarities.
We measure how much of the solution space each method actually explores by embedding every generated candidate and averaging the pairwise cosine distance across the resulting
pool~(Table~\ref{tab:diversity}). 
Specifically, we measure the Mean pairwise cosine distance:
\begin{equation}
\label{eq:cosine}
    \mathcal{D}=\frac{1}{\binom{N}{2}}\sum_{i<j}(1-\cos(x_i,x_j)),
\end{equation} 
over the embeddings, averaged across three independent runs for each (method, task).
Overall, the idea-driven AIM and ScientistOne produce the widest pools on Flash Attention~($\mathcal{D}\approx 0.115$), roughly $3{-}4\times$ the value produced
by every solution-driven baseline.
The consistent bottom of the table is populated by solution-driven methods.
EvoX collapses to the tightest pool on Flash Attention ($\mathcal{D}=0.0153$, roughly $7.5\times$ narrower than AIM), confirming that solution-driven approaches generally search narrower areas compared to idea-driven methods.
Also, the measurements' gap is reduced in Radix Sort compared to Flash Attention, exactly matching the observations in the convex-hull area analysis in Figure~\ref{fig:theory}.

\begin{table}[h!]
  \centering
  \small
  \caption{\textbf{Supplementary Cosine Similarity Analysis.} Entries are measured values of \eqref{eq:cosine}. Higher is more diverse. Bold marks the highest value per column.}
  \label{tab:diversity}
  \begin{tabular}{lcc}
    \toprule
    Method              & Flash Attention & Radix Sort \\
    \midrule
    \multicolumn{3}{l}{{\textit{Idea-driven Approaches}}} \\    
    
    AIM & 0.1150          & 0.0600            \\
    ScientistOne        & \textbf{0.1199} & 0.0661            \\
    Arbor               & 0.0396          & \textbf{0.0684}   \\
    \midrule
    \multicolumn{3}{l}{{\textit{Solution-driven Approaches}}} \\    
    AIRA (MCTS)         & 0.0311          & 0.0335            \\
    AdaEvolve           & 0.0297          & 0.0502            \\
    EvoX                & 0.0153          & 0.0391            \\
    \bottomrule
  \end{tabular}
\end{table}

\end{document}